\documentclass{article}

\usepackage[accepted]{mlsys2022}

\mlsystitlerunning{Federated Adversarial Domain Adaptation}

\usepackage{microtype}
\usepackage{graphicx}
\usepackage{subcaption}
\usepackage{booktabs} 

\usepackage{hyperref}

\mlsystitlerunning{Federated Adversarial Domain Adaptation}
\usepackage[utf8]{inputenc} 
\usepackage[T1]{fontenc}    
\usepackage{hyperref}       
\usepackage{url}            
\usepackage{booktabs}       
\usepackage{amsfonts}       
\usepackage{nicefrac}       
\usepackage{placeins}
\usepackage{microtype}      
\usepackage{amssymb}
\usepackage{xcolor}
\usepackage{framed}
\usepackage{booktabs}
\usepackage{bbm} 
\definecolor{shadecolor}{rgb}{1,0,0}
\usepackage{microtype}
\usepackage{wrapfig}

\usepackage{adjustbox}

\usepackage{amsthm}
\theoremstyle{plain}

\newtheorem*{rep@theorem}{\rep@title}
\newcommand{\newreptheorem}[2]{%
\newenvironment{rep#1}[1]{%
 \def\rep@title{#2 \ref{##1}}%
 \begin{rep@theorem}}%
 {\end{rep@theorem}}}
\makeatother
\newtheorem{theorem}{\protect\theoremname}
\providecommand{\theoremname}{Theorem}
\newreptheorem{theorem}{Theorem}
\newtheorem{assumption}{\protect\assname}
\providecommand{\assname}{Assumption}
\newreptheorem{assumption}{Assumption}

\providecommand{\corollaryname}{Corollary}
\newreptheorem{corollary}{Corollary}
\newtheorem{lemma}{\protect\lemmaname}
\providecommand{\lemmaname}{Lemma}
\newreptheorem{lemma}{Lemma}

\providecommand{\remarkname}{Remark}
\newreptheorem{remark}{Remark}

\providecommand{\definitionname}{Definition}
\newreptheorem{definition}{\definitionname}

\providecommand{\definitionname}{Conjecture}
\newreptheorem{conjecture}{\conjurnnam}

\providecommand{\definitionname}{Axiom}
\newreptheorem{axiom}{\axname}
\newtheorem{proposition}{Proposition}

\usepackage{cancel}

\usepackage{microtype}
\usepackage{graphicx}
\usepackage{booktabs} 

\usepackage{times}
\usepackage{graphicx} 
\usepackage{amsmath}
\usepackage[normalem]{ulem}
\usepackage{caption}
\usepackage{subcaption}
\usepackage{multirow}
\usepackage{tabularx}
\usepackage{enumitem}
\usepackage{verbatim}

\usepackage{hyperref}
\usepackage{algorithmic}
\usepackage{algorithm}
\usepackage{natbib}

\begin{document}

\twocolumn[
\mlsystitle{FedMM: Saddle Point Optimization for\\ Federated Adversarial Domain Adaptation}



\mlsyssetsymbol{equal}{*}

\begin{mlsysauthorlist}
\mlsysauthor{Yan Shen}{equal,SUNY}
\mlsysauthor{Jian Du}{equal,Ant}
\mlsysauthor{Han Zhao}{UIUC}
\mlsysauthor{Benyu Zhang}{Ant}
\mlsysauthor{Zhanghexuan Ji}{SUNY}
\mlsysauthor{Mingchen Gao}{SUNY}
\end{mlsysauthorlist}

\mlsysaffiliation{SUNY}{Department of Computer Science and Engineering,
University at Buffalo, NY.}
\mlsysaffiliation{Ant}{ Ant Group, Sunnyvale, CA.}
\mlsysaffiliation{UIUC}{ Department of Computer Science, University of Illinois at Urbana-Champaign, IL}

\mlsyscorrespondingauthor{Jian Du}{jd.jiandu@gmail.com}

\mlsyskeywords{Machine Learning, MLSys}

\vskip 0.3in

\begin{abstract}
Federated adversary domain adaptation is a unique distributed minimax training task due to the prevalence of label imbalance among clients, with each client only seeing a subset of the classes of labels required to train a global model. To tackle this problem, we propose a distributed minimax optimizer referred to as FedMM, designed specifically for the federated adversary domain adaptation problem. It works well even in the extreme case where each client has different label classes and some clients only have unsupervised tasks. We prove that FedMM ensures convergence to a stationary point 
with domain-shifted unsupervised data.
On a variety of benchmark datasets, extensive experiments show that FedMM consistently achieves either significant communication savings or significant accuracy improvements over federated optimizers based on the gradient descent ascent (GDA) algorithm. When training from scratch, for example, it outperforms other GDA based federated average methods by around $20\%$  in accuracy over the same communication rounds; and it consistently outperforms  when training from pre-trained models with an accuracy improvement from $5.4\%$ to $9\%$ for different  networks.
\end{abstract}
]



\printAffiliationsAndNotice{\mlsysEqualContribution} 

\section{Introduction}
Federated Learning (FL) is gaining popularity because it enables multiple clients to train machine learning models iteratively and distributedly without directly sharing the potentially sensitive data with other clients \cite{kairouz2019advances, li2020federated}.
The FL training pipeline involves exchanging local model parameters with a server to update the global model, and its communication overhead has been, in many cases, identified as the bottleneck~\cite{mcmahan2017communication, chen2020breaking}. 
Moreover, due to the heterogeneity, {\it domain shift} often exists between clients' data~\cite{quinonero2009dataset}, which is another characteristic feature of FL training,  resulting from the data being sampled from different parts of the sample space on different clients.
Because of the aforementioned two distinguishing features, FL training necessitates optimizers that converge on heterogeneous data  among clients  while requiring fewer communication rounds.

For data with distributional shifts, one of the most challenging settings  is that each local client only has access to a subset of the label classes in order to train the global/common model. In this situation, the global model's accuracy suffers considerably as a result of the gradient/model drift~\cite{mcmahan2017communication}. In the literature of domain adaptation, this problem is also known as label shift~\citep{zhang2013domain,tachet2020domain}. Under the setting of FL, it is a natural occurrence due to the imbalance between clients' label distributions, with the extreme case being individual clients with different domain labels, or clients without labels (unsupervised local model). Furthermore, recent techniques for domain adaptation with adversarial training~\citep{ganin2016domain,tzeng2017adversarial,zhao2018adversarial} on minimax objectives complicates convergence even further.

One method is to use the gradient descent ascent (GDA) method~\citep{lin2020gradient} directly as if the data are \emph{homogeneous and centralized} globally where data are aggregated together to find saddle point solutions~\cite{jin2020local, lin2020near}. However, because of the  domain shifts among clients in FL settings, a single client cannot access an unbiased sampling of the global objective (descent or ascent) gradient. A natural solution would be averaging on each client's gradients, which exactly corresponds to the FedSGDA approach in \cite{peng2019federated}. Its training efficiency, on the other hand, is low due to the requirement of large communication rounds between the server and clients. 
Without considering the issue of domain shift, there are several works on communication-efficient FL algorithms, a large spectrum are variations of the FedAvg~\cite{mcmahan2017communication}. However, if the data are non-i.i.d among clients, especially in the case of imbalanced label distributions, the performance of FedAvg would be significantly lower than that when all data was trained on a single client.

\begin{figure}[t]
\centering
\includegraphics[width=0.46\textwidth]{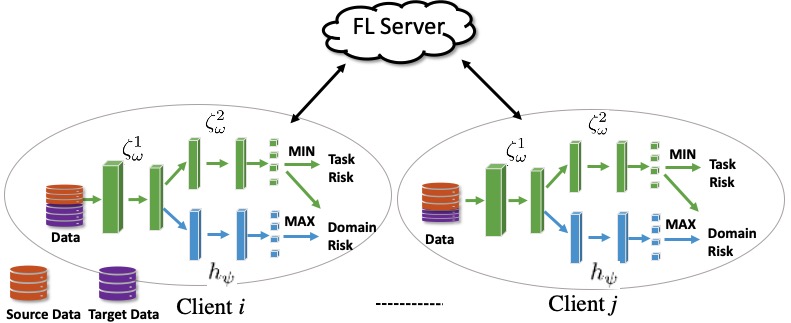}
\caption{A federated adversarial domain adaptation model.}
\label{fig:fl_da}
\end{figure}
For training a federated minimax objective, we show the typical pipeline of FedAvg with GDA, referred to as FedAvgGDA with network, in Fig.~\ref{fig:fl_da}.
Specifically, in each client's local oracle, only the source risk of the client's local source data (if any) and the domain risk of the client's source/target data are accessible. The federated domain adaptation algorithm optimizes the weighted sum of each client's local loss functions in a collaborative minimax fashion.  A detailed explanation is presented in the next section. However, the federated adversarial domain adaptation method is extremely sensitive to the unbalanced distributions of data labels, which has been analyzed theoretically in the literature~\citep{zhao2019learning}. We also empirically verified and confirmed this phenomenon, as shown in Fig.~\ref{fig:oldmethod}.

{\textbf{FedMM}}. 
We formulate this distributed  saddle point optimization  as a Federated MiniMax (FedMM) optimization on a sum of non-identical distributions. In particular, we use an augmented Lagrange function to enforce the global model consensus constraints. Furthermore, in each client's local optimization oracle, FedMM deconstructs the global sum by solving the augmented Lagrange of each function individually. The collection of Lagrange dual variables locally compensates for client-to-client model divergence caused by data domain shift. We detail the algorithm in Section~\ref{sec:FedMM}.



\textbf{Contributions:}
Label imbalance is a natural and extremely challenging problem in federated domain adaptation. As demonstrated in Fig~2, FedAvg’s low performance is driven by the imbalance of domain label distributions across clients. Our paper aims to tackle these challenging issues. We summarize our key contributions as follows:
\begin{itemize}
\item
 We  present,  FedMM,  a  specifically  designed  distributed optimizer for federated minimax optimizations
with non-separable minimization and maximization
variables, as well as clients with uneven label class
distributions. It works in the extreme case where each
client has disjoint classes of labels and some clients
even have unsupervised task.
\item
Under the generic federated saddle point optimization problem with a nonconvex-concave global objective function assumption, we prove that FedMM converges to a
stationary point for the nonconvex-strongly-concave
case.\footnote{We focus on the convergence analysis of the federated nonconvex-strongly-concave case, which is a difficult problem itself even in the centralized setting and has recently received increasing attention in the literature~\cite{luo2020stochastic, jin2020local, lin2020near}} Based on our theoretical analysis, we show that FedMM
converges to a stationary point even if the data distribution
suffers from domain shifts.
\item
FedMM consistently achieves either significant communication savings or significant accuracy improvements over the federated gradient descent ascent (GDA)
method on a variety of benchmark datasets with varying adversarial domain adaptation networks. For example, when training from scratch, 
it outperforms other GDA based federated average methods by around $20\%$ in accuracy over the same communication rounds; and it consistently outperforms when training from pre-trained models with an accuracy improvement from $5.4\%$ to $9\%$ for different  networks.
\end{itemize}


\begin{figure}
  \begin{center}
    \includegraphics[width=0.40\textwidth]{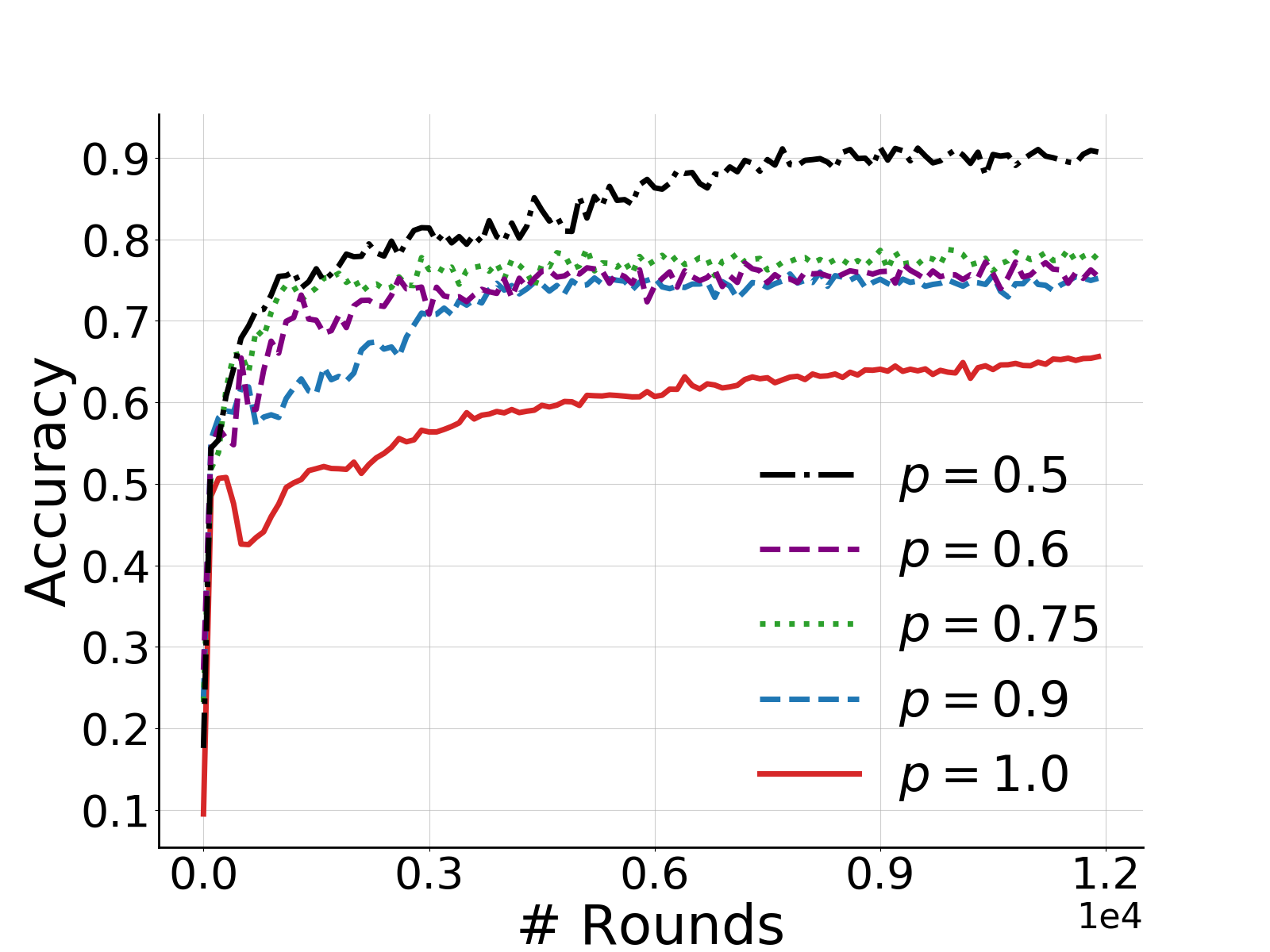}
  \end{center}
  \caption{FedAvgGDA for CDAN with 2 clients. The ratios for source and target data allocated to client 1 are $p$ and $(1-p)$. The remaining data pertains to client 2. It shows that the performance of FedAvgGDA degrades rapidly as the data distribution becomes unbalanced, which motivates our FedMM algorithm.}
  \label{fig:oldmethod}
 \end{figure}
\section{Centralized Adversarial Domain Adaptation}
\label{sect:preliminary}

Domain adaptation refers to the process of transferring knowledge from a labeled source domain to an unlabeled target domain~\citep{ben2010theory,zhao2019learning}. Let $\mathcal P$ and $\mathcal Q$ be the source and target distributions, respectively.
In a general formulation, the upper bound of the target prediction error is given by \citet{ben2010theory}
\begin{equation}
\label{upperbound}
\small
    \text{err}_{\mathcal Q}(\zeta) \leq \text{err}_{\mathcal P}(\zeta)+ d_{\mathcal{H}}(\mathcal{P,Q}) + \min_{\zeta^{*} \in \mathcal{F}}\{\text{err}_{\mathcal P}(\zeta^{*})+  \text{err}_{\mathcal Q}(\zeta^{*})\},
\end{equation}
where $\text{err}_{\mathcal Q}(\zeta)$ denotes the population loss of $\zeta$, i.e.,
$
    \text{err}_{\mathcal Q}(\zeta) \triangleq \mathbb{E}_{(\mathbf{x}_i, \mathbf{y}_i) \sim \mathcal{Q}}[\ell(\zeta(\mathbf{x}_i), \mathbf{y}_i)]
$, and
we use the
parallel notation $\text{err}_{\mathcal P}(\zeta)$ for the source domain. Besides,  
$d_{\mathcal{H}}(\mathcal{P,Q})$ is a discrepancy-based distance and $\min_{\zeta^{*} \in \mathcal{F}}\{\text{err}_{\mathcal P}(\zeta^{\ast})+  \text{err}_{\mathcal Q}(\zeta^{\ast})\}$ is a lower bound on the sum of source and target domain's population loss of $\zeta$ in a {\it{hypothesis}} class $\mathcal{F}$. 

For the unsupervised domain adaptation problem, it has been proven that minimizing the upper bound, which is the r.h.s in (\ref{upperbound}), leads to an architecture consisting of a {\it feature extractor} parameterized by $\omega$,  i.e., $\zeta^1_{\omega}$, a {\it label predictor},
parameterized also by $\omega$ i.e., $\zeta^2_{\omega}$ ( $\zeta_{\omega}\triangleq \zeta^2_{\omega}\circ \zeta^1_{\omega}$),\footnote{The parameters of $\zeta^1$ and $\zeta^1$ are  not the same. In this case,  we abuse the notation to simplify the expression.} and a {\it domain classifier}
parameterized by $\psi$,  i.e., $h_{\psi}$, as shown  in Fig~\ref{fig:fl_da}.
The feature extractor generates the domain-independent feature representations, which are then fed into the domain classifier and label predictor.
The domain classifier then tries to determine whether the extracted features belong to the source or target domain. Meanwhile, the label predictor predicts instance labels based on the extracted features of the labeled source-domain instances.

Minimizing the upper bound in (\ref{upperbound}) encourages the extracted feature to be both discriminative and invariant to changes between the source and target domains.
The upper bound minimization corresponding to a saddle point over the parameter space of $\omega$ and $\psi$ has been demonstrated using
$
\widehat{\mathbf{\omega}} 
\triangleq \arg\min_{\mathbf{\omega}} L_1\left(\mathbf{\omega} \right) - \nu L_2\left({\mathbf{\omega}}, \widehat{\mathbf{\psi}} \right)$
and $\widehat{\mathbf{\psi}}
\triangleq\arg \min_{\mathbf{\psi}} L_2\left(\widehat{\mathbf{\omega}}, \mathbf{\psi} \right)\nonumber$
with an equivalent minimax compact form as
\begin{equation}
\label{minimax}
\min_{\omega}\max_{\psi} F 
= \min_{\omega}\max_{\psi}  L_1\left(\omega \right) - \nu L_2\left(\omega, \psi \right).
\end{equation}
In the majority of  adversarial domain adaptation problems,  $L_1(\omega) \triangleq \mathbb{E}_{\left(\mathbf{x}_i, \mathbf{y}_i\right) \sim \mathcal{Q}}[\ell(\zeta_{\mathbf{\omega}}\left(\mathbf{x}_i), \mathbf{y}_i\right)]$ is the supervised learning loss on $\zeta$,   $L_2(\omega, \psi) \triangleq \mathbb{E}_{\left(\mathbf{x}_i\right) \sim \mathcal{Q}} D_Q(h_{\psi}\left(\zeta_{\omega}'\left(\mathbf{x}_i\right)\right) - \mathbb{E}_{\left(\mathbf{x}_i\right) \sim \mathcal{P}} D_P(h_{\psi}\left(\zeta_{\omega}'\left(\mathbf{x}_i\right)\right)$ is the domain classification loss, and $\nu$ is the trade-off coefficient between $L_1(\omega)$ and $L_2(\omega, \psi)$. With the commonly used cross-entropy loss for $L_2$, we have $D_Q(x) \triangleq 1-\log(x)$ and $D_P(x) \triangleq \log(1-x)$. Besides,   $\zeta_{\omega}^{\prime}$ is the feature and $h_{\psi}(\cdot): \mathbb{R}^D \to [0, 1]$ is the probabilistic prediction of the domain label. In general, $\zeta_{\omega}^{\prime}$ and $h_{\psi}(\cdot)$   include, but is not limited to, the following cases:


\begin{itemize}
\item
Domain-Adversarial Neural Networks (DANN) \citep{ganin2015unsupervised}:
In DANN, the input of $h_{\mathbf{\psi}}(\cdot)$ is  designed simply to be the domain invariant feature $\zeta_{\omega}^1(\mathbf{x}_i)$, i.e.,
$
h_{\psi}\left(\zeta^1_{\omega}\left(\mathbf{x}_i\right)\right)$.

\item 
Margin Disparity Discrepancy (MDD) \cite{zhang2019bridging}:
In  MDD, the input of $h_{\mathbf{\psi}}(\cdot)$ is the concatenation of $\zeta^1_{\omega}$ and 
$\arg\max_c \zeta_{\omega}\left(\mathbf{x}_i; c\right)$ with $c$ the class type
i.e., 
$
    h_{\psi}\left(\left[\zeta^1_{\omega}\left(\mathbf{x}_i\right),\arg\max_c \zeta_{\omega}\left(\mathbf{x}_i; c\right)\right]\right)
$.
\item
Conditional Domain Adaptation Network (CDAN)
\cite{long2017conditional}:
In CDAN, the input of $h_{\psi}$ is from the cross-product space of  $\zeta^1_{\omega}(\mathbf{x}_i)$ and  $\zeta_{\omega}(\mathbf{x}_i)$, i.e.,
$
    h_{\psi}\left(\zeta^1_{\omega}(\mathbf{x}_i)\otimes \zeta_{\omega}(\mathbf{x}_i)\right)
$.
\end{itemize}
Our FedMM is a generic federated adversarial domain adaptation framework in which each client is equipped with $h_{\psi}$ and $\eta_{\omega}$ depending on the availability of source data, target data, or both.

\section{ Federated Adversarial Domain Adaption Formulation}
Due to privacy concerns regarding sensitive data, the data cannot be shared among clients. As a result, federated adversary domain adaption addresses the problem by training a transferred model among clients from a labeled source domain to an unlabeled target domain. A central server coordinates a loose federation of clients exchanging local models to solve the learning task.

To express the federated adversarial domain adaptation objective, we convert the joint learning objective in (\ref{minimax}) into the form of a centralized average of all the clients' objective functions, as given by
\begin{equation}
\begin{split}
\label{FL}
    \min_{\omega}\max_{\psi}f(\omega, \psi) 
    &\triangleq
    \min_{\omega}\max_{\psi}
    \frac{1}{N}\sum_{i=1}^N f_i(\omega, \psi),
    \end{split}
\end{equation}
where $N$ is the number of clients, and $f_i(\omega, \psi) $ is the average loss function at the $i$-th client, which is computed by 
\begin{equation}
f_i(\omega, \psi) 
\triangleq \alpha_{i} \sum_{\xi_{j}^{(i)} \in \mathcal{D}_{i}} F_i\left(\omega, \psi ; \xi_{j}^{(i)}\right),
\end{equation}
where $\alpha_{i}$ is the weight coefficient, and
$F_i\left(\omega, \psi; \xi_{j}\right)$ is the loss function w.r.t the data point $\xi_j^{(i)} \triangleq \{\mathbf{x}_j, \mathbf{y}_j\}$ in data set $\mathcal{D}_{i}$. 
The   objective function at client $i$ is specified based on whether the data is from the source domain or the target domain, i.e., 
\[ F_i\left(\omega, \psi; \xi_{j}^{(i)}\right) \triangleq
  \begin{cases}
    &\ell\left(\zeta_{\omega}\left(\mathbf{x}_i\right), \mathbf{y}_i\right) + \nu\log(1-\\
    & \hspace*{9mm}h_{\psi}\left(\zeta_{\omega}'\left(\mathbf{x}_i\right)\right)),    \hspace*{6.5mm} \text{if } \xi_i \in \mathcal Q,\\
    &\nu\log(h_{\psi}\left(\zeta_{\omega}'\left(\mathbf{x}_i\right)\right)), \hspace*{7mm} \text{if } \xi_i \in \mathcal P.
  \end{cases}
\]
This novel structure introduces additional challenges below in federated adversarial domain adaptation that do not exist in existing adversarial domain adaptation problems or the federated learning literature:
\begin{itemize}
    \item Clients are restricted to compute the minimax optimization in a distributed manner rather than the centralized minimax optimization.
    \item To train a common model, both the set of feature extractor variables $\omega$ and domain classifier variables $\psi$ are non-separable cross clients.\footnote{In contrast, in the federated robust optimization problem~\cite{reisizadeh2020robust, deng2021local}, the corresponding  maximization variables $\psi$ are separable across clients.}
    \item
    The marginal label distributions are class-imbalanced cross clients due to the uneven distribution of source domain data and target domain data. In extreme cases, each client may only access data from the target domain or the source domain; therefore, different data distributions and loss functions among clients degrade distributed learning performance.
\end{itemize}
 

\subsection{Simple GDA based Algorithms}
\label{GDA}

The majority of federated optimizers, such as FedSGD, FedAvg~\cite{mcmahan2017communication}, FedProx~\cite{li2018federated}, FedPD~\cite{zhang2020fedpd}, and others, optimize the local optimal minimum value. The federated adversarial domain adaptation, on the other hand, has a more difficult task of converging to a saddle point in a distributed manner.

\citet{peng2019federated} propose FedSGDA algorithm by extending FedSGD with stochastic Gradient Descent Ascent (GDA) in the problem of federated domain adaption.
In order to make the paper self-contained, we summarize FedSGDA in Algorithm~(\ref{FedSGDA}). 
However, due to its single descent/ascension step per communication round, SGDA has a massive communication overhead. Later in the experiments, we observe that FedSGDA requires more than $1000$ rounds of communication.

\begin{algorithm}[t]
\caption{FedSGDA Algorithm \cite{peng2019federated}}
\label{FedSGDA}
\begin{algorithmic}[1]
\REQUIRE $\mathbf{x}^{0}, \eta_1, \eta_2, T$
\FOR{$t=0, \ldots, T-1$}
\FOR{each client $i\in [N]$ \texttt{in parallel}}
\STATE
$\omega_i^{t} = \omega_0^{t}\quad \psi_i^{t} = \psi_0^{t}$
\STATE \# Local Update: 
\STATE
$\omega_i^{t+1}= \omega_i^{t} - \eta\nabla_{\omega_i}f_i(\omega_i^{t},\psi_i^{t}) $
\STATE
$\psi_i^{t+1}= \psi_i^{t} + \eta\nabla_{\psi_i}f_i(\omega_i^{t},\psi_i^{t})$

\ENDFOR
\STATE \# Global Update: 
\STATE
$\omega_0^{t+1}=\frac{1}{N} \sum_{i=1}^{N} \omega_{i}^{t+1}
\quad \psi_0^{t+1}=\frac{1}{N} \sum_{i=1}^{N} \psi_{i}^{t+1}$
\ENDFOR
\end{algorithmic}
\end{algorithm}
\begin{algorithm}[htp!]
\caption{FedAvgGDA/FedProxGDA Algorithm}
\label{alg:FedAvg+}
\begin{algorithmic}[1]
\REQUIRE \text{Initialize} $\omega_0^{0}, \psi_0^{0}, \eta_1, \eta_2, \{M_i\}_{i=0}^N, T$
\FOR{$t=0, \ldots, T-1$}
\FOR{{each} client $i\in [N]$ \texttt{in parallel}}
\STATE \# For FedAvgGDA Algorithm:
\STATE
$\mathcal{L}_i(\omega, \psi)\! =\! f_i(\omega, \psi)  $
\STATE \# For FedProxGDA Algorithm:
\STATE
$\mathcal{L}_i(\omega, \psi)\! =\! f_i(\omega, \psi)  + \frac{\mu}{2}\| \omega - \omega_0^{t} \|_2^2 - \frac{\mu}{2}\|\psi-\psi_0^t\|_2^2$
\STATE
$\widehat\omega_i^{0} = \omega_0^{t}, \quad \widehat\psi_i^{0} = \psi_0^{t}$
\STATE \# Local Update: 
\FOR{$m=0, \ldots, M_i-1$ }
\STATE
$\widehat\omega_i^{m+1} = \widehat\omega_i^m - \eta_1 \nabla_{\omega_i} \mathcal L_i(\widehat\omega_i^m, \widehat\psi_i^m) $
\STATE
$\widehat\psi_i^{m+1}= \widehat\psi_i^m + \eta_2 \nabla_{\psi_i} \mathcal L_i(\widehat\omega_i^m, \widehat\psi_i^m) $
\ENDFOR
\STATE
${\omega}_i^{t+1} = \widehat{\omega}_i^{M_i},\quad {\psi}_i^{t+1} = \widehat{\psi}_i^{M_i} $
\ENDFOR
\STATE \# Global Update:
\STATE
$\omega_0^{t+1}=\frac{1}{N} \sum_{i=1}^{N} \omega_{i}^{t+1},\quad \psi_0^{t+1}=\frac{1}{N} \sum_{i=1}^{N} \psi_{i}^{t+1}$
\ENDFOR
\end{algorithmic}
\end{algorithm}
FedSGDA inspires us to simply extend FedAvg, a more communication efficient scheme, by GDA, resulting in FedAvgGDA, as shown in Algorithm~\ref{alg:FedAvg+}, where the server averages multi-step stochastic gradient descent w.r.t $\omega$ and stochastic gradient ascent w.r.t $\psi$ from all clients. Several works, including \citet{reisizadeh2020robust} and \citet{deng2021local}, use a similar or variant of FedAvgGDA for federated GAN training. \citet{rasouli2020fedgan} use FedAvgGDA as well. However, due to the unique class-imbalance problem in federated adversarial domain adaptation, the inter-client drift of a local models from a multi-step stochastic gradient descent ascent using FedAvgGDA is no longer negligible. As illustrated in Fig.~\ref{fig:oldmethod}. We also extend Fedprox \citep{li2018federated} by GDA, which leads to FedProxGDA in Algorithm~\ref{alg:FedAvg+}.

Motivated by the global consensus constraint in FedPD~\cite{zhang2020fedpd}, we address the problem of model drift from multiple steps of GDA by introducing a separate set of dual variables. The introduction of dual variables is intended to bridge the gradient gap between the distributed optimization and the centralized result.

\section{\indent\textbf{{FedMM Algorithm}}}\label{sec:FedMM}

Due to the distributed constraint in FL systems, the traditional centralized method introduced in Section~\ref{sect:preliminary} cannot perform the minimax optimization of (\ref{FL}).
Simply decomposing (\ref{FL}) into local optimization and global average as in algorithms like FedSGDA, FedAvgGDA, and FedProxGDA results in a servere performance degradation because these distributed training algorithms diverge from the central optimizer in (\ref{FL}), as validated in Fig.~\ref{fig:oldmethod}.
In this section, we look at how to reduce this divergence by reformulating the centralized problem in (\ref{FL}) into the federated saddle-point optimization problem with consensus constraints given by
\begin{equation}
\label{FL-minimax}
\begin{aligned}
\min_{\omega_0, \omega_i }\max_{\psi_0, \psi_i } 
\quad &
f(\omega, \psi) 
    =
\frac{1}{N}\sum_{i=1}^N f_i(\omega_i, \psi_i) \\
\textrm{s.t.} \quad & 
\omega_i = \omega_0 ,\quad \psi_i = \psi_0, \quad \forall i \in [N].
\end{aligned}
\end{equation}
The corresponding augmented Lagrangian form for each client is defined as
\begin{equation}
\label{local}
\begin{split}
   & \mathcal{L}_i(\omega_0, \omega_i, \lambda_i, \psi_0,\psi_i, \beta_i) \\
    \triangleq & f_i(\omega_i, \psi_i) 
    + 
     \left\langle\lambda_i, \omega_i -\omega_0\right\rangle +\frac{\mu_1}{2}\| \omega_i - \omega_0 \|_2^2 \\ 
     &  - \langle\beta_i,\psi_i - \psi_0\rangle-\frac{\mu_2}{2}
    \| \psi_i - \psi_0 \|_2^2.
\end{split}
\end{equation}
The centralized optimization problem in (\ref{FL}) is then transformed into a saddle-point minimax optimization of augmented Lagrangian functions over all primal-dual pairs, i.e., $\{  \omega_i, \omega_0, \lambda_i, \lambda_0, \psi_i, \psi_0, \beta_i, \beta_0\}$ for all clients $i\in [N]$:
\begin{equation}
\label{global}
\begin{split}
&\min_{\omega_0, \omega_i}\max_{\psi_0, \psi_i}  
    \mathcal{L}\left(\{\omega_i\}_{i=0}^N,  \{\psi_i\}_{i=0}^N, \{\lambda\}_{i=1}^N, \{\beta\}_{i=1}^N \right)\\
    \triangleq & 
    \min_{\omega_0, \omega_i}\max_{\psi_0, \psi_i}
    \frac{1}{N}\sum_{i=1}^N\mathcal{L}_i(\omega_0, \omega_i, \psi_0,\psi_i, \lambda_i, \beta_i).
\end{split}
\end{equation}
By fixing the global consensus variables $\{\omega_0, \psi_0\}$, the above problem is separable  w.r.t local pairs $\{\omega_i, \psi_i, \lambda_i, \beta_i \}$ for all $i\in [N]$. And the decomposed task could be independently updated on local clients periodically without global communication. 
The only problem left is to align the update of global consensus $\omega_0, \psi_0$ and local updates $\omega_i, \psi_i$ for all $i\in [N]$.  
Next, we demonstrate how to achieve distributed local updates and align local updates with global consensus.

 By substituting (\ref{local}) into (\ref{global}), we obtain the {augmented Lagrangian functions over all primal-dual parameters}:
\begin{equation}
\label{separable}
\begin{split}
&\min_{\omega_i}\max_{\psi_i} \mathcal{L}\left(\{\omega_i\}_{i=0}^N,  \{\psi_i\}_{i=0}^N, \{\lambda\}_{i=1}^N, \{\beta\}_{i=1}^N \right)  \\
= &\frac{1}{N}\sum_{i=1}^N\min_{\omega_i}\max_{\psi_i} f_i(\omega_i, \psi_i) 
    + 
     \left\langle\lambda_i, \omega_i -\omega_0\right\rangle + \\ &\frac{\mu_1}{2}\| \omega_i - \omega_0 \|_2^2  - \langle\beta_i,\psi_i - \psi_0\rangle-\frac{\mu_2}{2}
    \| \psi_i - \psi_0 \|_2^2.
\end{split}
\end{equation}
The  minimax optimization  w.r.t the global consensus variable $\omega_0$ and $\psi_0$ is given by:
\begin{eqnarray}
\label{minomega0}
\widehat \omega_0 =&
&\hspace{-7mm}\arg\min_{\omega_0}\frac{1}{N}\sum_{i=1}^N f_i(\omega_i, \psi_i) 
    + 
     \left\langle\lambda_i, \omega_i -\omega_0\right\rangle \nonumber\\
&+& \hspace{-3mm}\frac{\mu_1}{2}\| \omega_i - \omega_0 \|_2^2 
- \langle\beta_i,\psi_i - \psi_0\rangle-\frac{\mu_2}{2}
    \| \psi_i - \psi_0 \|_2^2. \nonumber\\
=   &&\hspace{-7mm}\frac{1}{N} \sum_{i=1}^{N}\omega_i + \frac{1}{\mu_2}\lambda_i,
\end{eqnarray}
where the closed-form solution is due to the quadratic optimization. Similarly, we obtain 
\begin{equation}
\label{maxpsi0}
\widehat \psi_0 = 
\frac{1}{N} \sum_{i=1}^{N}\psi_i + \frac{1}{\mu_2}\beta_i.  
\end{equation}
Eqn. (\ref{minomega0}) and (\ref{maxpsi0}) provide guidance for local update alignment with global consensus.
More specifically,
in each round, we optimize  each client's individual  $\omega_i$ and $\psi_i$, by fixing the global consensus constraints ($\omega_0$ and $\psi_0$) and dual parameters  ($\lambda_i$ and $\beta_i$).
Taking the $(t+1)$-th round update as an example.  Client $i$ receives the global parameters $\{\omega_0^t, \psi_0^t \}$ from the server and sets local  parameters $\widehat\omega^0_i=\omega_0^t, \widehat\psi^0_i=\psi_0^t$.\footnote{We use $\{\widehat\omega_i, \widehat\psi_i\}$ to denote the local iterative updates for $\{\omega_i, \psi_i\}$ to differentiate symbols 
of local updates and global updates.}  { {Then, the local saddle-point optimization}} of~(\ref{separable})  w.r.t  $\{\omega_i, \psi_i\}$ is updated by the local GDA: 
\begin{eqnarray}
\label{descent}
    \widehat\omega_i^{m+1} 
 \hspace{-6mm}   &&=
    \widehat\omega_i^{m} - \eta_1\nabla_{\omega_i}\mathcal{L}_i(\widehat\omega_i^{m},\widehat\psi_i^{m})\\
  \hspace{-6mm}    &&= \omega_i^m - \eta_1\left[\nabla_{\omega_i} f_i(\widehat\omega_i^m, \widehat\psi_i^m) + \mu_1(\widehat\omega_i^m - \omega_0^t) + \lambda_i^t \right]\nonumber
\end{eqnarray}
\begin{eqnarray}
\label{ascent}
    \widehat\psi_i^{m+1} 
    \hspace{-6mm}    &&= \widehat\psi_i^{m} + \eta_2\nabla_{\psi_i}
    \mathcal{L}_i(\widehat\omega_i^{m},\widehat\psi_i^{m})\\
    \hspace{-6mm}    &&=
    \widehat\psi_i^m + \eta_2 [\nabla_{\psi_i} f_i(\widehat\omega_i^m, \widehat\psi_i^m) - \mu_2(\widehat\psi_i^m - \psi_0^t) - \beta_i^t ].\nonumber
\end{eqnarray}
We denote
${\omega}_i^{t+1} = \widehat{\omega}_i^{M_i}$
and ${\psi}_i^{t+1} = \widehat{\psi}_i^{M_i}$ for the results of $M_i$-step local update. 
The dual parameters are then updated using GDA with
\begin{eqnarray}
    &\lambda_i^{t+1} 
     = \lambda_i^t + \mu_1(\omega_i^{t+1} - \omega_0^{t})\label{lambda_update}, \\
    &\beta_i^{t+1} 
    = \beta_i^t + \mu_2(\psi_i^{t+1} - \psi_0^{t})\label{beta_update}.
\end{eqnarray}
To align with the global consensus constraint obtained in (\ref{minomega0}) and (\ref{maxpsi0}), we set
\begin{equation}
\begin{split}\label{local_output}
\omega_{i}^{t+}&=\omega_{i}^{t+1}+ \frac{\eta_3^t}{\mu_1} \lambda_i^{t+1};
 \quad
\psi_{i}^{t+} =\psi_{i}^{t+1}+ \frac{\eta_3^t}{\mu_2} \beta_{i}^{t+1}.
\end{split}
\end{equation}

Therefore, the global consensus constraint is satisfied by the global update at the server with
\begin{equation}
\label{agg}
\begin{split}
\omega_0^{t+1}=\frac{1}{N} \sum_{i=1}^{N} \omega_{i}^{t+},
\quad \textrm{and} \quad 
\psi_0^{t+1}=\frac{1}{N} \sum_{i=1}^{N} \psi_{i}^{t+}.
\end{split}
\end{equation}
It should be noted that we use an exponential decay factor $\eta_3\leq 1$ in (\ref{local_output}). We find that $\eta_3$ helps the convergence even when the local training step $M_i$ is insufficient.

We can now summarize one round of the FedMM algorithm, which consists of three major steps:
(i) Parallel saddle-point optimization on all local augmented Lagrangian function $\mathcal{L}_i$'s. One optimization oracle example is based on stochastic GDA, as shown in (\ref{descent}) and (\ref{ascent}). 
{(ii)} Local gradient descent and ascent updates on dual variable ($\{\beta_i, \lambda_i\}$) as shown in  (\ref{beta_update}). {(iii)} Aggregation to update global consensus variables $\{\omega_0, \psi_0\}$ in (\ref{agg}). 
After one round of global communication. The global coordinated value of $\{\omega_0^{t+1}, \psi_0^{t+1} \}$ is then broadcasted back to each  client, triggering next-round updates. The detailed diagram of  FedMM is shown in Fig.~\ref{fig:fedmm} with FedMM algorithm summarized in~Algroithm~\ref{Alg:FedMM}.

\begin{figure}
  \begin{center}
    \includegraphics[width=0.5\textwidth]{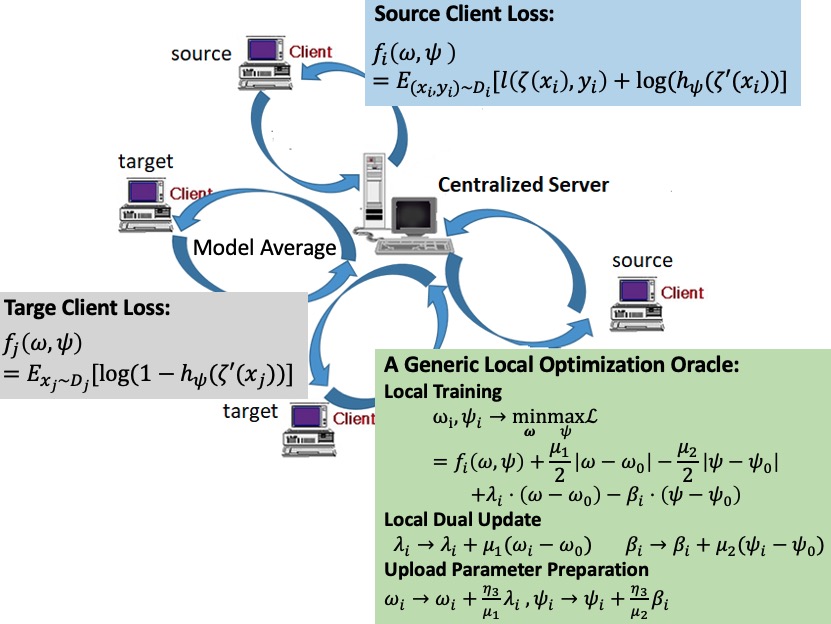}
  \end{center}
  \caption{The flowchart of the proposed FedMM algorithm. The source and target clients have different local minimax objectives due to the source and target domain distributions. In parallel, each client is running its local optimization oracle. Each  client then uploads its parameter to the server. After that, the server averages the client's parameters and broadcasts them back to the client, completing one-round updates.}
  \label{fig:fedmm}
 \end{figure}
\begin{algorithm}[htp!]
\caption{FedMM Algorithm}
\label{Alg:FedMM}
\begin{algorithmic}[1]
\REQUIRE \text{Initialize} $\omega_0^{0}, \psi_0^{0}, \mu_1, \mu_2, \eta_1, \eta_2, \eta_3, \{M_i\}_{i=0}^N, T$
\FOR{$t=0, \ldots, T-1$}
\FOR{{each} client $i\in [N]$ \texttt{in parallel}}


\STATE
$\widehat\omega_i^{0} = \omega_0^{t};\quad 
\quad
\widehat\psi_i^{0} = \psi_0^{t}$
\STATE \# Local Update: 
\FOR{$m=0, \ldots, M_i-1$ }
\STATE \# Gradient Descent:
\STATE
$\widehat\omega_i^{m+1} = \widehat\omega_i^m - \eta_1 [\nabla_{\omega_i} f_i(\widehat\omega_i^m, \widehat\psi_i^m) + \mu_1(\widehat\omega_i^m - \hspace*{12.5mm}\omega_0^t) + \lambda_i^t ]$
\STATE \# Gradient Ascent
\STATE
$\widehat\psi_i^{m+1}= \widehat\psi_i^m + \eta_2 [\nabla_{\psi_i} f_i(\widehat\omega_i^m, \widehat\psi_i^m) - \mu_2(\widehat\psi_i^m - \hspace*{12.5mm}\psi_0^t) - \beta_i^t ]$
\ENDFOR
\STATE
${\omega}_i^{t+1} = \widehat{\omega}_i^{M_i};\quad 
{\psi}_i^{t+1} = \widehat{\psi}_i^{M_i} $
\STATE \# Dual Descent:
\STATE
$\lambda_i^{t+1}=\lambda_i^{t}+\mu_1(\omega_{i}^{t+1}-\omega_0^{t})$
\STATE \# Dual Ascent:
\STATE
$\beta_{i}^{t+1}=\beta_{i}^{t}+\mu_2(\psi_{i}^{t+1}-\psi_0^{t}) 
$
\STATE
$
\omega_{i}^{t+}=\omega_{i}^{t+1}+ \frac{\eta_3^t}{\mu_1} \lambda_i^{t+1};
\hspace{1em}
\psi_{i}^{t+}=\psi_{i}^{t+1}+ \frac{\eta_3^t}{\mu_2} \beta_{i}^{t+1}$
\ENDFOR
\STATE \# Global Update: 
\STATE
$\omega_0^{t+1}=\frac{1}{N} \sum_{i=1}^{N} \omega_{i}^{t+}; \quad 
\quad \psi_0^{t+1}=\frac{1}{N} \sum_{i=1}^{N} \psi_{i}^{t+}$
\ENDFOR
\end{algorithmic}
\end{algorithm}

\section{Convergence Analysis}\label{analysis}
Finding a global saddle
point $\min_{{x} } \max_{{y} } f({x}, {y})$ in general is intractable~\cite{lin2020near}.
One approach is to equivalently reformulate the problem
 by 
$\min _{{x} }\left\{\Phi({x}):=\max _{{y} \in \mathcal{Y}} f({x}, {y})\right\}$,
and define an optimality notion for the local surrogate of global optimum of $\Phi$.
A series of theoretical analyses on the stationary point convergence condition of $\Phi$ with first-order algorithm were carried out to extend the convex-concave assumption to assumptions of nonconvex-strongly-concave\footnote{$f(x,\cdot)$ is not necessarily convex and,  $f(\cdot, y)$ is strongly  concave.}~\cite{rafique2018non, lu2020hybrid}, nonconvex-concave~\cite{lin2020near, nouiehed2019solving}, and nonconvex-nonconcave\footnote{$f(x,\cdot)$  is not necessarily convex and,  $f(\cdot, y)$ is not necessarily concave.}~\cite{jin2020local}.
Convergence analysis for a federated  optimizer, such as FedMM that involves bounding client's drift from global parameter via primal-dual method, on the other hand, is  more complicated.
We establish our main convergence results in this section and show that FedMM converges to the stationary point for the nonconvex-strongly-concave case.

Let $\psi^{\star}(\omega) 
\triangleq \arg\max_{\psi}f(\omega, \psi)$ be the optimal value of $\psi$ for the global objective function $f$ w.r.t $\omega$. Then  (\ref{FL}) can be reformulated as 
    $\min_{\omega}f(\omega, \psi) 
    = \min_{\omega} \frac{1}{N}\sum_{i}\Phi_i(\omega)$
with  
\begin{equation}
\Phi_i(\omega) \triangleq f_i(\omega, \psi^{\star}(\omega)),
\quad
\Phi(\omega) \triangleq \frac{1}{N}\sum_{i=1}^N\Phi_i(\omega) .\end{equation}
In this way, we equivalently reformulate the problem
as 
$\min _{{\omega}}\left\{\Phi(\omega)=\max _{{\phi}} f({\omega}, {\phi})\right\}$. To ease the presentation, we further define the augmented Lagrange of $\Phi_i$ by
\begin{equation}
\label{eqn:AL-Phi-main}
    \mathcal{L}_i^{\Phi}(\omega_i^t, \omega_0^t, \lambda_i^t) 
    =  
    \Phi_i(\omega_i^t)+\langle \lambda_i^t, \omega_i^t-\omega_0^t\rangle + \frac{\mu_1}{2}\left\|\omega_i^t- \omega_0\right\|^2.
\end{equation}
For our theoretical analysis, we make the following standard assumptions that have been used in the literature~\citep{lin2020gradient,jin2020local,luo2020stochastic,lin2020near}:




\begin{assumption}
\label{assump:Lipshictz-gradients}
(Lipshictz continuous gradients)
For all $i\in [N]$, there exists positive constants $L_{11}$, $L_{12}$, $L_{21}$, and $L_{22}$ such that for any $\omega, \omega^\prime\in \mathbb R^{d_1}$, and $\psi, \psi^\prime\in \mathbb R^{d_2}$, we have$
     \left\|\nabla_{\omega}f_i(\omega, \psi) - \nabla_{\omega}f_i(\omega^\prime, \psi) \right\| \leq 
    L_{11}\left\|\omega - \omega^\prime\right\|$,     $\left\|\nabla_{\omega}f_i(\omega, \psi) - \nabla_{\omega}f_i(\omega, \psi^\prime) \right\| \leq 
    L_{12}\left\|\psi - \psi^\prime\right\|$,  
     $\left\|\nabla_{\psi}f_i(\omega, \psi) - \nabla_{\psi}f_i(\omega^\prime, \psi) \right\| \leq 
    L_{21}\left\|\omega - \omega^\prime\right\|$, $    \left\|\nabla_{\psi}f_i(\omega, \psi) - \nabla_{\psi}f_i(\omega, \psi^\prime) \right\| \leq L_{22}\left\|\psi - \psi^\prime\right\|$.
\end{assumption}
\begin{assumption}
\label{assump:Stronglyconcave}
(Strongly concave $f_i(\cdot, \psi_i)$) For all $i\in [N]$,  $f_i(\omega, \psi)$ are strongly concave on $\psi$, i.e., there exists constant $B>0$ such that for any $\omega\in \mathbb R^{d_1}$, and $\psi, \psi^\prime\in \mathbb R^{d_2}$, we have
\begin{equation}
  \left\langle \nabla_\psi f_i(\omega, \psi) -\nabla_\psi f_i(\omega, \psi^\prime), \psi - \psi^\prime \right\rangle  \leq  -B\left\|\psi - \psi^\prime\right\|^2.
\end{equation}
\end{assumption}

\begin{assumption}
\label{assump:sufficient_loc}
(Sufficient local training) For all $i\in [N]$, after $M_i$-step update, the gradients w.r.t. $\omega_i$ and $\psi_i$ are finite and denoted by
\begin{equation}
\label{localerror}
    \|\nabla_{\omega} \mathcal{L}_i(\omega_i^t, \psi_i^t)\| = e_{\omega, i}^t, 
    \|\nabla_{\psi} \mathcal{L}_i(\omega_i^t, \psi_i^t)\| = e_{\psi, i}^t. 
\end{equation}
\end{assumption}
We set $\eta_3 = 1$ for the analysis without loss of generality.
\begin{assumption}
\label{assump:Lipsch}
The $\kappa$-Lipschitz continuity of  $\psi^{\star}(\omega)$, i.e.,
\begin{equation}
     \left\|  \psi^{\star}\left(\omega_i^{t-1}\right) - \psi^{\star}\left(\omega_i^{t}\right)\right\| \leq \kappa \left\|  \omega_i^{t-1} - \omega_i^{t}\right\|.
\end{equation}
\end{assumption}
In the following, we present some key results for FedMM convergence. Proofs of these results are deferred to the Appendix.
Let $\epsilon^t_i = \left\|\psi_i^t -  \psi^{\star}(\omega_i^t) \right\|$. We begin by displaying the upper bound expression of gradient of $\mathcal{L}_i^{\Phi}(\omega_i^t, \omega_0^t, \lambda_i^t)$ in (\ref{eqn:AL-Phi-main}).
\begin{lemma}
\label{lemma:boundedgrad}
After $M_i$-step updates, the gradient of $\mathcal{L}_i^{\Phi}(\omega_i^t, \omega_0^t, \lambda_i^t)$ is bounded by
\begin{equation}
\label{eqn:boundedgrad}
    \left\|\nabla_{\omega_i}\mathcal{L}_i^{\Phi}(\omega_i^t, \omega_0^t, \lambda_i^t)\right\| \leq L_{12}  \epsilon^t_i, \quad \forall i \in [N].
\end{equation}
\end{lemma}
The detailed proofs for this lemma are provided in Appendix~\ref{append-lemma1}. We further obtain an upper bound expression of the sum of the r.h.s of (\ref{eqn:boundedgrad}) in the following lemma.
\begin{lemma}
\label{lemma:boundpsi}
(Bounded optimal gap on $\psi_i^t$)  There exist positive constants $C_1$ , $C_2$  and $C_3$ such that 
\begin{equation}
\begin{split}
&\sum_{t=1}^T \sum_{i=0}^N \left\|\psi_i^{t} - \psi^{\star}(\omega_i^{t-1})\right\|^2 
\leq
 C_1\sum_{t=1}^T\left\|\omega_i^{t-1} - \omega_i^t\right\|^2 \\
&+ C_2\sum_{i = 1}^N\left\|\psi_i^{0} - \psi^{\star}(\omega_i^{0})\right\|^2  + C_3\sum_{i = 1}^N\sum_{\substack{j=1 \\ j\neq i}}^N\left\|\omega_i^0 - \omega_j^0\right\|^2. 
\end{split}
\end{equation}
\end{lemma}
The detailed proofs for this lemma are provided in Appendix~\ref{appned:lemm-boundpsi}. The upper bound of the descent of $\Phi(\omega_0^t)$ after $T$-round updates is  analyzed in the following lemma.
\begin{lemma} 
\label{lemma:descentPhi}
(Descent of $\Phi(\omega_0^t)$) After $T$-round global updates, the descent of $\Phi(\omega_0^t)$ is bounded by
\begin{equation}
\begin{split}
&\Phi(\omega_0^T) - \Phi(\omega_0^0) \leq 
\frac{(3\mu_1+ 16L_{\Phi})L_{12}^2}{\mu_1NL_{\Phi}}\sum_{i=1}^N\sum_{t=1}^T\epsilon^t_i\\
 &- \frac{\mu_1}{4}\sum_{t=1}^T
\left\|\omega_0^{t+1} - \omega_0^{t}\right\|^2 \\
&
-\frac{\mu_1^2-2\mu_1L_{\Phi}-4L_{\Phi}^2}{2\mu_1N}\sum_{i=1}^N\sum_{t=1}^T\left\|\omega_i^{t+1}- \omega_i^t\right\|^2.
\end{split}
\end{equation}
\end{lemma}
Appendix~\ref{appned:lemm-boundpsi} contains detailed proofs for this lemma. Following Lemma~\ref{lemma:boundedgrad}, we can obtain the upper bound of $\Phi(\omega_i^t)$, which is a function of $\epsilon^t_i$.
Then by substituting the results of Lemma~\ref{lemma:boundpsi} and Lemma~\ref{lemma:descentPhi} back to the upper bound of $\Phi(\omega_i^t)$, and following a series of algebraic manipulations, we finally obtain the convergence theorem. Please refer to  Appendix~\ref{subsec:theorem} and Appendix~\ref{append:error-bound} for the proof details of this theorem.
\begin{theorem} 
\label{theorem:convergence}
(Convergence on $\Phi(\omega)$) There exist positive constants $E_1$, $E_2$ $E_3$,  $E_4$, and $\epsilon$ such that after $T$ rounds of global updates, the upper bound for the accumulate descent of $\Phi(\omega_0^t)$ is given by
\begin{equation}
\label{converged}
\begin{split}
    &\Phi(\omega_0^0) - \Psi(\omega_0^T) \leq - E_1 \sum_{t=1}^T\left\|\nabla \Phi\left(\omega_{0}^{t}\right)\right\|^2 + E_4T\epsilon\\
    &+ E_2 \sum_{i=1}^N\left\|\psi_i^{0} - \psi^{\star}(\omega_i^{0})\right\|^2 + E_3 \sum_{i=1}^N\sum_{\substack{j=1 \\ j\neq i}}^N\left\|\omega^0_i - \omega^{0}_j\right\|^2.
\end{split}
\end{equation}
In particular, this implies $\limsup_{t\to\infty}\left\|\nabla \Phi\left(\omega_{0}^{t}\right)\right\| = O(\epsilon)$.
\end{theorem}
\textbf{Remark}~~
Because the l.h.s. of (\ref{converged}) admits a lower bound, so is the r.h.s. As a result, $\limsup_{T\to\infty}\sum_{t=1}^T\left\|\nabla \Phi\left(\omega_{0}^{t}\right)\right\|^2$ must converge, which implies that $\Phi\left(\omega_{0}^{t}\right)$ converges to a $\epsilon$-stationary point.
More specifically,
Dividing both sides of (\ref{converged}) by $T$ and taking $\limsup_{T\to\infty}$, we obtain
\begin{align*}
    \limsup_{T\to\infty}\frac{\sum_{t=1}^T\left\|\nabla \Phi \left(\omega_{0}^{t}\right)\right\|^2}{T} \leq 
    \frac{E_4\epsilon}{E_1},
\end{align*}
which implies that $\sum_{t=1}^T\left\|\nabla \Phi \left(\omega_{0}^{t}\right)\right\|^2 = O(T\epsilon)$ and for sufficiently large $t$, $\left\|\nabla \Phi \left(\omega_{0}^{t}\right)\right\|^2 = O(\epsilon)$. In the special case of $\epsilon = 0$, i.e., strict optimality is obtained at each local client, this result shows that the limiting point is a stationary point. We provide the detailed proof in  Appendix~\ref{append:error-bound}.

\section{Related Work}
FedSGD~\cite{mcmahan2017communication} suggests  one-step local SGD update and then sends the gradients to the server for global update.  
It mimics the centralized SGD training.
The high communication overhead, however, prevents it from being used in practice.
FedAvg~\cite{mcmahan2017communication} is a generalization of FedSGD, proposing multiple-step local SGD  per communication round, with a good accuracy-to-communication trade-off.
However, its accuracy suffers in non-i.i.d.\ scenarios.
Several works have been developed to address non-optimal behavior on non-i.i.d data, including 
 FedProx~\cite{li2018federated}, FedPD~\cite{zhang2020fedpd}, SCAFFOLD~\cite{karimireddy2020scaffold},  FedNova~\cite{wang2020tackling},
  and FedDyn~\cite{acar2021federated}. These works aim to minimize a sum of non-identical functions, where each function can only be accessed locally. 
Moreover, Auto-FedAvg ~\cite{xia2021auto} adjusted weights at the aggregation during training.
These results cannot be directly applied to federated saddle point optimization problems, such as the federated adversarial domain adaptation, which seeks a federated minimax optimization.
 
There are several works~\cite{rasouli2020fedgan, reisizadeh2020robust, deng2021local} that bring the communication efficiency to minimax optimization based on FedAvg, such as the federated GAN~\cite{rasouli2020fedgan}  that uses a binary classification function to distinguish between real and generated data. Because there is no label-imbalanced problem across the training functions of local clients, it works well for distributed GAN learning. This type of FedAvgGDA is sensitive to data imbalance in the federated domain adaptation problem, as demonstrated later in the experiment. Furthermore, the federated robust optimization~\cite{reisizadeh2020robust, deng2021local} differs from the federated adversarial domain adaptation problem in that the set of maximization variables is separable across local client-side functions. These strategies, however, are unsuitable for federated domain adaptation due to structural differences. 
Note that the FLRA in~\citet{reisizadeh2020robust} corresponds to  FedAvgGDA, and its convergence analysis
cannot be directly borrowed to our case because each local client in our study is optimized on the augmented Lagrangian local function rather than the pure local function.
 


\section{Experiments}
\label{Experiments}
On real-world data sets, FedMM is evaluated with  three representative domain adaptation methods:  DANN~\cite{ganin2015unsupervised},  MDD~\cite{zhang2019bridging}, and  CDAN~\cite{long2017conditional}. Please refer to Section \ref{sect:preliminary} for more information on these methods.
Our experiments are primarily concerned with
the training communication overhead and the test accuracy on the label-free target data set\footnote{The code is available at https://github.com/yshen22/fedmm}.

\indent\textbf{Datasets and Source/Target Data Distribution:}
\textbf{MNISTM} \cite{ganin2016domain}
is a dataset that demonstrates domain adaptation by combining MNIST with randomly colored image patches from the BSD500 dataset  \cite{arbelaez2010contour}.
55,000 labeled images from the source domain and 55,000 unlabelled images from the target domain are used for training; and 55,000 images from the target domain are used for testing.

\noindent\textbf{Office-31} \cite{saenko2010adapting} is a typical domain adaptation dataset made up of three distinct domains with 31 categories in each domain. 
There are 4,652 images in total from 31 classes.
We will focus on the worst-case scenario (as analyzed in Fig~\ref{fig:oldmethod}), where the source and target domain data are allocated to different clients for all datasets.



\indent\textbf{Benchmarks:}
We compare FedMM with the FedSGDA in~\cite{peng2019federated}.
Furthermore,  
most existing federated optimizers were designed to solve the loss function minimization, which is unsuitable for adversarial domain adaptation. To make a fair comparison, we extend FedAvg
\cite{mcmahan2017communication} and FedProx \cite{li2018federated} with recently proposed minimax optimizer~\cite{lin2020gradient}
and refer to them  as FedAvgGDA and $\text{FedProxGDA}$ with details explained in Section~\ref{GDA} and summarized 
 in Algorithm~\ref{alg:FedAvg+}.


\begin{figure*}[htb!]
\centering
\begin{subfigure}{.33\textwidth}
  \centering
\includegraphics[width=1.0\linewidth]{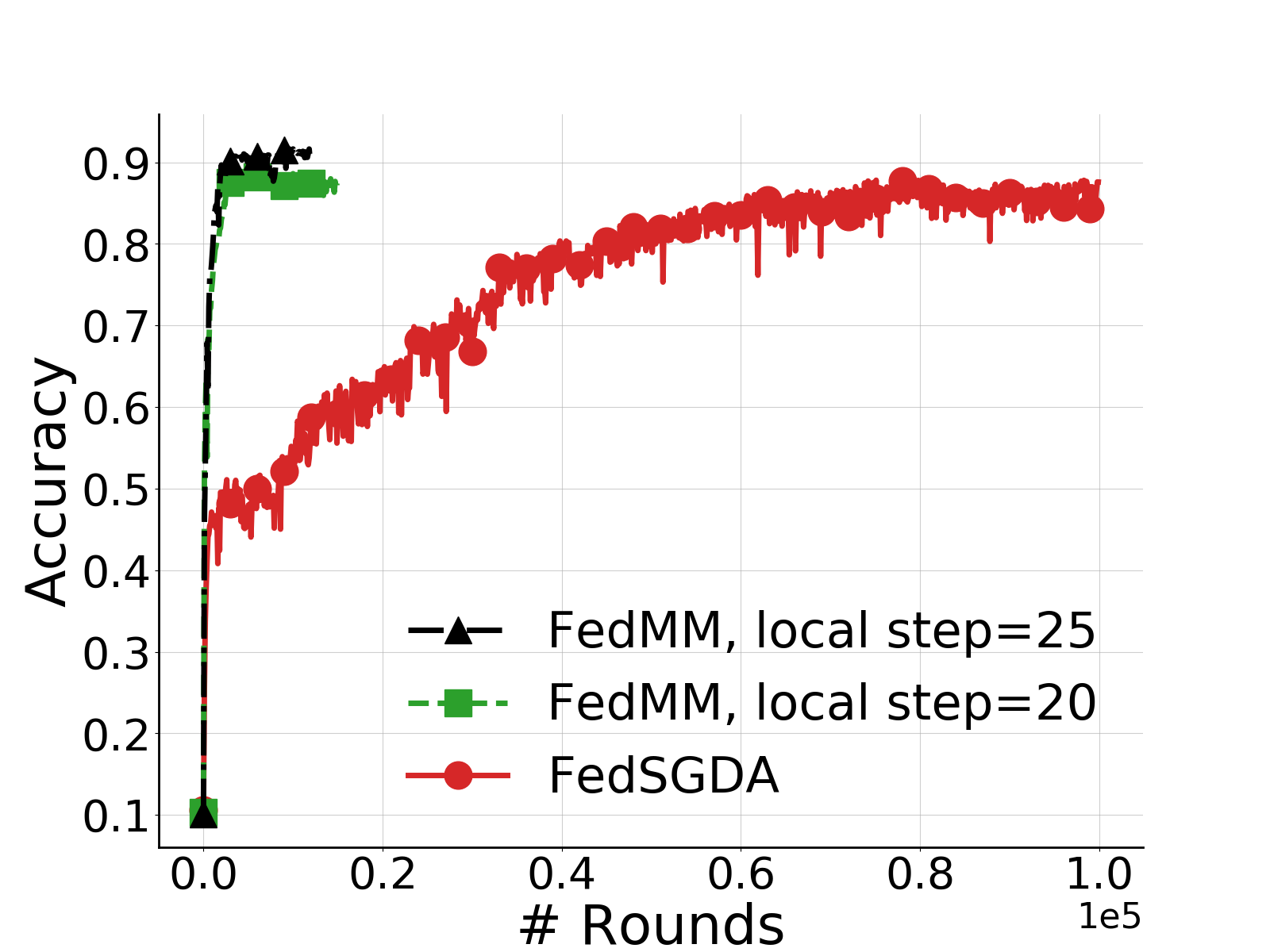}
  \caption{DANN}
  \label{fig:rate_DANN}
\end{subfigure}%
\begin{subfigure}{.33\textwidth}
  \centering
\includegraphics[width=1.0\linewidth]{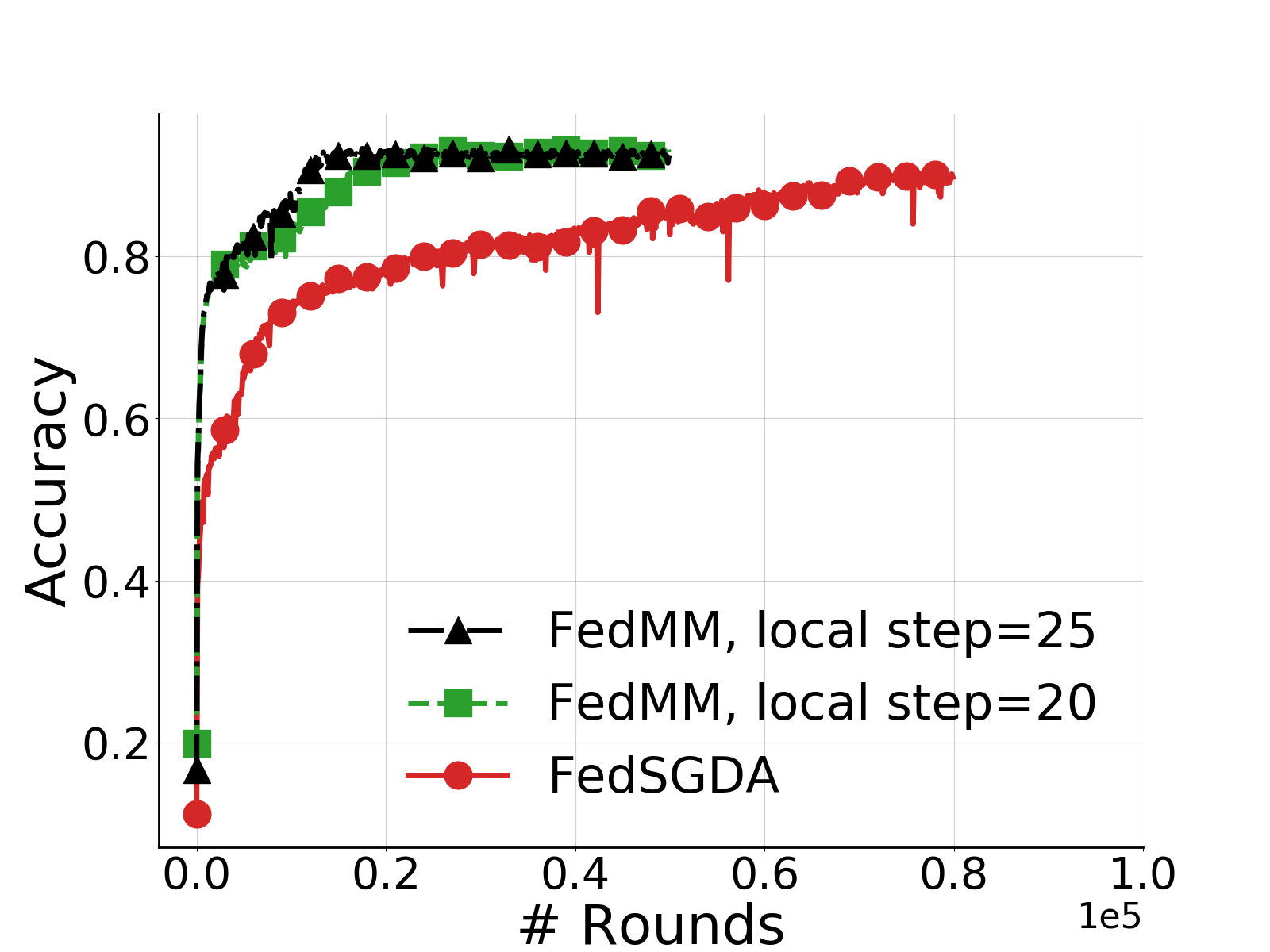}
  \caption{MDD}
  \label{fig:rate_MDD}
\end{subfigure}
\begin{subfigure}{.33\textwidth}
  \centering
  \includegraphics[width=1.0\linewidth]{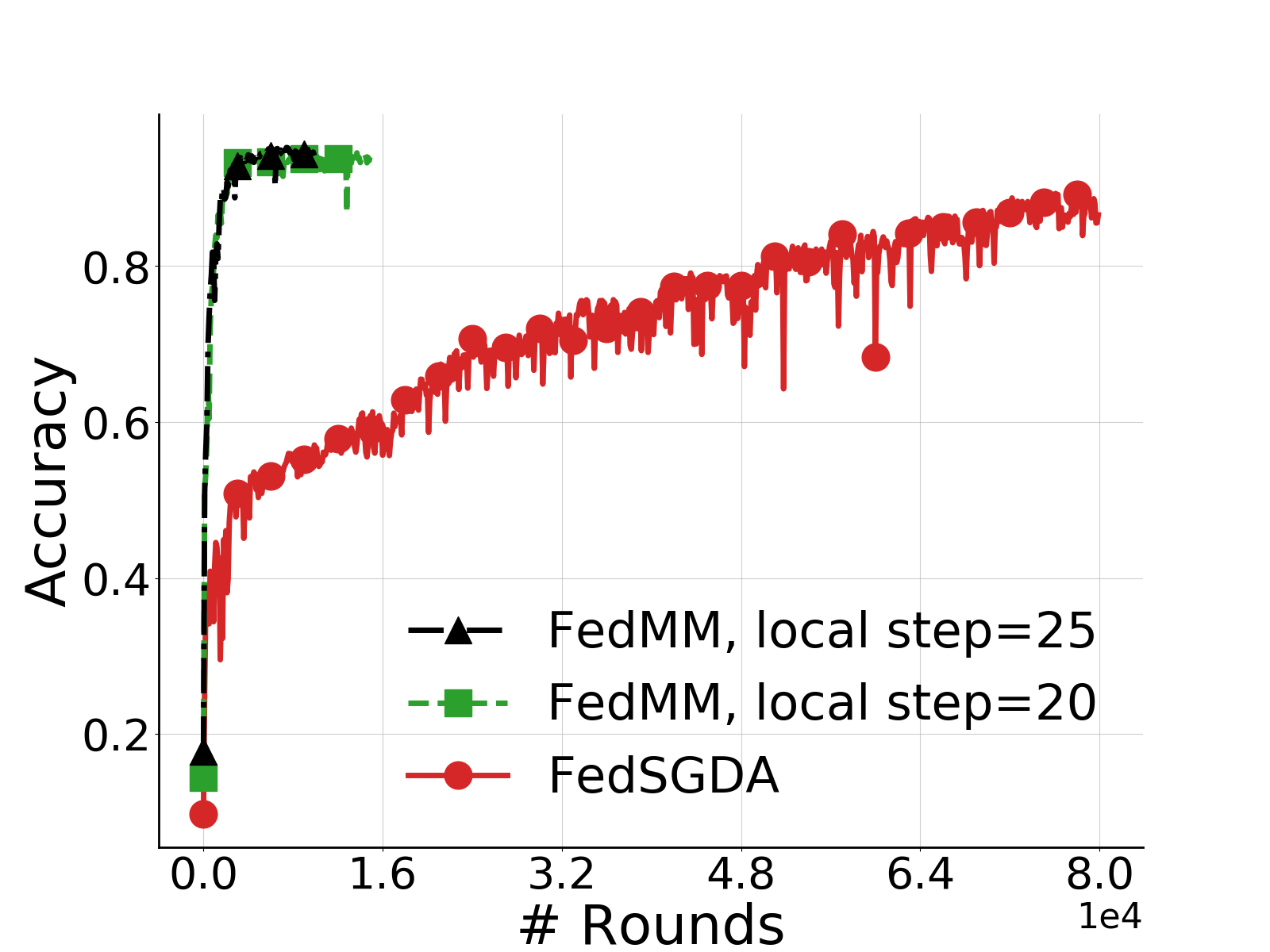}
  \caption{CDAN}
  \label{fig:rate_CDAN}
\end{subfigure}%
\caption{Comparisons of convergence for the proposed FedMM with FedSGDA~\cite{peng2019federated} based on different adversarial domain networks, i.e., DANN, MDD, and CDAN.}
\label{fig:rate}
\end{figure*}
\begin{figure*}[htb!]
\centering
\begin{subfigure}{.33\textwidth}
  \centering
  \includegraphics[width=1.0\linewidth]{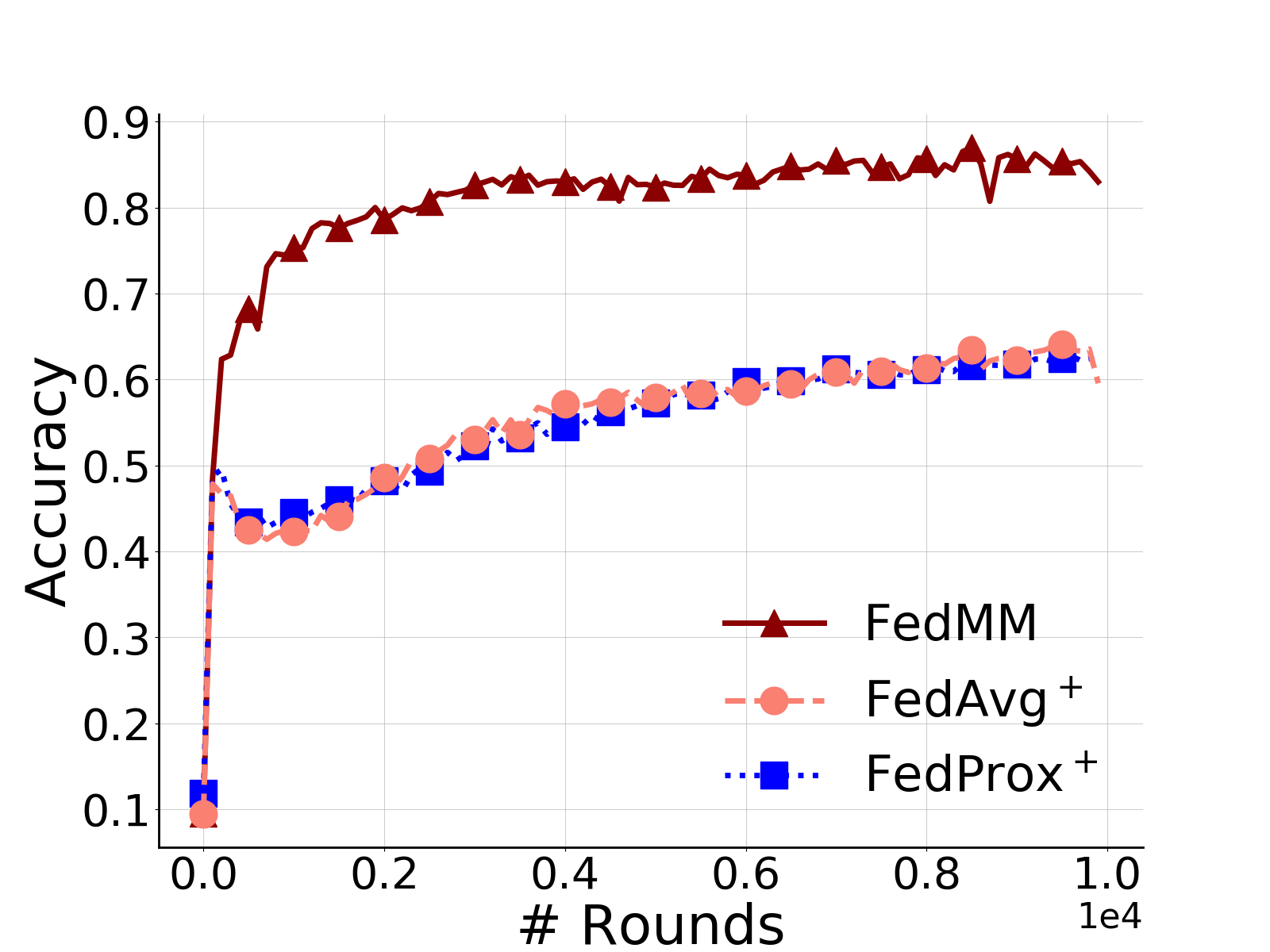}
  \caption{DANN: 1-source/1-target clients}
  \label{fig:acc_DANN}
\end{subfigure}%
\begin{subfigure}{.33\textwidth}
  \centering
  \includegraphics[width=1.0\linewidth]{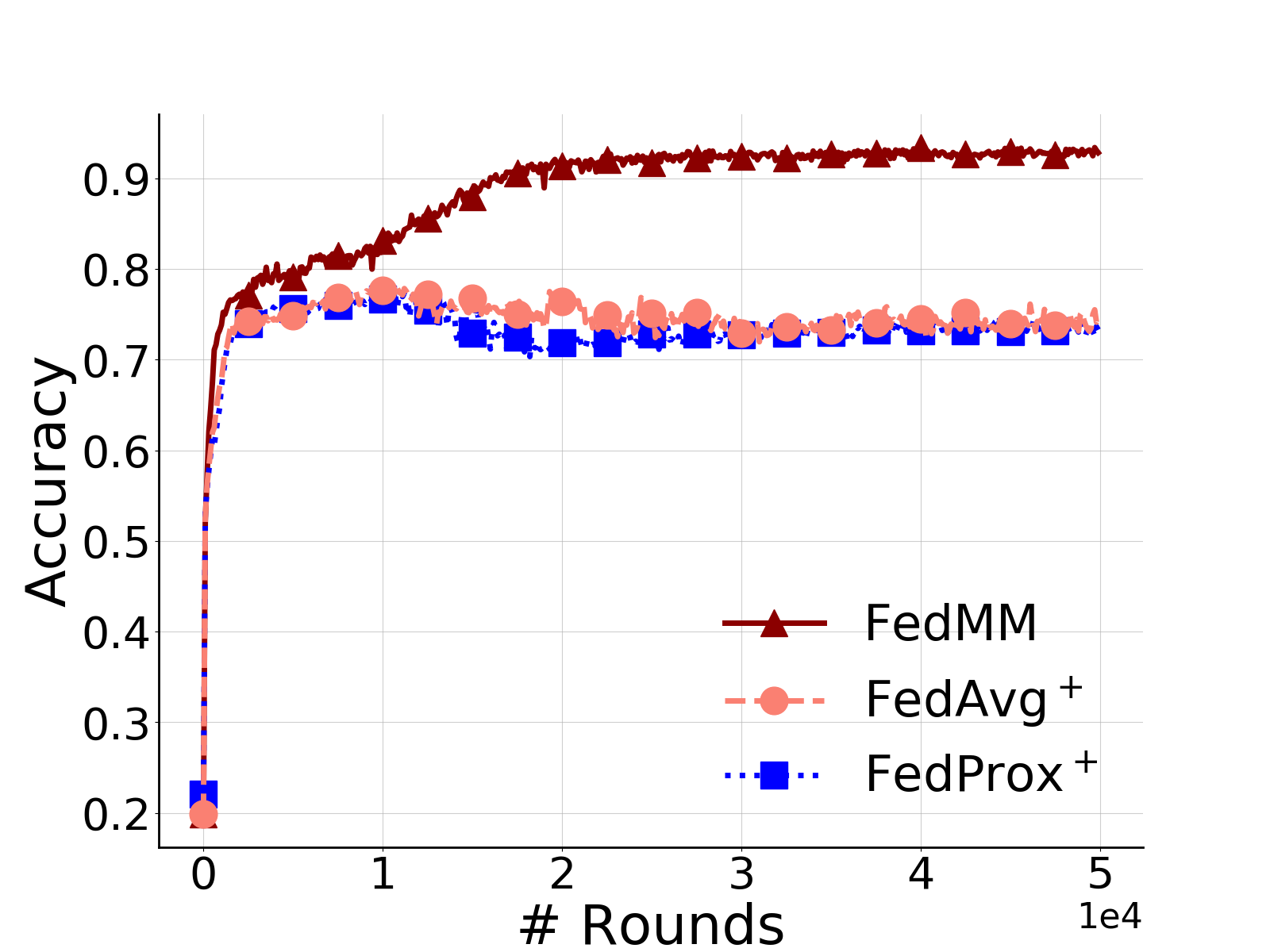}
  \caption{MDD: 1-source/1-target clients}
  \label{fig:acc_MDD}
\end{subfigure}
\begin{subfigure}{.33\textwidth}
  \centering
    \includegraphics[width=1.0\linewidth]{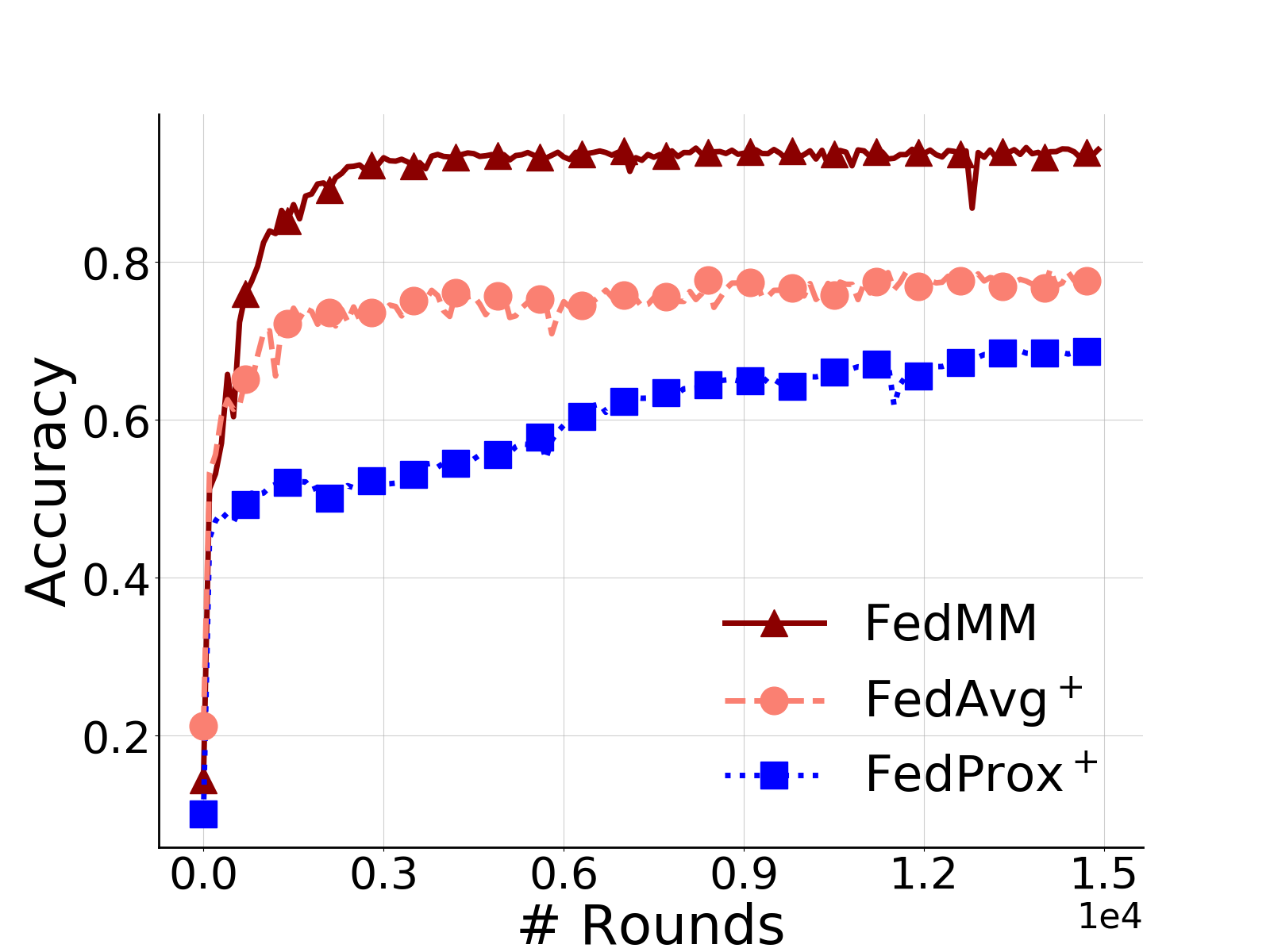}
  \caption{CDAN: 1-source/1-target clients}
  \label{fig:acc-CDAN}
\end{subfigure}%

\begin{subfigure}{.33\textwidth}
  \centering
  \includegraphics[width=1.0\linewidth]{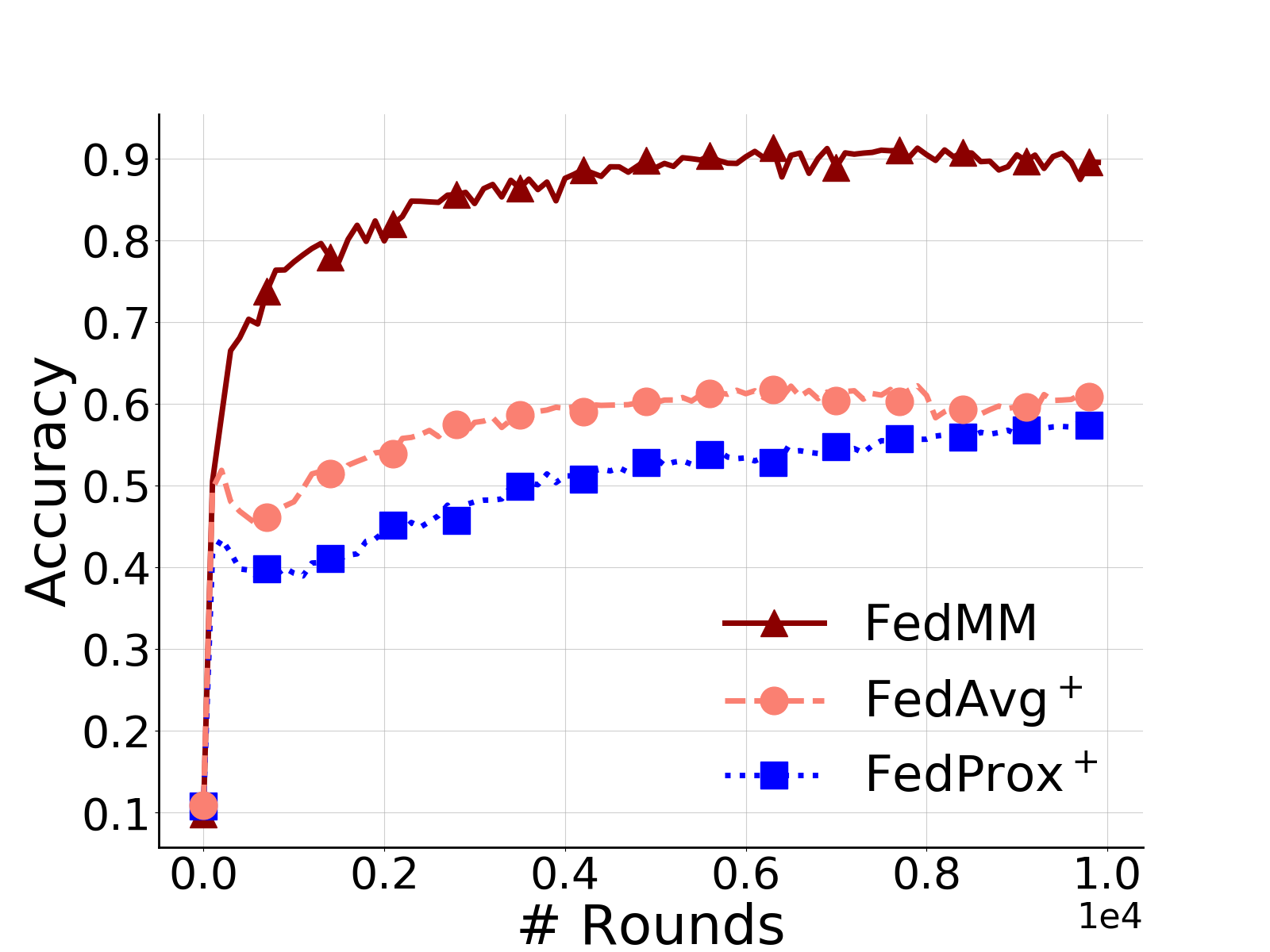}
  \caption{DANN: 1-source/2-target clients}
  \label{fig:acc_DANN_s1t2}
\end{subfigure}%
\begin{subfigure}{.33\textwidth}
  \centering
  \includegraphics[width=1.0\linewidth]{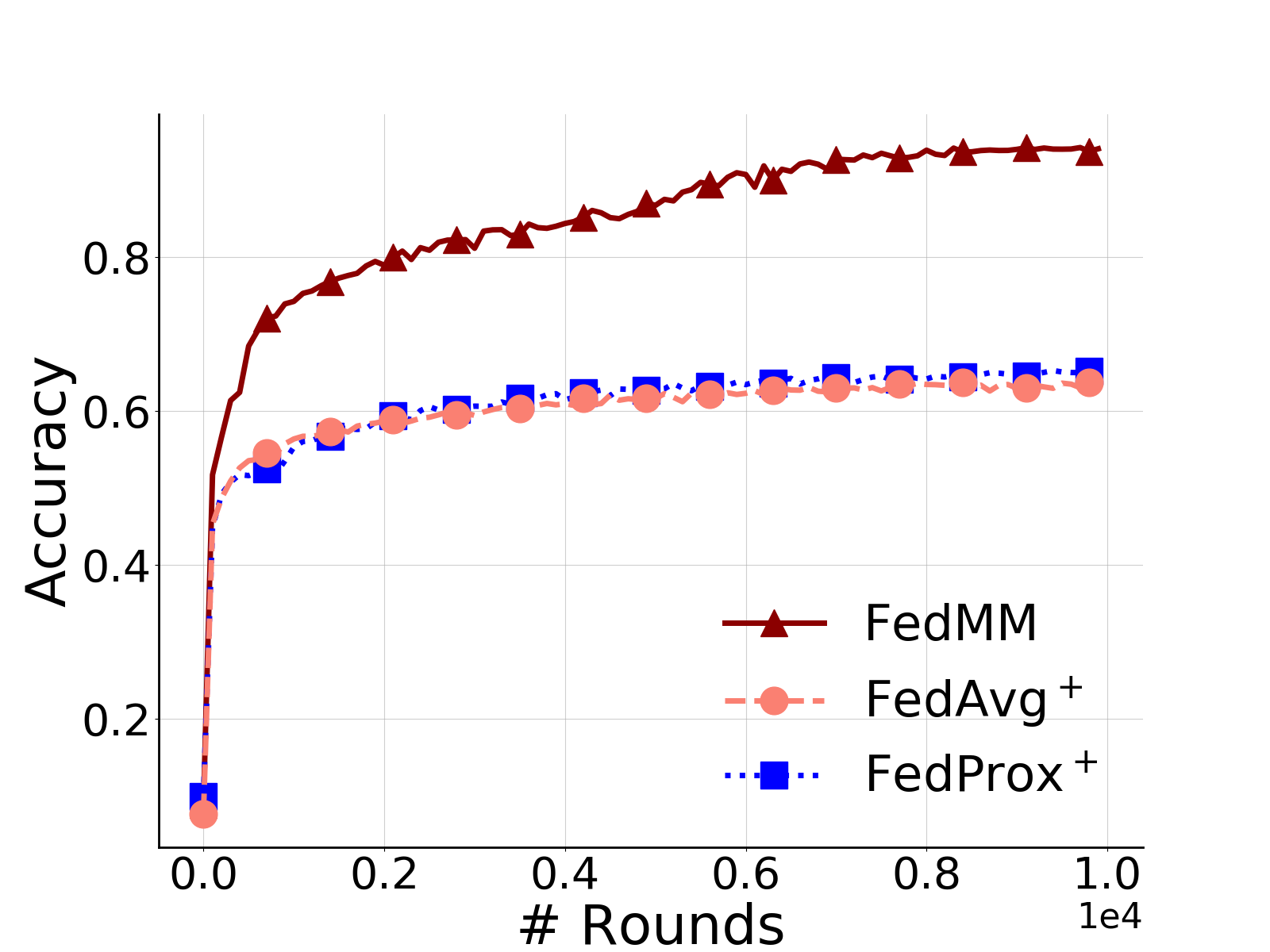}
  \caption{MDD: 1-source/2-target clients}
  \label{fig:acc-CDAN}
\end{subfigure}%
\begin{subfigure}{.33\textwidth}
  \centering
  \includegraphics[width=1.0\linewidth]{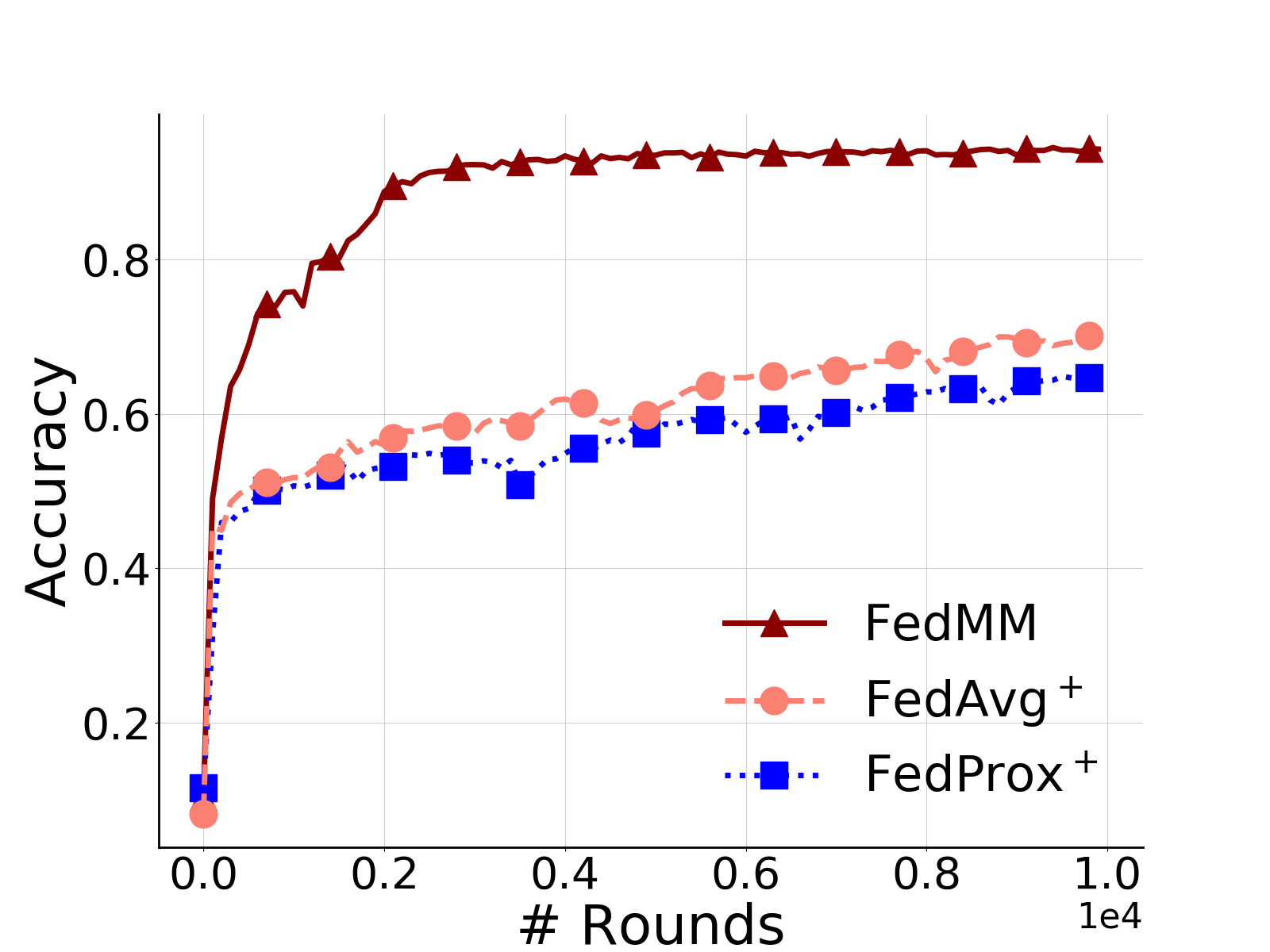}
  \caption{CDAN: 1-source/2-target clients}
  \label{fig:acc_MDD}
\end{subfigure}
\begin{subfigure}{.33\textwidth}
  \centering
  \includegraphics[width=1.0\linewidth]{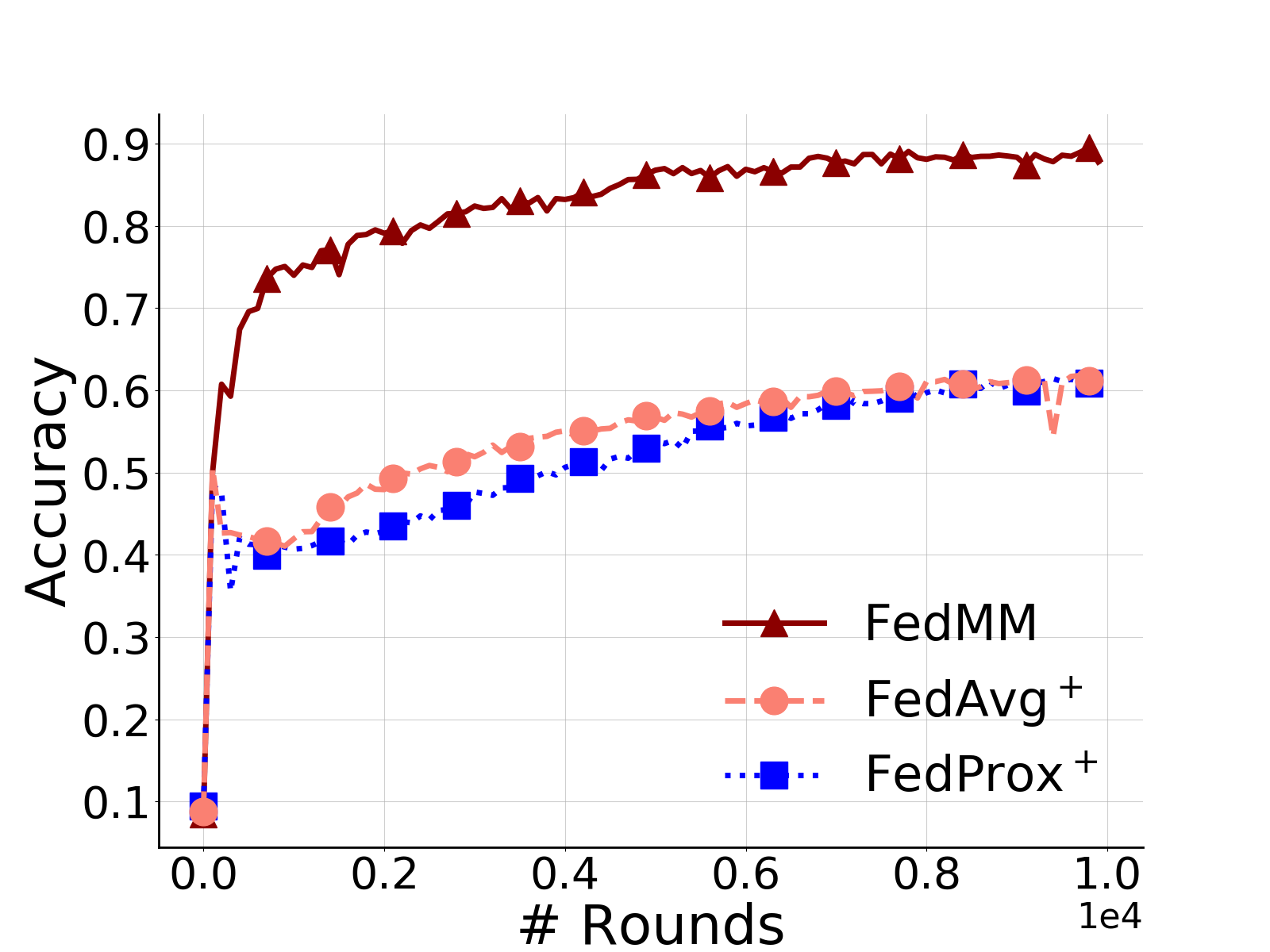}
  \caption{DANN: 2-source/1-target clients}
  \label{fig:acc_DANN_s2t1}
\end{subfigure}
\begin{subfigure}{.33\textwidth}
  \centering
  \includegraphics[width=1.0\linewidth]{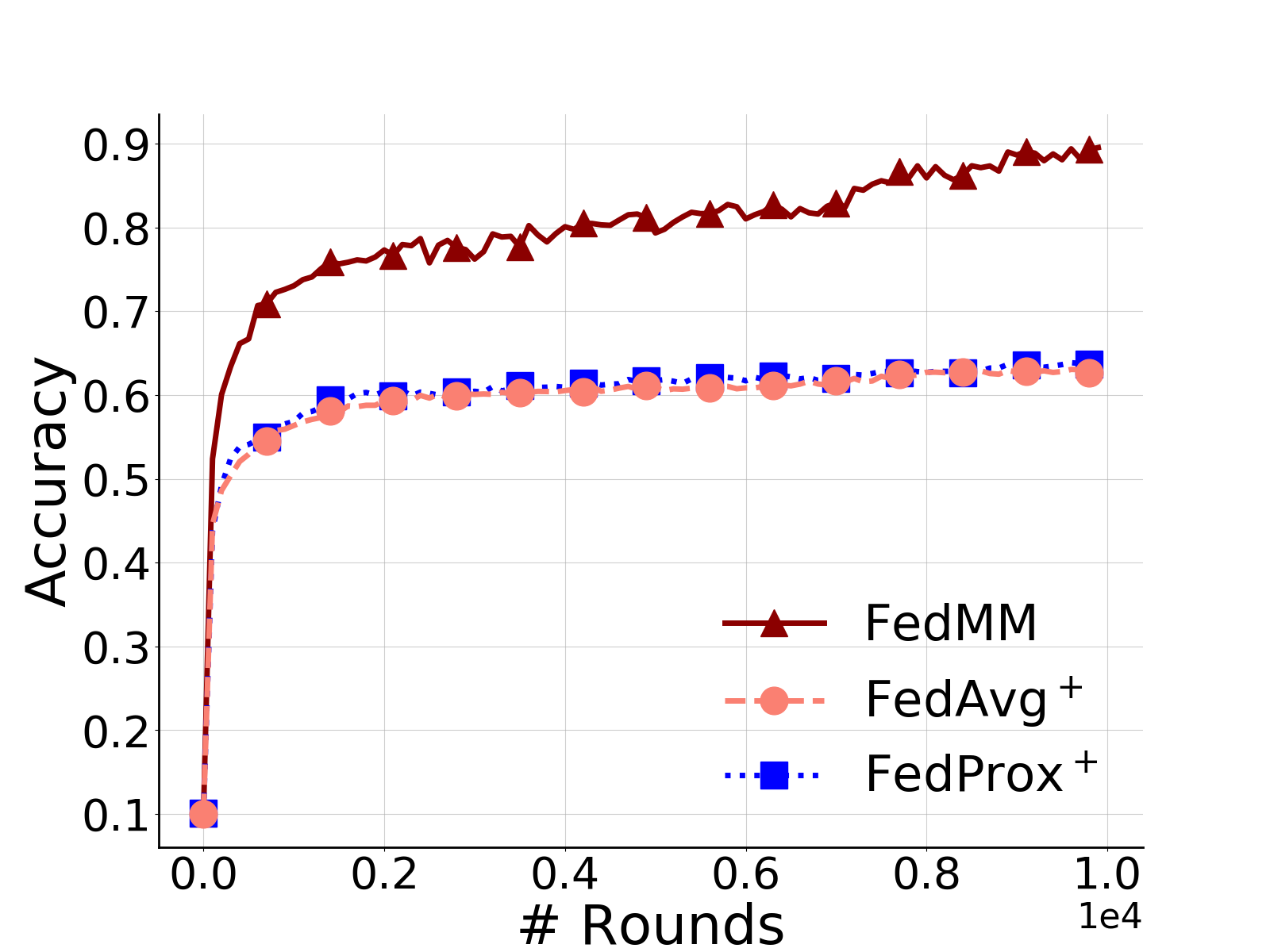}
  \caption{MDD: 2-source/1-target clients}
  \label{fig:acc_DANN}
\end{subfigure}%
\begin{subfigure}{.33\textwidth}
  \centering
  \includegraphics[width=1.0\linewidth]{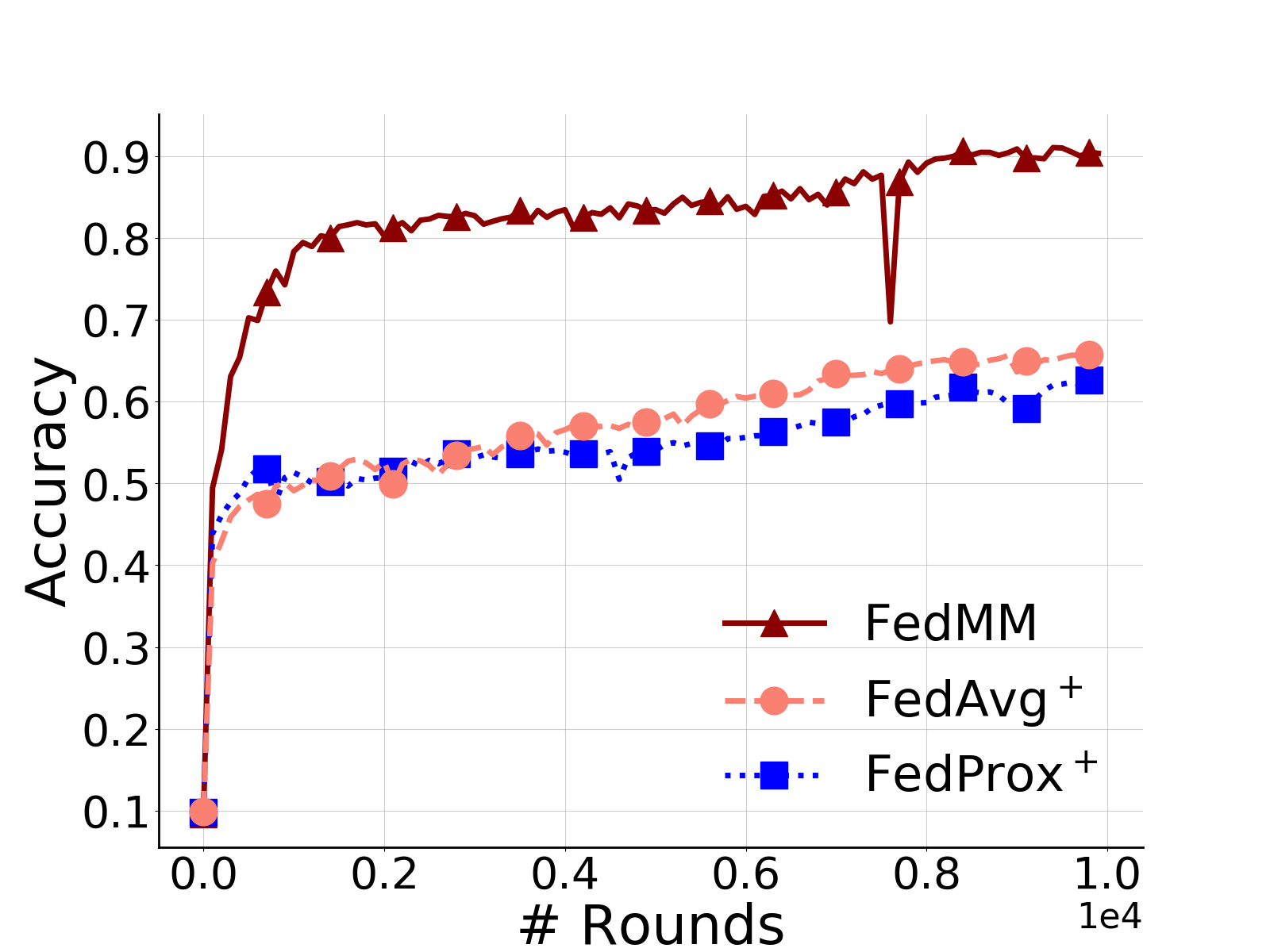}
  \caption{CDAN: 2-source/1-target clients}
  \label{fig:acc-CDAN}
\end{subfigure}%
\caption{Comparisons of convergence for the proposed FedMM with FedAvgGDA and FedProxGDA. }
\label{fig:acc}
\end{figure*}

\begin{table*}[t]
\centering
\caption{Unsupervised target domain test accuracy ($\%$) on Office-31.}
\begin{adjustbox}{width=0.8\textwidth}
\begin{tabular}{|c|c|c|c|c|c|c|c|c|c|}
\hline 
\multirow{2}{*}{} & \multicolumn{3}{c|}{FedAvgGDA} & \multicolumn{3}{c|}{{\text{FedSGDA}}} & \multicolumn{3}{c|}{{\textbf{FedMM}}}
\tabularnewline
\cline{2-4} \cline{5-7} \cline{8-10}
 & {DANN} & {MDD}& {CDAN} & {DANN} & {MDD}& {CDAN} & {DANN} & {MDD} & {CDAN}\tabularnewline
\hline 
\hline 
{A$\to$W} & {60.1} & {73.2} & {62.9} & {60.3} & {76.4} & 55.3 & {65.5} & {{79.7}} & 64.7 \tabularnewline
\hline 
{D$\to$W} & {86.1} & {93.6} & 86.8& {84.9} & {94.7} & 83.4 & {89.6} & {95.9} & 93.4 \tabularnewline
\hline 
{W$\to$ D} & {93.6} & {97.8} & {94.2} & {93.7} & {98.3} & 94.0 & {96.7} & {98.5}& 94.0 \tabularnewline
\hline 
{A$\to$ D} & {63.5} & {72.1} & 65.1 & {65.3} & {75.3} & 67.7 & {67.8} & {78.8} & 66.9\tabularnewline
\hline 
{D$\to$ A} & {33.7} & {47.9} & 40.3 & {36.9} & {49.2} & 47.1 & {44.3} & {60.3} & 51.4 \tabularnewline
\hline 
{W$\to$ A} & {40.5} & {51.7} & 45.5 & {40.3} & {52.6} & 43.3 & {48.7} & 55.5 & {59.6} \tabularnewline
\hline 
{Average} & 62.9 & 72.7 & 65.8 & 63.5 & {74.4} & {65.1} & \textbf{68.7} & \textbf{78.1} & \textbf{{71.7}}
\tabularnewline
\hline
\end{tabular}\hfill{}
\end{adjustbox}
\label{table_acc}
\end{table*}


\begin{table*}[htb!]
\centering
\caption{Communication rounds ($\times 100$) for training on Office-31.}
\begin{adjustbox}{width=0.78\textwidth}
\begin{tabular}{|c|c|c|c|c|c|c|c|c|c|}
\hline 
\multirow{2}{*}{} & \multicolumn{3}{c|}{FedAvgGDA} & \multicolumn{3}{c|}{{\text{FedSGDA}}} & \multicolumn{3}{c|}{{\textbf{FedMM}}}
\tabularnewline
\cline{2-4} \cline{5-7} \cline{8-10}
 & {DANN} & {MDD}& {CDAN} & {DANN} & {MDD}& {CDAN} & {DANN} & {MDD} & {CDAN}\tabularnewline
\hline 
\hline 
{A$\to$ W} & {10} & {31} & {17} & {59} & {255} & 78 & {13} & {23} & 29 \tabularnewline
\hline 
{D$\to$ W} & {13} & {27} & 13& {40} & {188} & 49 & {15} & 18 & 9 \tabularnewline
\hline 
{W$\to$ D} & {8} & {11} & {7} & {29} & {92} & 16 & {14} & 19& 10 \tabularnewline
\hline 
{A$\to$ D} & {7} & {22} & 21 & {56} & {400} & 13 & {7.5} & 22 & 32\tabularnewline
\hline 
{D$\to$ A} & {24} & {31} & 34 & {48} & {300} & 95 & {39} & 19 & 17 \tabularnewline
\hline 
{W$\to$ A} & {18} & {18} & 14 & {88} & {321} & 85 & {25} & 15 & 13 \tabularnewline
\hline 
{Average} & {{13.3}} & {{23.3}} & {{17.7}} & {{53.3}} & {259.3} & {{65.1}} & \textbf{{18.9}} & \textbf{19.3} & \textbf{{18.3}}
\tabularnewline
\hline
\end{tabular}\hfill{}
\end{adjustbox}
\label{table_round}
\end{table*}
\indent\textbf{Networks:}
On MNISTM, we use a three-layer  convolutional network as the invariant feature extractor. On Office-31, we use the pre-trained MobileNetV2 \cite{sandler2018mobilenetv2} on ImageNet \cite{russakovsky2015imagenet} as the feature extractor.
Both the task classifier and the domain classifier are  two-layer fully-connected neural networks. 

\indent\textbf{Hyper-parameters:}
The dual variables, i.e., $\{\beta_i, \lambda_i \}$,  are set to ${0}$ at the start of training, $\mu_1$ and $\mu_2$ are set to $1.0$ during all training settings. During local training, the learning rate $\eta_1$ is fixed to $0.01$. In the experiment of training from scratch on MNISTM, $\eta_2=0.01$, $\mu_1=1$ and $\mu_2=1$. In the experiment of training from pre-trained model  on Office-31, we set customized layer-wise learning rate. In details, the learning rate of feature extractors is set as $0.0005$, and $\eta_2=0.1, 0.04$ {and} $0.1 $  for MDD, DANN and CDAN methods, respectivly. Besides, $\nu=0.1$ for MDD, $\nu=0.25$ for DANN and CDAN methods. The rest learning rate are all fixed at $0.01$. For exponential decay parameter, we set $\eta_3=1.0001^{-1}$ for MDD, $\eta_3=1.0002^{-1}$ for CDAN and $\eta_3=1.0005^{-1}$ for DANN.  

\indent\textbf{Data Distribution:} Fig.~\ref{fig:oldmethod} has already demonstrated that as the degree of inter-client domain shift (label imbalance) increases, federated learning performance degrades significantly.
As a result, we will focus on the worst-case scenario, in which the source domain data and target domain data are allocated to different clients separately, i.e., $p=1.0$, to verify the effectiveness of FedMM in the experiments.

\indent\textbf{Performance of Training from Scratch}

We begin by examining the convergence property of our proposed FedMM algorithm when it is trained from scratch on MNISTM.

Fig.~\ref{fig:rate} compares the global communication rounds of our proposed FedMM to FedSGDA. 
We compare FedMM with $M_i= 20$, and $M_i= 25$. Thanks to the local multi-steps minimax optimization at each client, FedMM has a quick convergence rate saving more than  $90\%$ communication rounds compared to FedSGDA to achieve similar test accuracy. Furthermore, the FedMM convergence rate can be improved by increasing the local steps of primal and dual ascent descent.

In Fig.~\ref{fig:acc}, the convergence property of our proposed FedMM is further compared with other representative federated training algorithms with multiple local descent and ascent updates, namely FedAvgGDA and FedProxGDA with $M_i= 20$ for different number of source/target clients settings. While both the FedAvgGDA and FedProxGDA algorithms converge, FedMM consistently outperforms them in terms of test accuracy for all three widely used domain adaptation methods. The results clearly show that FedMM has a superior test accuracy for training from scratch with more than $20\%$ accuracy improvement.

This enormous improvement is understandable given that the FedMM is intended to bridge the gap between the distributed local model and the global model through distributed consensus in the minimax optimization context.
Because of the unique structure of federated adversarial domain adaptation, when the source and target data are distributed across different clients, model drift becomes a severe problem (validated in Fig.~\ref{fig:oldmethod}), which did not occur in any previous federated learning problems in the literature.

\indent\textbf{Performance of Training from Pre-trained Models}

We further examine how the proposed FedMM algorithm performs with the pre-trained MobileNetV2 as a feature extractor. In this part, all the experiments are conducted on Office-31.
Test accuracy and training communication rounds using FedMM, FedAvgGDA, and FedSGDA for commonly used domain adaptation methods are included in Table~\ref{table_acc} and Table~\ref{table_round}, respectively.
Note that
FedMM's performance improvement  is reduced when compared to the training from scratch case in Fig.~\ref{fig:acc}. This is because feature extractor parameters in this pre-trained models have approached optimal values. Nevertheless,  
we take the best average results of FedAvgGDA and FedSGDA (averaged over all tasks) for DANN, MDD, and CDAN  and compare them to FedMM. As highlighted in Table~\ref{table_acc}, FedMM improves by $8.2\%$, $5.4\%$, and $9.0\%$ for DANN, MDD and CDAN, respectively.
Besides, both FedAvgGDA and FedMM cost much less  communication rounds than FedSGDA. However,  FedMM does not have a significant communication advantage over FedAvgGDA due to the additional dual variables.

\section{Conclusions}
We propose FedMM for federated adversarial domain adaptation in this paper. FedMM is designed specifically for federated minimax optimizations with non-separable minimization and maximization variables, as well as clients with uneven label class distributions. We show that  FedMM ensures convergence for clients by using both supervised source domain data and unsupervised target domain data. Experiments show that FedMM outperforms state-of-the-art algorithms in terms of communication rounds and test accuracy on various benchmark datasets.
It outperforms other methods by around a $20\%$ improvement in accuracy over the same communication rounds when training from scratch, and it also
clearly outperforms other methods when training from pre-trained models.

\newpage
\nocite{langley00}


\bibliographystyle{mlsys2022}

\appendix
\onecolumn
\section{Appendix: Convergence Analysis for FedMM in Algorithm~\ref{Alg:FedMM}}
Because the proof is lengthy, we begin by demonstrating convergence to the stationary point by assuming sufficient local training is obtained to ensure local convergence (Section~\ref{sec:property}-Section~\ref{subsec:theorem}).
This assumption is further removed with the results being extended to the convergence proof with bounded local convergence error, as shown in Section~\ref{append:error-bound}.


\subsection{Notation}
\label{sec:notations}
Let $\psi^{\star}(\omega) 
$ be the optimal value of $\psi$ for the global objective function $f$ for $\omega$, which is given by
\begin{equation}
\label{psistar}
\psi^{\star}(\omega) 
\triangleq \arg\max_{\psi}f(\omega, \psi).
\end{equation}
Then  (\ref{FL}) is reformulated as 
    $\min_{\omega}f(\omega, \psi) 
    = \min_{\omega} \frac{1}{N}\sum_{i}\Phi_i(\omega)$
with  
\begin{equation}
\Phi_i(\omega) \triangleq f_i(\omega, \psi^{\star}(\omega)),\quad \text{and}
\quad
\Phi(\omega) \triangleq \frac{1}{N}\sum_{i=1}^N\Phi_i(\omega) .\end{equation}
In this way, we equivalently reformulate the problem
as 
$\min _{{\omega}}\left\{\Phi(\omega)=\max _{{\phi}} f({\omega}, {\phi})\right\}$.
We further define the augmented Lagrange of $\Phi_i$ by
\begin{equation}
\label{eqn:AL-Phi}
    \mathcal{L}_i^{\Phi}(\omega_i^t, \omega_0^t, \lambda_i^t) 
    =  
    \Phi_i(\omega_i^t)+\langle \lambda_i^t, \omega_i^t-\omega_0^t\rangle + \frac{\mu_1}{2}\left\|\omega_i^t- \omega_0\right\|^2.
\end{equation}

In Table~\ref{table:natations}, some notations are further defined to represent  some commonly used computations in the proof. 
\setlength{\arrayrulewidth}{0.4mm}
\setlength{\tabcolsep}{6pt}
\renewcommand{\arraystretch}{1.5}
\begin{table}[htb!]
\caption{
Main notations}
\label{table:natations}
\centering
{
\centering

\begin{tabular}
{ |p{5.5cm}|p{5.5cm}|  }
\hline
Notation & Explanation  \\
\hline
$ \bar\psi_t = a_t = \sum_{i = 1}^N\sum_{j \neq i} \left\|\psi_i^t - \psi_j^t\right\|^2 $ & Average deviation among  $\psi_i^t$'s.\\
$\widetilde\psi_t = b_t = \sum_{i = 1}^N \|\psi_i^t - \psi_i^{t-1}\|^2$ &  Average update increment  for $\psi_i^{t}$.  \\
$\epsilon_t=\sum_{i = 1}^N \left\|\psi_i^t - \psi^{\star}(\omega_i^t)\right\|^2$ & Average distance to optimum for  $\psi_i^t$.  \\
$\widetilde\omega_t = d_t  = \sum_{i = 1}^N \left\|\omega_i^t - \omega_i^{t-1}\right\|^2$ & Average update increment for  $\omega^t_i$.  \\
$\bar \omega_t=e_t = \sum_{i = 1}^N\sum_{j \neq i} \left\|\omega_i^t - \omega_j^t\right\|^2$ & Average deviation among $\omega^t_i$.  \\

\hline
\end{tabular}
}
\end{table}

\subsection{Assumptions}
\label{sec:assumptions}
\begin{assumption}
\label{assump:Lipshictz-gradients}
(Lipshictz continuous gradients)
For all $i\in [N]$, there exists positive constants $L_{11}$, $L_{12}$, $L_{21}$, and $L_{22}$ such that for any $\omega, \omega^\prime\in \mathbb R^{d_1}$, and $\psi, \psi^\prime\in \mathbb R^{d_2}$, we have
\begin{align*}
    &\left\|\nabla_{\omega}f_i(\omega, \psi) - \nabla_{\omega}f_i(\omega^\prime, \psi) \right\| \leq 
    L_{11}\left\|\omega - \omega^\prime\right\|,  \quad \left\|\nabla_{\omega}f_i(\omega, \psi) - \nabla_{\omega}f_i(\omega, \psi^\prime) \right\| \leq 
    L_{12}\left\|\psi - \psi^\prime\right\|, \\
    &\left\|\nabla_{\psi}f_i(\omega, \psi) - \nabla_{\psi}f_i(\omega^\prime, \psi) \right\| \leq 
    L_{21}\left\|\omega - \omega^\prime\right\|,  \quad \left\|\nabla_{\psi}f_i(\omega, \psi) - \nabla_{\psi}f_i(\omega, \psi^\prime) \right\| \leq L_{22}\left\|\psi - \psi^\prime\right\|. \\
\end{align*}
\end{assumption}
\begin{assumption}
\label{assump:Stronglyconcave}
(Strongly concave $f_i(\cdot, \psi_i)$) For all $i\in [N]$,  $f_i(\omega, \psi)$ are strongly concave on $\psi$, i.e., there exists constant $B>0$ such that for any $\omega\in \mathbb R^{d_1}$, and $\psi, \psi^\prime\in \mathbb R^{d_2}$, we have
\begin{equation}
  \left\langle \nabla_\psi f_i(\omega, \psi) -\nabla_\psi f_i(\omega, \psi^\prime), \psi - \psi^\prime \right\rangle  \leq  -B\left\|\psi - \psi^\prime\right\|^2.
\end{equation}
\end{assumption}

\begin{assumption}
\label{assump:sufficient_loc}
(Sufficient local training) For all $i\in [N]$, after $M_i$-step update, the gradients w.r.t. $\omega_i$ and $\psi_i$ are finite and denoted by
\begin{equation}
\label{localerror}
    \|\nabla_{\omega} \mathcal{L}_i(\omega_i^t, \psi_i^t)\| = e_{\omega, i}^t,  \quad
    \|\nabla_{\psi} \mathcal{L}_i(\omega_i^t, \psi_i^t)\| = e_{\psi, i}^t, \quad\forall  t\in [T]. 
\end{equation}
\end{assumption}
We set $\eta_3 = 1$ for the analysis without loss of generality.
\begin{assumption}
\label{assump:Lipsch}
The $\kappa$-Lipschitz continuity of  $\psi^{\star}(\omega)$, i.e.,
\begin{equation}
     \left\|  \psi^{\star}\left(\omega_i^{t-1}\right) - \psi^{\star}\left(\omega_i^{t}\right)\right\| \leq \kappa \left\|  \omega_i^{t-1} - \omega_i^{t}\right\|, \quad\forall t\in [T].
\end{equation}
\end{assumption}

\subsection{Basic Properties of FedMM}
\label{sec:property}


\begin{proposition}
\label{prop:psi_iter}
In Algorithm~\ref{Alg:FedMM},
the following update of $\sum_{i=1}^N \psi_i^t$ is valid for all $t\in [T]$:
\begin{equation}
\label{psi_iter}
\sum_{i=1}^N \psi_i^{t+1} = \sum_{i=1}^N \psi_i^t + \frac{1}{\mu_2}\sum_{i=1}^N  \nabla_{\psi} f_i\left(\omega^{t+1}_i, \psi^{t+1}_i\right) . 
\end{equation}
\end{proposition}
\begin{proof}
Applying Assumption~\ref{assump:sufficient_loc} to (\ref{ascent}) and replace $\widehat\omega^m$ as well as $\widehat\psi^m$ with $\omega^{t+1}$ as well as $\widehat\psi^{t+1}$ respectively, we obtain:
\begin{equation}
\label{local_optimal1}
\nabla_{\psi_i}\mathcal{L}_i(\omega_i^{t+1},\psi_i^{t+1}) =
\nabla_{\psi} f_i\left(\omega_i^{t+1}, \psi_i^{t+1}\right) - \mu_2 \left(\psi_i^{t+1} - \psi_0^{t}\right) - \beta_i^t =0.
\end{equation}
By further making a summation for all $i\in [N]$,  we have
\begin{equation}
\label{local_optimal3}
\sum_{i=1}^N
\nabla_{\psi} f_i\left(\omega_i^{t+1}, \psi_i^{t+1}\right) - 
\mu_2 \sum_{i=1}^N\psi_i^{t+1}
+\mu_2 N\psi_0^{t}
=
\sum_{i=1}^N \beta_i^t.
\end{equation}

In addition,
we get the following equation by substituting (\ref{local_output}) into (\ref{agg}), which is given by
\begin{equation}
    \psi_0^t =  \frac{1}{N} \sum_{i=1}^N \left(\psi_i^t + \frac{1}{\mu_2}\beta_i^t\right).
\end{equation}
We finally prove (\ref{psi_iter}) by substituting the above equation into
 (\ref{local_optimal3}).
\end{proof}
\begin{proposition}
\label{prop:omega_iter}
In Algorithm~\ref{Alg:FedMM},
the following update of $\sum_{i=1}^N \omega_i^t$ is valid for all $t\in [T]$:
\begin{equation}
\label{omega_iter}
\sum_{i=1}^N \omega_i^{t+1} = \sum_{i=1}^N \omega_i^t - \frac{1}{\mu_1}\sum_{i=1}^N  \nabla_{\omega} f_i\left(\omega^{t+1}_i, \psi^{t+1}_i\right) . 
\end{equation}
\end{proposition}
\begin{proof}
The proof procedure is similar to that for Proposition~\ref{prop:psi_iter}.  The following are the specifics.
Applying Assumption~\ref{assump:sufficient_loc} to (\ref{descent}) and replace $\widehat\omega^m$ as well as $\widehat\psi^m$ with $\omega^{t+1}$ as well as $\widehat\psi^{t+1}$ respectively, we obtain:
\begin{equation}
\label{local_optimal1}
\nabla_{\omega_i}\mathcal{L}_i(\omega_i^{t+1},\psi_i^{t+1}) =
\nabla_{\omega} f_i\left(\omega_i^{t+1}, \psi_i^{t+1}\right) + \mu_1 \left(\psi_i^{t+1} - \psi_0^{t}\right) + \lambda_i^t =0.
\end{equation}
By further making a summation for all $i\in [N]$, we have
\begin{equation}
\label{local_optimal2}
\sum_{i=1}^N
\nabla_{\omega} f_i\left(\omega_i^{t+1}, \psi_i^{t+1}\right) + 
\mu_1 \sum_{i=1}^N\omega_i^{t+1}
-\mu_1 N\omega_0^{t}
=
-\sum_{i=1}^N \lambda_i^t.
\end{equation}
In addition, by substituting (\ref{local_output}) into (\ref{agg}), the following holds:
\begin{equation}
    \omega_0^t =  \frac{1}{N} \sum_{i=1}^N \left(\omega_i^t + \frac{1}{\mu_1}\lambda_i^t\right).
\end{equation}
We finally prove (\ref{omega_iter}) by substituting the above equation into (\ref{local_optimal2}).
\end{proof}

\begin{proposition}
\label{prop:psi_step}
In Algorithm~\ref{Alg:FedMM}, the update of $\psi_i^{t}$ holds true for 
all  $i\in [N]$ and $t\in [T]$:
\begin{equation}
\label{delta_phi}
\mu_2 \left(\psi_i^{t+1} - \psi^t_0\right) = \nabla_{\psi} f_i\left(\omega^{t+1}_i, \psi^{t+1}_i\right) -  \nabla_{\psi} f_i\left(\omega^t_i, \psi^t_i\right).
\end{equation}
\end{proposition}
\begin{proof}
Applying Assumption~\ref{assump:sufficient_loc} to (\ref{ascent}) and replace $\widehat\omega^m$ as well as $\widehat\psi^m$ with $\omega^{t+1}$ as well as $\widehat\psi^{t+1}$ respectively, we obtain:
\begin{equation}
\label{local_optimal}
    \nabla_{\psi} f_i\left(\omega_i^{t+1}, \psi_i^{t+1}\right) - \mu_2 \left(\psi_i^{t+1} - \psi_0^{t}\right) - \beta_i^t =0. 
\end{equation}
By substituting the $\beta_i^t$'s update equation in (\ref{beta_update}) 
into (\ref{local_optimal}), we have
\begin{equation}
\label{betar+1}
    \nabla_{\psi}f_i\left(\omega_i^{t+1}, \psi_i^{t+1}\right) - \beta_i^{t+1} =0. 
\end{equation}
By replacing $t+1$ with $t$ in the preceding equation, we get
\begin{equation}
\label{betar}
    \nabla_{\psi}f_i(\omega_i^{t}, \psi_i^{t}) - \beta_i^{t} =0.
\end{equation}
By subtracting (\ref{betar}) from (\ref{betar+1}), we arrive at
\begin{equation}
    \beta_i^{t+1} - \beta_i^{t} = \nabla_{\psi}f_i\left(\omega_i^{t+1}, \psi_i^{t+1}\right) - \nabla_{\psi}f_i\left(\omega_i^{t}, \psi_i^{t}\right).
\end{equation}
By substituting $\beta_i^{t+1}-\beta_i^t$ in (\ref{beta_update}) to the l.h.s of the above equation, we have proved (\ref{delta_phi}).
\end{proof}
\begin{proposition}
\label{prop:delta_omega}
In Algorithm~\ref{Alg:FedMM}, the update of $\omega_i^{t}$ holds true for 
all  $i\in [N]$ and $t\in [T]$:
\begin{equation}
\label{delta_omega}
    \mu_1(\omega_i^{t+1} - \omega_0^t) = \nabla_{\omega} f_i(\omega_i^{t}, \psi_i^{t}) - \nabla_{\omega} f_i\left(\omega_i^{t+1}, \psi_i^{t+1}\right). 
\end{equation}
\end{proposition}
\begin{proof}
The proof procedure is similar to that for Proposition~\ref{prop:psi_step}.  The following are the specifics.
Applying Assumption~\ref{assump:sufficient_loc} to (\ref{descent}) and replace $\widehat\omega^m$ as well as $\widehat\psi^m$ with $\omega^{t+1}$ as well as $\widehat\psi^{t+1}$ respectively, we obtain:
\begin{equation}
\label{local_optimal2}
    \nabla_{\omega}f_i\left(\omega_i^{t+1}, \psi_i^{t+1}\right) + \mu_1(\omega_i^{t+1} - \omega_0^t) + \lambda_i^t = 0.
\end{equation}
By substituting the $\lambda_i^t$'s the $\beta_i^t$'s update equation in (\ref{lambda_update})
into (\ref{local_optimal2}), we have 
\begin{equation}
\label{lambdait+1}
    \nabla_{\omega} f_i\left(\omega_i^{t+1}, \psi_i^{t+1}\right) + \lambda_i^{t+1} = 0.
\end{equation}
By replacing $t+1$ with $t$ in the preceding equation, we get
\begin{equation}
\label{lambdait}
    \nabla_{\omega} f_i(\omega_i^{t}, \psi_i^{t}) + \lambda_i^{t} = 0.
\end{equation}
By subtracting (\ref{lambdait+1}) from (\ref{lambdait}), we arrive at
\begin{equation}
    \lambda_i^{t+1} - \lambda_i^t = \nabla_{\omega} f_i(\omega_i^{t}, \psi_i^{t}) - \nabla_{\omega} f_i\left(\omega_i^{t+1}, \psi_i^{t+1}\right).
\end{equation}
By substituting $\lambda_i^{t+1} - \lambda_i^t$ in (\ref{lambda_update}) to l.h.s of the above equation, we have proved (\ref{delta_omega}).
\end{proof}

\subsection{Proof of Lemma~\ref{lemma:boundedgrad}}
\label{append-lemma1}
We prove Lemma~\ref{lemma:boundedgrad} as follows. We repeat Lemma~\ref{lemma:boundedgrad} in the following Lemma~\ref{lemma:append-boundedgrad} to make the appendix self-contained.
\begin{lemma}
\label{lemma:append-boundedgrad}
After $M_i$-step updates, the gradient of $\mathcal{L}_i^{\Phi}(\omega_i^t, \omega_0^t, \lambda_i^t)$ is bounded by
\begin{equation}
\label{eqn:boundedgrad}
    \left\|\nabla_{\omega_i}\mathcal{L}_i^{\Phi}(\omega_i^t, \omega_0^t, \lambda_i^t)\right\| \leq L_{12}  \epsilon^t_i, \quad \forall i \in [N].
\end{equation}
\end{lemma}
\begin{proof}
After $M_i$-step updates, we have
\begin{equation}
    \nabla_{\omega_i}\mathcal{L}_i(\omega_i^t, \omega_0^t, \lambda_i^t) =0,
\end{equation}
which implies that
\begin{equation}
    \nabla_{\omega}f_i(\omega_i^t, \psi_i^t) + \lambda_i^t + \mu_1(\omega_i^t - \omega_0^t) = 0.
\end{equation}
Since  $\Phi_i$ is differentiable with $\nabla \Phi_i(\omega) = \nabla_{\omega} f_i(\omega, \psi^{\star}(\omega))$ and from Danskin’s theorem~\cite{rockafellar2015convex},  we have
\begin{equation}
\begin{split}
  \left\|\nabla_{\omega_i}\mathcal{L}_i^{\Phi}(\omega_i^t, \omega_0^t, \lambda_i^t)\right\| &=  \left\|\nabla \Phi_i(\omega_i^t) + \lambda_i^t + \mu_1(\omega_i^t - \omega_0^t)\right\| \\
  & = \left\|\nabla \Phi_i(\omega_i^t) -  \nabla_{\omega}f_i(\omega_i^t, \psi_i^t) \right\| \\
  & = \left\|\nabla_{\omega} f_i(\omega_i^t, \psi^{\star}(\omega_i^t)) - \nabla_{\omega}f_i(\omega_i^t, \psi_i^t) \right\|.
\end{split}
\end{equation}
From $L_{12}$-Lipschitz of $\nabla_{\omega}f_i(\omega, \psi)$ on $\psi$, we have
\begin{equation}
  \left\|\nabla_{\omega_i}\mathcal{L}_i^{\Phi}(\omega_i^t, \omega_0^t, \lambda_i^t)\right\| =  \left\|\nabla_{\omega} f_i(\omega_i^t, \psi^{\star}(\omega_i^t)) - \nabla_{\omega}f_i(\omega_i^t, \psi_i^t) \right\| \leq L_{12}\left\|\psi_i^t- \psi^{\star}(\omega_i^t)\right\|.
\end{equation}
\end{proof}
\subsection{Proof of Lemma~\ref{lemma:boundpsi}}
\label{appned:lemm-boundpsi}
\begin{lemma}
\label{lemma:psiconverge}
In Algorithm~\ref{Alg:FedMM}, the following inequality holds for $t=1,\ldots, T$.
 \begin{equation}
\epsilon_t
 \leq 
 \frac{8\mu_2}{NB} \bar\psi_{t-1} 
 +
  \frac{2\mu_2}{2\mu_2 + B}\epsilon_{t-1}+ \frac{8\mu_2}{B} \kappa^2 \widetilde{\omega}_t\\
 +  \frac{8(\mu_2+L_{22})^2}{N\mu_2B}\bar{\psi}_t
 +   \frac{8L_{21}^2}{N\mu_2B}\bar{\omega}_t. 
 \end{equation}
\end{lemma}
\begin{proof}
Considering the fact that 
$\frac{1}{N}\sum_{i=1}^N  \nabla_{\psi} f_i\left(\omega^{t}_i, \psi^{t}_i\right) 
=\nabla_{\psi} f\left(\omega^{t}_i, \psi^{t}_i\right)$, we  reformulate (\ref{psi_iter}) in  Proposition~\ref{prop:psi_iter} and after some tedious algebra manipulations, we obtain 
\begin{equation}
\begin{split}
\psi_i^{t} - \psi^{\star}\left(\omega_i^{t}\right)
=&  \psi_i^{t-1} - \psi^{\star}\left(\omega_i^{t}\right) + 
\frac{1}{\mu_2}\nabla_{\psi}f\left(\omega_i^{t}, \psi_i^{t}\right) \\
&+  
\frac{1}{N\mu_2}\sum_{\substack{j=1 \\ j\neq i}}^N \left(\nabla_{\psi}f_j\left(\omega_j^{t}, \psi_j^{t}\right) - \nabla_{\psi}f_j\left(\omega_i^{t}, \psi_i^{t}\right)\right) \\
&+ \frac{1}{N\mu_2} \sum_{\substack{j=1 \\ j\neq i}}^N\left(
- \mu_2\left(\psi^{t-1}_i - \psi^{t-1}_j\right) + \mu_2\left(\psi^{t}_i - \psi^{t}_j \right)    \right).
\end{split}
\end{equation} 
By taking norm on both sides of the above equation and  considering the triangle inequality, we have
\begin{equation}
\label{psi_iter1}
\begin{split}
\left\|\psi_i^{t} - \psi^{\star}\left(\omega_i^{t}\right) - \frac{1}{\mu_2}
\nabla_{\psi}f\left(\omega_i^{t}, \psi_i^{t}\right) \right\|
\leq 
&\left\|\psi_i^{t-1} - \psi^{\star}\left(\omega_i^{t}\right)\right\|  \\
&+ \frac{1}{N\mu_2}\sum_{\substack{j=1 \\ j\neq i}}^N\left\|\nabla_{\psi}f_j(\omega_j^{t}, \psi_j^{t}) - \nabla_{\psi}f_j\left(\omega_i^{t}, \psi_i^{t}\right)\right\| \\
&+ \frac{1}{N} \sum_{\substack{j=1 \\ j\neq i}}^N\left\|\psi^{t-1}_i - \psi^{t-1}_j\right\| + \frac{1}{N} \sum_{\substack{j=1 \\ j\neq i}}^N\left\|\psi^{t}_i - \psi^{t}_j \right\|.
\end{split}
\end{equation}

Besides, the $B$-strongly concavity of $f(\cdot, \psi)$ implies
\begin{equation}
\label{B-strongcov}
\sqrt{\frac{\mu_2+B}{\mu_2}} \left\|  \psi_i^{t} - \psi^{\star}\left(\omega_i^{t}\right)\right\|
\leq
\left\|\psi_i^{t} - \psi^{\star}\left(\omega_i^{t}\right) - \frac{1}{\mu_2}
\nabla_{\psi}f\left(\omega_i^{t}, \psi_i^{t}\right) \right\| .    
\end{equation}
Sum of  the two preceding inequalities along the same sign direction leads to
\begin{equation}
\begin{split}
    \sqrt{\frac{\mu_2+B}{\mu_2}} \left\|  \psi_i^{t} - \psi^{\star}\left(\omega_i^{t}\right)\right\|
    \leq &
    \left\|\psi_i^{t-1} - \psi^{\star}\left(\omega_i^{t}\right)\right\|  \\
&+ \frac{1}{N\mu_2}\sum_{j \neq i}\left\|\nabla_{\psi}f_j(\omega_j^{t}, \psi_j^{t}) - \nabla_{\psi}f_j\left(\omega_i^{t}, \psi_i^{t}\right)\right\| \\
&+ \frac{1}{N} \sum_{\substack{j=1 \\ j\neq i}}^N\left\|\psi^{t-1}_i - \psi^{t-1}_j\right\| + \frac{1}{N} \sum_{\substack{j=1 \\ j\neq i}}^N\left\|\psi^{t}_i - \psi^{t}_j \right\|.
\end{split}
\end{equation}
We further obtain
\begin{equation}
\begin{split}
\left\|  \psi_i^{t} - \psi^{\star}\left(\omega_i^{t}\right)\right\|
\leq 
&
    \sqrt{\frac{\mu_2}{\mu_2+B}}[\left\|\psi_i^{t-1} - \psi^{\star}\left(\omega_i^{t}\right)\right\|  \\
&+\frac{1}{N\mu_2} \sum_{\substack{j=1 \\ j\neq i}}^N\left\|\nabla_{\psi}f_j(\omega_j^{t}, \psi_j^{t}) - \nabla_{\psi}f_j\left(\omega_i^{t}, \psi_i^{t}\right)\right\| \\
&+ \frac{1}{N} \sum_{\substack{j=1 \\ j\neq i}}^N\left\|\psi^{t-1}_i - \psi^{t-1}_j\right\| + \frac{1}{N} \sum_{\substack{j=1 \\ j\neq i}}^N\left\|\psi^{t}_i - \psi^{t}_j \right\|]. 
\end{split}
\end{equation}
Furthermore, the Lipschitz continuity properties in  Assumption~\ref{assump:Lipshictz-gradients} imply the following inequity:
\begin{equation}
\begin{split}
  L_{21}\left\|  \omega_j^{t}- \omega_i^{t}\right\| + L_{22} \left\|  \psi_j^{t}- \psi_i^{t}\right\|
\geq
&
  \left\| \nabla_{\psi}f_j\left(\omega_j^{t}, \psi_i^{t}\right) - \nabla_{\psi}f_j\left(\omega_i^{t}, \psi_i^{t}\right)\right\|\\
&  + \left\| \nabla_{\psi}f_j\left(\omega_j^{t}, \psi_i^{t}\right) - \nabla_{\psi}f_j\left(\omega_j^{t}, \psi_j^{t}\right)\right\|.
  \end{split}
\end{equation}
Applying the triangle inequality in Euclidean geometry on the r.h.s of the above inequality, we further obtain 
\begin{equation}
\label{L21L12-lip}
    L_{21}\left\|  \omega_j^{t}- \omega_i^{t}\right\| + L_{22} \left\|  \psi_j^{t}- \psi_i^{t}\right\|
    \geq
    \left\| \nabla_{\psi}f_j\left(\omega_j^{t}, \psi_j^{t}\right) - \nabla_{\psi}f_j\left(\omega_i^{t}, \psi_i^{t}\right)\right\|.
\end{equation}
By putting (\ref{B-strongcov}) and (\ref{L21L12-lip}) back into the corresponding items in the r.h.s. of (\ref{psi_iter1}), we have
\begin{equation}
\label{psiconverge_t1}
\begin{split}
 \left\|\psi_i^{t} - \psi^{\star}\left(\omega_i^{t}\right)\right\| \leq & \sqrt{\frac{\mu_2}{\mu_2 + B}} \left\|  \psi_i^{t-1} - \psi^{\star}\left(\omega_i^{t}\right)\right\| + \sqrt{\frac{\mu_2}{\mu_2 + B}}\frac{\mu_2+L_{22}}{N\mu_2} \sum_{\substack{j=1 \\ j\neq i}}^N\left\|  \psi_j^{t}- \psi_i^{t}\right\|\\
 &+ \sqrt{\frac{\mu_2}{\mu_2 + B}} \frac{L_{21}}{N\mu_2} \sum_{\substack{j=1 \\ j\neq i}}^N\left\|  \omega_j^{t}- \omega_i^{t}\right\| + \sqrt{\frac{\mu_2}{\mu_2 + B}}\frac{1}{N} \sum_{\substack{j=1 \\ j\neq i}}^N \left\|  \psi_j^{t-1}- \psi_i^{t-1}\right\|.
 \end{split}
 \end{equation}
 According to the $\kappa$-Lipschitz continuity of  $\psi^{\star}(\omega)$, 
the r.h.s of (\ref{psiconverge_t1}) can be further amplified by applying the  triangle inequality on the first item, and we obtain
 \begin{equation}
 \label{eq61}
\begin{split}
 \left\|\psi_i^{t} - \psi^{\star}\left(\omega_i^{t}\right)\right\|
 \leq 
 &
 \sqrt{\frac{\mu_2}{\mu_2 + B}} \left\|  \psi_i^{t-1} - \psi^{\star}\left(\omega_i^{t-1}\right)\right\| + \sqrt{\frac{\mu_2}{\mu_2 + B}}\kappa \left\|  \omega_i^{t-1} - \omega_i^{t}\right\|\\
 &+ \sqrt{\frac{\mu_2}{\mu_2 + B}}\frac{1}{N} \sum_{\substack{j=1 \\ j\neq i}}^N \left\|  \psi_j^{t-1}- \psi_i^{t-1}\right\|
 +  \sqrt{\frac{\mu_2}{\mu_2 + B}}\frac{L_{21}}{N\mu_2} \sum_{\substack{j=1 \\ j\neq i}}^N\left\|  \omega_j^{t}- \omega_i^{t}\right\| \\
 &+  \sqrt{\frac{\mu_2}{\mu_2 + B}}\frac{\mu_2+L_{22}}{N\mu_2} \sum_{\substack{j=1 \\ j\neq i}}^N\left\|  \psi_j^{t}- \psi_i^{t}\right\|.
 \end{split}
 \end{equation}
The following inequality is obtained by further applying Cauchy-Schwarz inequality on the above equation.
\begin{equation}
\begin{split}
    \left\|\psi_i^{t} - \psi^{\star}\left(\omega_i^{t}\right)\right\|^2 \leq
 &
 \frac{2\mu_2}{2\mu_2 + B} \left\|  \psi_i^{t-1} - \psi^{\star}\left(\omega_i^{t-1}\right)\right\|^2 + \frac{8\mu_2}{B}\kappa^2 \left\|  \omega_i^{t-1} - \omega_i^{t}\right\|^2\\
 &+ \frac{8\mu_2}{NB} \sum_{\substack{j=1 \\ j\neq i}}^N \left\|  \psi_j^{t-1}- \psi_i^{t-1}\right\|^2
 +  \frac{8L_{21}^2}{N\mu_2B} \sum_{\substack{j=1 \\ j\neq i}}^N\left\|  \omega_j^{t}- \omega_i^{t}\right\|^2 \\
 &+  \frac{8(\mu_2+L_{22})^2}{N\mu_2B} \sum_{\substack{j=1 \\ j\neq i}}^N\left\|  \psi_j^{t}- \psi_i^{t}\right\|^2.  
\end{split}    
\end{equation}
We have finally proved Lemma~\ref{lemma:psiconverge} by summing over $i$ on both sides of the above inequality.
\end{proof}

\begin{lemma}
\label{lemma:lemmapsidiv}
In Algorithm~\ref{Alg:FedMM}, the following inequality holds for all $t\in [T]$.
\begin{equation}
\label{lemmapsidiv}
    \frac{\mu_2^2}{2N}\bar{\psi}_t \leq 4L^2_{22}\widetilde{\psi}_t + 4L^2_{21}\widetilde{\omega}_t.
\end{equation}
\end{lemma}
\begin{proof}
By taking the absolute value at both sides of (\ref{delta_phi}) in Proposition~\ref{prop:psi_step}, we have
\begin{equation}
  \mu_2 \left\|  \psi^{t}_{i} - \psi^{t-1}_0\right\| = \left\|  \nabla_{\psi}f_i(\omega_i^{t}, \psi_i^{t}) - \nabla_{\psi}f_i(\omega_i^{t-1}, \psi_i^{t-1})\right\|. 
\end{equation}
Since $\nabla_{\psi}f_i(\omega, \psi)$ is $L_{21}$-Lipschitz continuity on $\omega$ and $L_{22}$-Lipschitz continuity on $\psi$, the r.h.s. of the above equation can be further amplified, which leads to
\begin{equation}
  \mu_2 \left\|  \psi^{t}_{i} - \psi^{t-1}_0\right\| \leq L_{22}\left\|\psi_i^{t} - \psi_i^{t-1}\right\| + L_{21}\left\|  \omega_i^{t} - \omega_i^{t-1}\right\|. 
\end{equation}
Next, we focus on the l.h.s. of the above inequality. According to the the  triangle inequality, it is evident that
$
\left\|  \psi^{t}_{i} -\psi^{t}_{j}\right\|
\leq 
     \left\|  \psi^{t}_{i} - \psi^{t-1}_0\right\| + \left\|  \psi^{t}_{j} - \psi^{t-1}_0\right\|, \quad \forall i,j\in[N].
$ Then, we have
\begin{equation}
\begin{split}
    \mu_2\left\|\psi_i^t - \psi_j^t \right\| \leq & L_{22}\left\|\psi_i^t - \psi_i^{t-1} \right\| + L_{21}\left\|\psi_i^t - \psi_i^{t-1} \right\| \\
    &
    + L_{22}\left\|\psi_j^t - \psi_j^{t-1} \right\| + L_{21}\left\|\psi_j^t - \psi_j^{t-1} \right\|. 
\end{split}
\end{equation}
According to the Cauchy-Schwarz inequality, the above inequality is equivalent to
\begin{equation}
\begin{split}
\mu_2^2\left\|\psi_i^t - \psi_j^t \right\|^2 
\leq &
4L_{22}^2\left\|\psi_i^t - \psi_i^{t-1} \right\|^2 + 4L_{21}^2\left\|\psi_i^t - \psi_i^{t-1} \right\|^2 \\
& + 4L_{22}^2\left\|\psi_j^t - \psi_j^{t-1} \right\|^2 + 4L_{21}^2\left\|\psi_j^t - \psi_j^{t-1} \right\|^2. 
\end{split}
\end{equation}
Then by summing up on a set of  $\{i,j\}$ pairs with $i, j\in [N]$ and $i \neq j$, we have proved~(\ref{lemmapsidiv}).
    
\end{proof}
\begin{lemma}
\label{lemma:lemmapsiupdate}
In Algorithm~\ref{Alg:FedMM}, the following inequality holds for all $t\in [T]$.
\begin{equation}
\label{lemmapsiupdate}
  \widetilde{\psi}_t
  \leq   \frac{4L_{22}^2}{\mu_2^2}\epsilon_t + 
  \frac{4(\mu_2+L_{22})^2}{N\mu_2^2}\bar{\psi}_t
 +  \frac{4L_{21}^2}{N\mu_2^2} \bar{\omega}_t + \frac{4}{N} \bar{\psi}_{t-1}.
\end{equation}
\end{lemma}
\begin{proof}
From Proposition~\ref{prop:psi_iter}, we have
\begin{equation}
\label{avgphi}
    \frac{1}{N}\sum_{i=1}^N \psi_i^{t+1} - \frac{1}{N}\sum_{i=1}^N \psi_i^{t} = \frac{1}{N\mu_2} \sum_{i=1}^N \nabla_{\psi} f_i\left(\omega_i^{t+1}, \psi_i^{t+1}\right),
\end{equation}
where the $\psi_i^{t}$ is equivalently represented by 
\begin{equation}
    \psi_i^{t} = \frac{1}{N}\sum_{i=1}^N \psi_i^{t} + \frac{1}{N}\sum_{\substack{j=1 \\ j\neq i}}^N \left(\psi_i^{t} - \psi_j^{t}\right).
\end{equation}
Following the same procedure, $f\left(\psi_i^{t+1}, \omega_i^{t+1}\right)$ is equivalently denoted by
\begin{equation}
\begin{split}
    f\left(\psi_i^{t+1}, \omega_i^{t+1}\right)
    &= \frac{1}{N}\sum_{j=1}^N f_j\left(\psi_i^{t+1}, \omega_i^{t+1}\right)\\
    &= \frac{1}{N}\sum_{j=1}^N f_j\left(\psi_j^{t+1}, \omega_j^{t+1}\right) + \frac{1}{N}\sum_{\substack{j=1 \\ j\neq i}}^N \left(f_j\left(\psi_i^{t+1}, \omega_i^{t+1}\right) - f_j\left(\psi_j^{t+1}, \omega_j^{t+1}\right)\right).
    \end{split}
\end{equation}
By substituting the preceding two equations back to~(\ref{avgphi}) and after some tedious algebra manipulations, we obtain
\begin{equation}
\label{lemma22}
\begin{split}
  \left\|  \psi_i^{t+1} - \psi_i^t\right\| =
   &
  \frac{1}{\mu_2}\nabla_{\psi}f\left(\omega_i^{t+1}, \psi_i^{t+1}\right) + \frac{1}{N\mu_2}\sum_{\substack{j=1 \\ j\neq i}}^N (\nabla_{\psi}f_j(\omega_j^{t+1}, \psi_j^{t+1}) \\
  & - \nabla_{\psi}f_j\left(\omega_i^{t+1}, \psi_i^{t+1}\right)) - \frac{1}{N} \sum_{\substack{j=1 \\ j\neq i}}^N(\psi^t_i - \psi^t_j) + \frac{1}{N} \sum_{\substack{j=1 \\ j\neq i}}^N(\psi^{t+1}_i - \psi^{t+1}_j )  \\
  \leq & \frac{1}{\mu_2}\left\|  \nabla_{\psi}f\left(\omega_i^{t+1}, \psi_i^{t+1}\right)\right\| + \frac{1}{N\mu_2}\sum_{\substack{j=1 \\ j\neq i}}^N
  \left\|   \nabla_{\psi}f_j\left(\omega_j^{t+1}, \psi_j^{t+1}\right) - \nabla_{\psi}f_j\left(\omega_i^{t+1}, \psi_i^{t+1}\right)\right\|\\
  &+ \frac{1}{N}\sum_{\substack{j=1 \\ j\neq i}}^N\left\|  \psi_i^t - \psi_j^t\right\| + \frac{1}{N}\sum_{\substack{j=1 \\ j\neq i}}^N\left\|  \psi_i^{t+1} - \psi_j^{t+1}\right\|.
\end{split}
\end{equation}
According to the definition of $\psi^{\star}\left(\omega_i^{t+1}\right)$ in~(\ref{psistar}), we have $\nabla_{\psi}f(\omega_i^{t+1}, \psi^{\star}\left(\omega_i^{t+1}\right)) =0 $. By further taking into account that $\nabla_{\psi}f_j(\omega, \psi)$ is $L_{22}$-Lipschitz continuity on $\psi$, we have
\begin{equation}
\label{lip22}
\begin{split}
  \left\|\nabla_{\psi}f\left(\omega_i^{t+1}, \psi_i^{t+1}\right)\right\| 
  &=
 \left\|\nabla_{\psi}f\left(\omega_i^{t+1}, \psi_i^{t+1}\right) - \nabla_{\psi}f(\omega_i^{t+1}, \psi^{\star}\left(\omega_i^{t+1}\right))\right\|\\
  &\leq
  L_{22}\left\|\psi_i^{t+1} - \psi^{\star}\left(\omega_i^{t+1}\right)\right\|.
  \end{split}
\end{equation}
Besides, since $\nabla_{\psi}f_j(\omega, \psi)$ is $L_{21}$-Lipschitz continuity on $\omega$ and $L_{22}$-Lipschitz continuity on $\psi$, we have
\begin{equation}
\label{lip2221}
 \left\|  \nabla_{\psi}f_j\left(\omega_j^{t+1}, \psi_j^{t+1}\right) - \nabla_{\psi}f_j\left(\omega_i^{t+1}, \psi_i^{t+1}\right)\right\|
    \leq L_{21}\left\|  \omega_j^{t+1}- \omega_i^{t+1}\right\| + L_{22} \left\|  \psi_j^{t+1}- \psi_i^{t+1}\right\|.
\end{equation}
By substituting (\ref{lip22}) and  (\ref{lip2221}) back into the inequality in (\ref{lemma22}), we  obtain:
\begin{equation}
\begin{split}
  \left\|  \psi_i^{t+1} - \psi_i^t\right\| 
  \leq &  \frac{L_{22}}{\mu_2}\left\|\psi_i^{t+1} - \psi^{\star}\left(\omega_i^{t+1}\right)\right\| + 
  \frac{\mu_2+L_{22}}{N\mu_2} \sum_{\substack{j=1 \\ j\neq i}}^N\left\|  \psi_j^{t+1}- \psi_i^{t+1}\right\|\\
 &+  \frac{L_{21}}{N\mu_2} \sum_{\substack{j=1 \\ j\neq i}}^N\left\|  \omega_j^{t+1}- \omega_i^{t+1}\right\| + \frac{1}{N} \sum_{\substack{j=1 \\ j\neq i}}^N \left\|  \psi_j^{t}- \psi_i^{t}\right\|.
\end{split}
\end{equation}
Finally, by applying Cauchy-Schwarz inequality, we have
\begin{equation}
\begin{split}
\left\|  \psi_i^{t+1} - \psi_i^t\right\|^2 \leq &  \frac{4L_{22}^2}{\mu_2^2}\left\|\psi_i^{t+1} - \psi^{\star}\left(\omega_i^{t+1}\right)\right\|^2 + 
  \frac{4(\mu_2+L_{22})^2}{N\mu_2^2} \sum_{\substack{j=1 \\ j\neq i}}^N\left\|  \psi_j^{t+1}- \psi_i^{t+1}\right\|^2\\
 &+  \frac{4L_{21}^2}{N\mu_2^2} \sum_{\substack{j=1 \\ j\neq i}}^N\left\|  \omega_j^{t+1}- \omega_i^{t+1}\right\|^2 + \frac{4}{N} \sum_{\substack{j=1 \\ j\neq i}}^N \left\|  \psi_j^{t}- \psi_i^{t}\right\|^2.
\end{split}
\end{equation}
By further making a summation of the above inequality for all $i\in[N]$, we have proved~(\ref{lemmapsiupdate}).
\end{proof}

\begin{lemma}
\label{lemma:lemmaboundbaromegat}
In Algorithm~\ref{Alg:FedMM}, the following inequality holds for all $t\in [T]$.
\begin{equation}
\label{baromegaineq}
\frac{1}{2N}\bar{\omega}_t \leq \frac{6L^2_{\Phi}}{\mu_1^2}\widetilde{\omega}_t + \frac{6L_{12}^2}{\mu_1^2}\epsilon_t+ \frac{6L_{12}^2}{\mu_1^2}\epsilon_{t-1}.
\end{equation}
\end{lemma}
\begin{proof}
As stated in (\ref{delta_omega}), we have
\begin{equation}
    \mu_1(\omega_i^{t+1} - \omega_0^t) = \nabla_{\omega}f_i\left(\omega_i^{t+1}, \psi_i^{t+1}\right) - \nabla_{\omega}f_i(\omega_i^{t}, \psi_i^{t}),
\end{equation}
which is equivalent to
\begin{equation}
\begin{split}
    \mu_1\left(\omega_i^{t+1} - \omega_0^t\right) = & \nabla \Phi_i\left(\omega_i^{t+1}\right) - \nabla \Phi_i\left(\omega_i^{t}\right) + \nabla_{\omega}f_i\left(\omega_i^{t+1}, \psi_i^{t+1}\right)\\
    & - \nabla \Phi_i\left(\omega_i^{t+1}\right) - \left(\nabla_{\omega}f_i(\omega_i^{t}, \psi_i^{t}) - \nabla \Phi_i\left(\omega_i^{t}\right)\right).
\end{split}
\end{equation}
By first taking norm on both sides of the above equation and then applying the triangle inequality on the r.h.s, we further obtain the following inequality:
\begin{equation}
\begin{split}
   \mu_1 \left\|\omega_i^{t+1} - \omega_0^t \right\| 
   \leq &
   \left\|\nabla \Phi_i\left(\omega_i^{t+1}\right) - \Phi_i\left(\omega_i^{t}\right)\right\| + \left\| \nabla_{\omega}f_i\left(\omega_i^{t+1}, \psi_i^{t+1}\right) - \Phi_i\left(\omega_i^{t+1}\right)\right\| \\
   & + \left\|\nabla_{\omega}f_i(\omega_i^{t}, \psi_i^{t}) - \Phi_i\left(\omega_i^{t}\right)\right\|.
   \end{split}
\end{equation}
According to the claim in Assumption~\ref{assump:Lipshictz-gradients} of the Lipschitz condition, the r.h.s of the above inequality is further amplified, which leads to the following inequality:
\begin{equation}
\label{omegai-omega0}
\begin{split}
  \left\|\omega_i^{t+1} - \omega_0^t \right\| \leq  \frac{L_{\Psi}}{\mu_1}\left\|\omega_i^{t+1}- \omega_i^{t}\right\| + \frac{L_{12}}{\mu_1}\left\|\psi_i^{t+1} -\psi^{\star}_i\left(\omega_i^{t+1}\right)\right\| +
  \frac{L_{12}}{\mu_1}\left\|\psi_i^{t} -\psi^{\star}_i\left(\omega_i^{t}\right)\right\|. 
 \end{split}
\end{equation}
By replacing $\omega_i^{t+1}$ with $\omega_j^{t+1}$ and following the similar analysis, we  get the following inequality:
\begin{equation}
  \left\|\omega_j^{t+1} - \omega_0^t \right\| \leq  \frac{L_{\Psi}}{\mu_1}\left\|\omega_j^{t+1}- \omega_j^{t}\right\| + \frac{L_{12}}{\mu_1}\left\|\psi_j^{t+1} -\psi^{\star}_j\left(\omega_i^{t+1}\right)\right\|+ 
  \frac{L_{12}}{\mu_1}\left\|\psi_j^{t} -\psi^{\star}_j(\omega_j^{t})\right\|. 
\end{equation}
Summing the two previous inequalities along the same sign direction and then applying the triangle inequality results in:
\begin{equation}
\begin{split}
    \left\|\omega_i^{t+1} - \omega_j^{t+1} \right\| 
    \leq &
    \left\|\omega_i^{t+1} - \omega_0^t \right\| + \left\|\omega_j^{t+1} - \omega_0^t \right\| \\
    \leq & \frac{L_{\Phi}}{\mu_1}\left\|\omega_i^{t+1} - \omega_i^t\right\| + \frac{L_{\Phi}}{\mu_1}\left\|\omega_j^{t+1} - \omega_j^t\right\| + \frac{L_{12}}{\mu_1}\left\|\psi_i^{t} - \psi^{\star}\left(\omega_i^{t}\right)\right\| \\
&+ \frac{L_{12}}{\mu_1}\left\|\psi_i^{t-1} - \psi^{\star}(\omega_i^{t-1})\right\|
 + \frac{L_{12}}{\mu_1}\left\|\psi_j^{t} - \psi^{\star}(\omega_j^{t})\right\| + \frac{L_{12}}{\mu_1}\left\|\psi_j^{t-1} - \psi^{\star}(\omega_j^{t-1})\right\|.
\end{split}
\end{equation}
Finally, by applying Cauchy-Schwarz inequality on the r.h.s of the above equation, we have
\begin{equation}
\begin{split}
    \left\|\omega_i^{t+1} - \omega_j^{t+1} \right\|^2 
    \leq
    & 
    \frac{6L^2_{\Phi}}{\mu_1^2}\left\|\omega_i^{t+1} - \omega_i^t\right\|^2 + \frac{6L_{\Phi}^2}{\mu_1^2}\left\|\omega_j^{t+1} - \omega_j^t\right\|^2 \\
    &
    + \frac{6L_{12}^2}{\mu_1^2}\left\|\psi_i^{t} - \psi^{\star}\left(\omega_i^{t}\right)\right\|^2  + \frac{6L_{12}^2}{\mu_1^2}\left\|\psi_i^{t-1} - \psi^{\star}(\omega_i^{t-1})\right\|^2\\
& + \frac{6L_{12}^2}{\mu_1^2}\left\|\psi_j^{t} - \psi^{\star}(\omega_j^{t})\right\|^2 
+ \frac{6L_{12}^2}{\mu_1^2}\left\|\psi_j^{t-1} - \psi^{\star}(\omega_j^{t-1})\right\|^2.
\end{split}
\end{equation}
By summing up on $i, j$ and replacing $t+1$ with $t$ in the above inequality, we have finally shown (\ref{baromegaineq}).
\end{proof}
\begin{lemma}
\label{lemma:lemmatildepsit}
In Algorithm~\ref{Alg:FedMM}, the following inequality holds for all $t\in [T]$.
\begin{equation}
\begin{split}
  \bar{\psi}_t \leq & A_1 \epsilon_t + A_2 \widetilde{\omega}_t + A_3 \bar{\omega}_t +  A_4\bar{\psi}_{t-1}.
\end{split}
\end{equation}
\end{lemma}
\begin{proof}
By substituting the result of Lemma~\ref{lemma:lemmapsidiv} into that of Lemma ~\ref{lemma:lemmapsiupdate}, we obtain
\begin{equation}
 \frac{\mu_2^2}{2N} \bar{\psi}_{t} - 4L_{21}^2\widetilde{\omega}_t \leq  \frac{16L_{22}^4}{\mu_2^2}\epsilon_{t} + 
  \frac{16L_{22}^2}{N\mu_2^2} \left(\mu_2+L_{22}\right)^2\bar{\psi}_{t}
 +  \frac{16L_{22}^2L_{21}^2}{N\mu_2^2} \bar{\omega}_{t} + \frac{4L_{22}^2}{N} \bar{\psi}_{t-1},  
\end{equation}
which is further simplified and results in
\begin{equation}
\begin{split}
  \left(\frac{\mu_2^2}{2N}  -  \frac{16L_{22}^2}{N\mu_2^2} \left(\mu_2+L_{22}\right)^2\right) \bar{\psi}_{t} \leq
  &
  \frac{16L_{22}^4}{\mu_2^2}\epsilon_{t}
+ 
   \frac{16L_{22}^2L_{21}^2}{N\mu_2^2} \bar{\omega}_{t} + \frac{16L_{22}^2}{N} \bar{\psi}_{t-1} + 4L_{21}^2\widetilde{\omega}_t.
 \end{split}
\end{equation}
Then there must exit positive constants $A_1, A_2, A_3,$ and $A_4$ such that
\begin{equation}
  \bar{\psi}_t \leq A_1 \epsilon_t + A_2 \widetilde{\omega}_t + A_3\bar{\omega}_t +  A_4 \bar{\psi}_{t-1},
\end{equation}
where we have 
\begin{align}
\label{A1-A4}
    A_1 & = 16L_{22}^4N /  \left(\mu_2^4  -  16L_{22}^2 \left(\mu_2+L_{22}\right)^2\right),\\
    A_2 & = 16L_{22}^2L_{21}^2 /  \left(\mu_2^4  -  16L_{22}^2 \left(\mu_2+L_{22}\right)^2\right),\\
    A_3 & = 16L_{22}^2\mu_2^2/ \left(\mu_2^4  -  16L_{22}^2 \left(\mu_2+L_{22}\right)^2\right), \\
    A_4 & = 4L_{21}^2N\mu_2^2/ \left(\mu_2^4  -  16L_{22}^2 \left(\mu_2+L_{22}\right)^2\right).
\end{align}
\end{proof}
\begin{lemma}
In Algorithm~\ref{Alg:FedMM}, the following inequality holds for all $t\in [T]$.
\label{lemma:epsilonaddbarpsit}
\begin{equation}
\label{epsilonaddbarpsit}
\epsilon_t + B_1 \bar{\psi_t} \leq B_2 (\epsilon_t + B_1 \bar{\psi_t}) + B_3 \bar{\omega_t}.
\end{equation}
\end{lemma}
\begin{proof}
By substituting the result of Lemma~\ref{lemma:lemmaboundbaromegat} back into the result of Lemma \ref{lemma:psiconverge}, we have
\begin{equation}
\begin{split}
    \epsilon_t \leq & \frac{8(\mu_2+L_{22})^2}{NB\mu_2}\bar{\psi}_t +  \frac{8\mu_2}{NB}\bar{\psi}_{t-1} + \frac{96L_{21}^2L_{12}^2}{\mu_2\mu_1^2B}\epsilon_t \\
    &
    + \left(\frac{96L_{21}^2L_{12}^2}{\mu_2\mu_1^2B} + \frac{2\mu_2}{2\mu_2+B}\right) \epsilon_{t-1} + \left(\frac{96L_{21}^2L_{\Phi}^2}{\mu_2\mu_1^2B} + \frac{8\mu_2\kappa^2}{B}\right)\widetilde{\omega}_t. 
\end{split}
\end{equation}
From the above inequality, there must exist  $\mu_1$ and $\mu_2$ that construct the constants  $A_5, A_6, A_7,$ and $A_8$ such that
\begin{equation}
\label{iterepsilont}
    \epsilon_t \leq A_5 \bar{\psi}_t + A_6\bar{\psi}_{t-1} + A_7\epsilon_{t-1} + A_8 \widetilde{\omega}_t.
\end{equation}
 After some tedious algebra manipulations from the formulas of $A_1-A_{8}$, we obtain
\begin{align*}
A_5 &= \frac{8(\mu_2+L_{22})^2}{NB\mu_2} /\left(1-\frac{96L_{21}^2L_{12}^2}{\mu_2\mu_1^2B}\right), \\
A_6 &= \frac{8\mu_2}{NB} /\left(1-\frac{96L_{21}^2L_{12}^2}{\mu_2\mu_1^2B}\right),  \\
A_7 &= \left(\frac{96L_{21}^2L_{12}^2}{\mu_2\mu_1^2B} + \frac{2\mu_2}{2\mu_2+B}\right) / \left(1-\frac{96L_{21}^2L_{12}^2}{\mu_2\mu_1^2B}\right), \\
A_8 &= \left(\frac{96L_{21}^2L_{\Phi}^2}{\mu_2\mu_1^2B} + \frac{8\mu_2\kappa^2}{B}\right) / \left(1-\frac{96L_{21}^2L_{12}^2}{\mu_2\mu_1^2B}\right).
\end{align*}
Then by substituting the result of Lemma~\ref{lemma:lemmaboundbaromegat} into that of Lemma~\ref{lemma:lemmatildepsit}, we have
\begin{equation}
\label{iterbarpsit}
    \bar{\psi}_t \leq A_9 \epsilon_t + A_{10}\epsilon_{t-1} + A_4\bar{\psi}_{t-1} + A_{11}\widetilde{\omega}_t.
\end{equation}
where $A_9$ $A_{10}$, and $A_{11}$ are also positive constants, i.e.,
\begin{align}
    &A_9 = A_1 + A_3\frac{12NL_{12}^2}{\mu_1^2}, \\
    &A_{10} = A_3\frac{12NL_{12}^2}{\mu_1^2}, \\
    &A_{11} = A_2 + A_3\frac{12NL_{\Phi}^2}{\mu_1^2}. 
\end{align}
By taking tedious manipulation on the formulas of $A_9, A_{10}, A_{11}$ from $A_1-A_8$ , we calculate $A_9, A_{10}, A_{11}$ as the following forms
\begin{align*}
&A_9  = 16L_{22}^4N /  \left(\mu_2^4  -  16L_{22}^2 \left(\mu_2+L_{22}\right)^2\right) + 192L_{22}^2L_{12}^2\mu_2^2/ \left(\mu_2^4\mu_1^2  -  16L_{22}^2\mu_1^2 \left(\mu_2+L_{22}\right)^2\right), \\
 &A_{10} = 192L_{22}^2L_{12}^2\mu_2^2/ \left(\mu_2^4\mu_1^2  -  16L_{22}^2\mu_1^2 \left(\mu_2+L_{22}\right)^2\right),  \\
 & A_{11}= 16L_{22}^2L_{21}^2 /  \left(\mu_2^4  -  16L_{22}^2 \left(\mu_2+L_{22}\right)^2\right) + 192L_{22}^2L_{12}^2\mu_2^2/ \left(\mu_2^4\mu_1^2  -  16L_{22}^2\mu_1^2 \left(\mu_2+L_{22}\right)^2\right).
\end{align*}
By computing
 (\ref{iterbarpsit}) $\times \bar{B}_1 + $  (\ref{iterepsilont}), we have
\begin{equation}
    (1- \bar{B}_1A_9)\epsilon_t + (\bar{B}_1- A_5)\bar{\psi}_t \leq 
     (\bar{B}_1A_{10}+ A_7)\epsilon_{t-1} + (\bar{B}_1A_4+A_6)\bar{\psi}_{t-1} + (\bar{B}_1A_{11}+A_8)\widetilde{\omega}_t,
\end{equation}
And we further scale to the following form
\begin{equation}
    \epsilon_t + B_1 \bar{\psi_t} \leq B_2 (\epsilon_t + B_1 \bar{\psi_t}) + B_3 \bar{\omega_t},
\end{equation}
where
\begin{align}
   B_1 &= \frac{\bar{B}_1-A_5}{1-\bar{B}_1A_9}, \\
   B_2 &= \max \{ \frac{\bar{B}_1A_{10} + A_7}{1-\bar{B}_1A_9}, \frac{\bar{B}_1A_{4} + A_6}{\bar{B}_1- A_5} \}, \\
   B_3 &= \frac{\bar{B}_1A_{11} + A_8}{1-\bar{B}_1A_9}.
\end{align}
\end{proof}

We prove Lemma~\ref{lemma:boundpsi} as follows. We repeat Lemma~\ref{lemma:boundpsi} in the following Lemma~\ref{lemma:boundci} to make the appendix self-contained.
\begin{lemma}
\label{lemma:boundci}
In Algorithm~\ref{Alg:FedMM}, the following inequality holds.
\begin{equation}
\label{eqn:boundci}
\begin{split}
\sum_{t=1}^T \sum_{i=0}^N \left\|\psi_i^{t} - \psi^{\star}(\omega_i^{t-1})\right\|^2 
\leq
&
C_1\sum_{t=1}^T\left\|\omega_i^{t-1} - \omega_i^t\right\|^2 + C_2\sum_{i = 1}^N\left\|\psi_i^{0} - \psi^{\star}(\omega_i^{0})\right\|^2  \\
&+ C_3\sum_{i = 1}^N\sum_{\substack{j=1 \\ j\neq i}}^N\left\|\omega_i^0 - \omega_j^0\right\|^2.
\end{split}
\end{equation}
\end{lemma}
\begin{proof}
By the recursive update (\ref{epsilonaddbarpsit}) in Lemma~\ref{lemma:epsilonaddbarpsit}, we obtain
\begin{equation}
  \epsilon_t + B_1 \bar{\psi}_t \leq B_2^t(\epsilon_0 + B_1 \bar{\psi}_0) + \sum_{t'=0}^t B_3B_2^{t-t'} \bar{\omega}_{t'}.
\end{equation}
By summing up on $t$, we get
\begin{equation}
    \sum_{t=1}^T (\epsilon_t + B_1\bar{\psi}_t) \leq \sum_{t=1}^TB_2^t(\epsilon_0+ B_1\bar{\psi}_0) + \sum_{t=1}^T \sum_{t' \leq t} B_3B_2^{t-t'} \bar{\omega}_t,
\end{equation}
where the r.h.s. can be further amplified by the sum of finite exponential series, which result in
\begin{equation}
  \sum_{t=1}^T \left(\epsilon_t + B_1\bar{\psi}_t\right)  \leq \frac{1}{1-B_2}\left(\epsilon_0 + B_1\bar{\psi}_0\right)  + \frac{B_3}{1-B_2}\sum_{t=1}^T\bar{\omega}_t.
\end{equation}
Finally, we  reach
\begin{equation}
  \sum_{t=1}^T \epsilon_t  \leq C_1\sum_{t=1}^T\bar{\omega}_t + C_2\epsilon_0 + C_3\bar{\psi}_0,
\end{equation}
where 
\begin{align}
    C_1 &= \frac{B_3}{1-B_2}, \\
    C_2 &= \frac{1}{1-B_2}  \\
    C_3 &=  \frac{B_1}{1-B_2}.
\end{align}

\end{proof}

\subsection{Proof of Lemma~\ref{lemma:descentPhi}}
\label{subsec:lemma3}
\begin{lemma}
(One Global Round Descent on $\mathcal{L}_\Phi$) In Algorithm~\ref{Alg:FedMM}, the following inequality holds.
\label{Lphiiter}
\begin{equation}
\label{eqn:Lphiiter}
\begin{split}
&\frac{1}{N}\sum_{i=1}^N\left(\mathcal{L}^{\Phi}_i(\omega_i^{t+1}, \omega_{0}^{t+1}, \lambda_i^{t+1}) - \mathcal{L}^{\Phi}_i(\omega_i^{t}, \omega_{0}^{t}, \lambda_i^{t})\right) \\
\leq & 
\frac{1}{N}\sum_{i=1}^N \Big(- \frac{\mu_1 - 2L_{\Phi}}{2}\left\|\omega_i^{t+1} - \omega_i^{t}\right\|^2 + \frac{1}{\mu_1}\left\|\lambda_i^{t+1} - \lambda_i^t\right\|^2  \\
&
+
L_{12}^2\left\|\psi_i^{t+1} - \psi^{\star}\left(\omega_i^{t+1}\right)\right\|^2\Big) - \frac{\mu_1}{2}\left\|\omega_{0}^{t+1} - \omega_0^t\right\|^2. 
\end{split}
\end{equation}
\end{lemma}
\begin{proof}
From the $L_{\Phi}$-Lipschiz on $\Phi_i(\omega)$, we have
\begin{equation}
    -\Phi_i(\omega_i^t) \leq -\Phi_i\left(\omega_i^{t+1}\right) + \left\langle-\nabla\Phi_i\left(\omega_i^{t+1}\right), \omega_i^t - \omega_i^{t+1}\right\rangle+ \frac{L_{\Phi}}{2}\left\|\omega_i^t - \omega_i^{t+1}\right\|^2.
\end{equation}
By considering the definition of $ \mathcal{L}^{\Phi}_i(\omega_i^{t}, \omega_0^t, \lambda_i^t)$ in (\ref{eqn:AL-Phi}), we
obtain the following inequalities to bound one global round  descent, i.e.,
\begin{equation}
\label{LL1}
\begin{split}
    &\mathcal{L}^{\Phi}_i(\omega_i^{t+1}, \omega_0^t, \lambda_i^t) - 
     \mathcal{L}^{\Phi}_i(\omega_i^{t}, \omega_0^t, \lambda_i^t) \\
     & \leq
  \left\langle\nabla\Phi_i(\omega_i^{t+1}, \omega_i^{t+1} - \omega_i^t)\right\rangle + \frac{L_{\Phi}}{2}\left\|\omega_i^{t+1}
  - \omega_i^t\right\|^2+ \left\langle\lambda_i^t, \omega_i^{t+1} - \omega_i^t\right\rangle + \frac{\mu_1}{2}\left\|\omega_i^{t+1}-\omega_0^t\right\|^2 - \frac{\mu_1}{2}\left\|\omega_i^{t}-\omega_0^t\right\|^2 \\
  &=\left\langle\nabla\Phi_i\left(\omega_i^{t+1}\right)+\lambda_i^r, \omega_i^{t+1} - \omega_i^t\right\rangle + \frac{L_{\Phi}}{2}\left\|\omega_i^{t+1} - \omega_i^t\right\|^2 + \frac{\mu_1}{2}\left\langle\omega_i^{t+1} - \omega_i^t-2\omega_0^t, \omega_i^{t+1} - \omega_i^t\right\rangle\\
  &=\left\langle\nabla\Phi_i\left(\omega_i^{t+1}\right)+\lambda_i^r+\mu_1(\omega_i^{t+1}-\omega_0^t),\omega_i^{t+1}-\omega_i^t\right\rangle + \frac{L}{2}\left\|\omega_i^{t+1} - \omega_i^t\right\|^2 -\frac{\mu_1}{2}\left\|\omega_i^{t+1} - \omega_i^t\right\|^2\\
  &\leq \frac{1}{2L}\left\|\nabla\Phi_i\left(\omega_i^{t+1}\right)+\lambda_i^r+\mu_1(\omega_i^{t+1}-\omega_0^t)\right\|^2 - \frac{\mu_1-2L}{2}\left\|\omega_i^{t+1} - \omega_i^t\right\|^2 \\
  &\leq - \frac{\mu_1-2L}{2}\left\|\omega_i^{t+1} - \omega_i^t\right\|^2 + \frac{L_{12}^2}{2L}\left\|\psi_i^{t+1} - \psi^{\star}\left(\omega_i^{t+1}\right)\right\|^2.
\end{split}
\end{equation}
where the second last inequality is due to the inequality of arithmetic and geometric means (AM-GM).
Next, from the iteration on $\lambda_i^t$ in Algorithm~\ref{Alg:FedMM}, we bound on the ascent from the iteration $\lambda_i^t$, i.e.,
\begin{equation}
\label{LL2}
\begin{split}
  \mathcal{L}^{\Phi}_i(\omega_i^{t+1}, \omega_0^t, \lambda_i^{t+1}) -\mathcal{L}^{\Phi}_i(\omega_i^{t+1}, \omega_0^t, \lambda_i^t)
  &=
  \left\langle\lambda_{i}^{t+1}-\lambda_i^t, \omega_i^{t+1}-\omega_0^t\right\rangle \\
  &= \frac{1}{\mu_1}\left\|\lambda_i^{t+1}-\lambda_i^t\right\|^2.
\end{split}
\end{equation}
Finally, we bound on the descent from the iteration on $\omega_0^t$ in Algorithm~\ref{Alg:FedMM}, we have
\begin{equation}
\label{LL3}
\begin{split}
  &\sum_{i=1}^N\left(\mathcal{L}^{\Phi}_i\left(\omega_i^{t+1}, \omega_0^{t+1}, \lambda_i^{t+1}\right) -\mathcal{L}^{\Phi}_i\left(\omega_i^{t+1}, \omega_0^t, \lambda_i^{t+1}\right)\right) \\
  =&  -\frac{\mu_1N}{2}\left\|\omega_0^{t}-\frac{1}{N}\sum_i(\omega_i^{t+1}+\frac{1}{\mu}\lambda_i)\right\|^2  + \frac{\mu_1N}{2}\left\|\omega_0^{t+1}  -\frac{1}{N}\sum_i(\omega_i^{t+1}+\frac{1}{\mu}\lambda_i)\right\|^2. 
 \end{split}
\end{equation}
According to the definition of $\omega_0^{t+1}$, i.e., $\omega_0^{t+1}= \frac{1}{N}\sum_i(\omega_i^{t+1}+\frac{1}{\mu}\lambda_i)$, we have
\begin{equation}
\label{LL4}
   \sum_{i=1}^N\left(\mathcal{L}^{\Phi}_i(\omega_i^{t+1}, \omega_0^{t+1}, \lambda_i^{t+1}) -\mathcal{L}^{\Phi}_i(\omega_i^{t+1}, \omega_0^t, \lambda_i^{t+1})\right)= -\frac{\mu_1N}{2}\left\|\omega_0^{t} - \omega_0^{t+1}\right\|^2.
\end{equation}
Making a summation of (\ref{LL1}), (\ref{LL2}), (\ref{LL3}), and (\ref{LL4}) along the same sign direction of the inequalities, and after some algebraic manipulations, we  prove (\ref{eqn:Lphiiter}).
\end{proof}
\begin{lemma}
\label{deltalambdabound}
(Bound on $\lambda_i^t$'s Iteration) In Algorithm~\ref{Alg:FedMM}, the following inequality holds for all $i\in [N]$ and $t\in [T]$.
\begin{equation}
\label{eqn:deltalambdabound}
 \left\|\lambda_i^t - \lambda_i^{t+1}\right\|^2 \leq 2L_{\Phi}^2\left\|\omega_i^{t}- \omega_i^{t+1}\right\|^2 + 4L_{12}^2\left\|\psi_i^{t+1} - \psi^{\star}\left(\omega_i^{t+1}\right)\right\|^2 + 4L_{12}^2\left\|\psi_i^{t} - \psi^{\star}\left(\omega_i^{t}\right)\right\|^2.
\end{equation}
\end{lemma}
\begin{proof}
Applying Assumption~\ref{assump:sufficient_loc} to (\ref{descent}), we have
$
    \nabla_{\omega} f_i\left(\omega_i^{t+1}, \psi_i^{t+1}\right)+ \lambda_i^{t+1} = 0.
$
By simply replacing $t+1$ with $t$ in the above inequality, we further obtain 
$
    \nabla_{\omega} f_i(\omega_i^t, \psi_i^t)+ \lambda_i^t = 0.
$
Subtracting the above two equations leads to
\begin{equation}
\label{lambda114}
    \lambda_i^t - \lambda_i^{t+1} = \nabla_{\omega} f_i\left(\omega_i^{t+1}, \psi_i^{t+1}\right) - \nabla_{\omega} f_i(\omega_i^{t}, \psi_i^{t}).
\end{equation}
Taking the norm on the both sides, we have
\begin{equation}
\begin{split}
    \left\|\lambda_i^t - \lambda_i^{t+1}\right\| = &\Big\|
    -\nabla \Phi_i\left(\omega_i^{t+1}\right) + \nabla \Phi_i\left(\omega_i^{t}\right) -\nabla_{\omega} f_i\left(\omega_i^{t+1}, \psi_i^{t+1}\right) + \nabla\Phi_i\left(\omega_i^{t}\right) \\ 
    &\quad + \nabla_{\omega} f_i(\omega_i^{t}, \psi_i^{t}) - \nabla\Phi_i\left(\omega_i^{t+1}\right)\Big\| \\
    \leq & 
    \left\|\nabla\Phi_i\left(\omega_i^{t+1}\right) - \nabla \Phi_i\left(\omega_i^{t}\right)\right\| + \left\|\nabla_{\omega} f_i\left(\omega_i^{t+1}, \psi_i^{t+1}\right) - \nabla\Phi_i\left(\omega_i^{t+1}\right)\right\| \\
    & \quad + \left\|\nabla_{\omega} f_i(\omega_i^{t}, \psi_i^{t}) - \nabla\Phi_i\left(\omega_i^{t}\right)\right\| \\
    \leq & 
    L_{\Phi}\left\|\omega_i^t - \omega_i^{t+1}\right\| + L_{12}\left\|\psi_i^{t+1} - \psi^{\star}\left(\omega_i^{t+1}\right)\right\| +  L_{12}\left\|\psi_i^{t} - \psi^{\star}\left(\omega_i^{t}\right)\right\|,
\end{split}   
\end{equation}
where the fist inequality is due to 
triangle inequality, and the second is due to the Lipschitz continuous property.
By applying Cauchy-Schwarz inequality on the last inequality of (\ref{lambda114}), we finally prove (\ref{eqn:deltalambdabound}).
\end{proof}
\begin{lemma}
(Descent on $\mathcal{L}_\Phi$) In Algorithm~\ref{Alg:FedMM}, the following inequality holds.
\label{Ldescent}
\begin{equation}
\begin{split}
\label{eqn:LT-0}
&\sum_{t=1}^T \sum_{i=1}^N \left(\mathcal{L}_i^{\Phi}(\omega_i^{T}, \omega_{0}^{T}, \lambda_i^{T}) - \mathcal{L}_i^{\Phi}(\omega_i^{0}, \omega_{0}^{0}, \lambda_i^{0})\right) \\
& \leq  -\frac{\mu_1^2-2L_{\Psi}\mu_1-4L_{\Phi}^2}{2\mu_1}\sum_{i=1}^N\sum_{t=1 }^T\left\|\omega_i^{t+1}- \omega_i^t\right\|^2 - \sum_{t=1}^T\frac{N\mu_1}{4}\left\|\omega_0^t - \omega_0^{t+1}\right\|^2 \\
&\quad + \frac{2(\mu_1+ 8L_{\Phi})L_{12}^2}{L_{\Phi}\mu_1}\sum_{i=1}^{N}\left\|\psi_i^{t+1} - \psi^{\star}\left(\omega_i^{t+1}\right)\right\|^2.
\end{split}
\end{equation}
\end{lemma}
\begin{proof}
By summing all the inequalities of (\ref{eqn:Lphiiter}) in Lemma~\ref{Lphiiter} along the same sign direction for all  $t=\in [T]$, we obtain
\begin{equation}
\begin{split}
    &\sum_i^{N}\left(\mathcal{L}^{\Phi}_i(\omega_i^{t+1}, \omega_{0}^{t+1}, \lambda_i^{t+1}) - \mathcal{L}^{\Phi}_i(\omega_i^{t}, \omega_{0}^{t}, \lambda_i^{t})\right) \\
    \leq &\sum_{i=1}^N\sum_{t=1}^T - \frac{\mu_1 - 2L_{\Phi}}{2}\left\|\omega_i^{t+1} - \omega_i^{t}\right\|^2 + \frac{1}{\mu_1}\sum_{i=1}^N\sum_{t=1}^T\left\|\lambda_i^{t+1} - \lambda_i^t\right\|^2
    \\
    &+ L_{12}^2\sum_{i=1}^N\sum_{t=1}^T\left\|\psi_i^{t+1} - \psi^{\star}\left(\omega_i^{t+1}\right)\right\|^2 - \frac{N\mu_1}{2}\left\|\omega_{0}^{t+1} - \omega_0^t\right\|^2.
\end{split}
\end{equation}
By substituting (\ref{eqn:deltalambdabound}) in Lemma \ref{deltalambdabound} into the above inequality, we have proved (\ref{eqn:LT-0}).
\end{proof}
\begin{lemma}
\label{lowerboundL}
(Lower Bound on $\mathcal{L}^{\Phi}_i$) In Algorithm~\ref{Alg:FedMM}, the following inequality holds for all $t\in [T]$ and $i\in [N]$.
\begin{equation}
\label{eqn:lowerboundL}
  \mathcal{L}_i^{\Phi}(\omega_i^t, \omega_0^t, \lambda_i^t) \geq \Phi_{i}(\omega_0^t) - \frac{L_{12}^2}{L_{\Phi}}\left\|\psi_i^{t} - \psi^{\star}\left(\omega_i^{t}\right)\right\|^2 + \frac{\mu_1- 2L_{\Phi}}{2}\left\|\omega_0^t- \omega_i^t\right\|^2.
\end{equation}
\end{lemma}
\begin{proof}
From the $L_{\Phi}$-Lipschiz on $\nabla \Phi_i(\omega)$, we have
\begin{equation}
\label{LPsiPsi}
    \Phi_i(\omega_0^t) - \Phi_i(\omega_i^t) \leq \left\langle\nabla\Phi_i(\omega_i^t), \omega_0^t - \omega_i^t\right\rangle + \frac{L_{\Phi}}{2}\left\|\omega_0^t - \omega_i^t\right\|^2.
\end{equation}
Then we have
\begin{equation}
\begin{split}
    \mathcal{L}_i^{\Phi}(\omega_i^t, \omega_0^t, \lambda_i^t) -  \Phi_{i}(\omega_0^t) &\geq \left\langle\nabla\Phi_i(\omega_i^t) + \lambda_i^t, \omega_0^t - \omega_i^t\right\rangle + \frac{\mu_1 -L_{\Phi}}{2}\left\|\omega_0^t - \omega_i^t\right\|^2 \\
    & \geq \left\langle\nabla\Phi_i(\omega_i^t) - \nabla_{\omega} f_i(\omega_i^t, \psi_i^t), \omega_0^t - \omega_i^t\right\rangle + \frac{\mu_1 -L_{\Phi}}{2}\left\|\omega_0^t - \omega_i^t\right\|^2.
\end{split}
\end{equation} 
Besides, by applying the AM-GM inequality, we obtain
\begin{equation}
\begin{split}
  \left\langle\nabla\Phi_i(\omega_i^t) - \nabla_{\omega} f_i(\omega_i^t, \psi_i^t), \omega_0^t - \omega_i^t\right\rangle 
  & \geq
  -\frac{1}{2 L_{\Phi}}\left\|\nabla\Phi_i(\omega_i^t) - \nabla_{\omega} f_i(\omega_i^t, \psi_i^t)\right\|^2\\
  &\geq -\frac{L_{12}^2}{2L_{\Phi}}\left\|\psi_i^{t} - \psi^{\star}\left(\omega_i^{t}\right)\right\|^2.
\end{split}    
\end{equation}
Substituting the above result back to (\ref{LPsiPsi}), we have proved (\ref{eqn:lowerboundL}).
\end{proof}

We prove Lemma~\ref{lemma:descentPhi} as follows. Note that we repeat Lemma~\ref{lemma:descentPhi} in the following Lemma~\ref{lemma:thm-descent-Phi} to make the appendix self-contained.
\begin{lemma}
\label{lemma:thm-descent-Phi}
(Descent on $\Phi$)
In Algorithm~\ref{Alg:FedMM}, the following inequality holds.
\begin{equation}
\label{eqn:thm-descent-Phi}
\begin{split}
\Phi(\omega_0^T) - \Phi(\omega_0^0) 
\leq
&
-\frac{\mu_1^2-2L_{\Phi}\mu_1-4L_{\Phi}^2}{2\mu_1N}\sum_{i=1}^N\sum_{t =1}^T\left\|\omega_i^{t+1}- \omega_i^t\right\|^2 -\frac{\mu_1}{4} \sum_{t =1}^T\left\|\omega_0^t - \omega_0^{t+1}\right\|^2 \\
&+ \frac{(3\mu_1+ 16L_{\Phi})L_{12}^2}{\mu_1NL_{\Phi}}\sum_{i=1}^N\sum_{t =1}^T\left\|\psi_i^{t+1} - \psi^{\star}\left(\omega_i^{t+1}\right)\right\|^2.
\end{split}
\end{equation}
\end{lemma}
\begin{proof}
From the initial condition of $\lambda_i^0=0$ and $\omega_i^0=\omega_0^0$, we have
\begin{equation}
    \Phi_i(\omega_0^0) = \frac{1}{N}\sum_{i=1}^{N}\mathcal{L}^{\Phi}_i(\omega_i^0,\omega_0^0,\lambda_i^0 ).
\end{equation}
Substituting (\ref{eqn:lowerboundL}) in Lemma~(\ref{lowerboundL}) into (\ref{eqn:LT-0}) in Lemma~(\ref{Ldescent}) results in
\begin{equation}
\begin{split}
    \frac{1}{N}\sum_{i=1}^{N}\mathcal{L}_i^{\Phi}(\omega_i^t, \omega_0^t, \lambda_i^t) &\geq \Phi(\omega_0^t) - \frac{L_{12}^2}{NL_{\Phi}}\sum_{i=1}^N\left\|\psi_i^{t} - \psi^{\star}\left(\omega_i^{t}\right)\right\|^2 + \sum_{i=1}^N\frac{\mu_1- 2L_{\Phi}}{2N}\left\|\omega_0^t- \omega_i^t\right\|^2 \\
    &\geq \Phi(\omega_0^t) - \frac{L_{12}^2}{NL_{\Phi}}\sum_{i=1}^N\left\|\psi_i^{t} - \psi^{\star}\left(\omega_i^{t}\right)\right|^2.
\end{split}
\end{equation}
By substituting the above result  into (\ref{eqn:LT-0}) in Lemma \ref{Ldescent}, we have proved (\ref{eqn:thm-descent-Phi}).
\end{proof}
\subsection{Proof of Theorem~\ref{theorem:convergence}}
\label{subsec:theorem}

\begin{lemma}
\label{boundomega0i}
In Algorithm~\ref{Alg:FedMM}, the following inequality holds for all $t\in [T]$ and $i\in [N]$.
\begin{equation}
\label{eqn:boundomega0i}
\left\|\omega_{0}^t - \omega_{i}^t\right\| \leq \left\|\omega_{0}^{t-1} - \omega_{0}^t\right\| + \frac{L_{\Phi}}{\mu_1}\left\|\omega_{i}^{t-1} - \omega_{i}^t\right\| + \frac{L_{12}}{\mu_1}\left\|\psi_i^{t} - \psi^{\star}\left(\omega_i^{t}\right)\right\| + \frac{L_{12}}{\mu_1}\left\|\psi_i^{t-1} - \psi^{\star}(\omega_i^{t-1})\right\|
\end{equation}
\end{lemma}
\begin{proof}
Using the triangle inequality, we have
\begin{equation}
    \left\|\omega_i^{t+1} - \omega_0^{t+1}\right\| \leq \left\|\omega_i^{t+1} - \omega_0^t\right\| + \left\|\omega_i^{t+1} - \omega_i^t\right\|.
\end{equation}
Substituting (\ref{omegai-omega0}) into the above inequality results in
\begin{equation}
\left\|\omega_{0}^{t+1} - \omega_{i}^{t+1}\right\| \leq \left\|\omega_{0}^{t} - \omega_{0}^{t+1}\right\| + \frac{L_{\Phi}}{\mu_1}\left\|\omega_{i}^{t+1} - \omega_{i}^t\right\| + \frac{L_{12}}{\mu_1}\left\|\psi_i^{t} - \psi^{\star}\left(\omega_i^{t}\right)\right\| + \frac{L_{12}}{\mu_1}\left\|\psi_i^{t+1} - \psi^{\star}\left(\omega_i^{t+1}\right)\right\|.
\end{equation}
By replacing $t+1$ with $t$, we have proved (\ref{eqn:boundomega0i}).
\end{proof}
\begin{lemma}
\label{bounddeltapsi1}
(Bounded $\left\|\nabla \Psi\left(\omega_{0}^{t}\right)\right\|^{2}$) In Algorithm~\ref{Alg:FedMM}, the following inequality holds for $t\in [T]$.
\begin{equation}
\label{eqn:bounddeltapsi1}
\begin{split}
    \left\|\nabla \Phi\left(\omega_{0}^{t}\right)\right\| \leq 
    & L_{\Phi}\left\|\omega^t_0 - \omega^{t-1}_0\right\| +  \frac{\mu_1^2 + L_{\Phi}^2}{N\mu_1}\sum_{i=1}^N\left\|\omega^{t-1}_i - \omega^t_i\right\| + \frac{L_{12}(L_{\Phi}+\mu_1)}{N\mu_1}\sum_{i=1}^N\left\|\psi_i^{t-1} - \psi^{\star}(\omega_i^{t-1})\right\|\\
    & + \frac{L_{12}L_{\Phi}}{N\mu_1}\sum_{i=1}^N\left\|\psi_i^{t} - \psi^{\star}\left(\omega_i^{t}\right)\right\|.
\end{split}
\end{equation}
\end{lemma}
\begin{proof}
From Lemma~\ref{omega_iter}, we have
\begin{equation}
\sum_{i=1}^N \omega_i^{t} = \sum_{i=1}^N \omega_i^{t-1} - \frac{1}{\mu_1}\sum_{i=1}^N  \nabla_{\omega} f_i\left(\omega^{t}_i, \psi^{t}_i\right),
\end{equation}
which can be equivalently represented by
\begin{equation}
\begin{split}
\sum_{i=1}^N \omega_i^{t} = 
&\sum_{i=1}^N \omega_i^{t-1} - \frac{1}{\mu_1}\sum_{i=1}^N  \big(\nabla \Phi_i(\omega_0^{t}) + (\nabla \Phi_i\left(\omega_i^{t}\right) - \nabla \Phi_i(\omega_0^{t}))  \\
&+ \left(\nabla_{\omega} f_i\left(\omega^{t}_i, \psi^{t}_i\right) - \nabla \Phi_i\left(\omega_i^{t}\right)\right)\big).   \end{split}
\end{equation}
Next, we  get
\begin{equation}
    \nabla \Phi(\omega_0^{t}) =  \frac{\mu_1}{N} \sum_{i=1}^N \omega_i^{t-1} - \frac{\mu_1}{N}\sum_{i=1}^N \omega_i^{t} - \frac{1}{N} (\nabla \Phi_i\left(\omega_i^{t}\right) - \nabla \Phi_i(\omega_0^{t})) + \frac{1}{N}(\nabla_{\omega} f_i\left(\omega^{t}_i, \psi^{t}_i\right) - \nabla \Phi_i\left(\omega_i^{t}\right)).
\end{equation}
According to the triangle inequality, we have
\begin{equation}
    \left\|\nabla \Phi(\omega_0^{t})\right\| \leq \frac{\mu_1}{N}\sum_{i=1}^N \left\| \omega_i^{t} -\omega_i^{t-1}\right\|+\frac{\mu_1}{N}\left\|\nabla \Phi_i\left(\omega_i^{t}\right) - \nabla \Phi_i(\omega_0^{t})\right\| + \frac{1}{N}\left\|\nabla_{\omega} f_i\left(\omega^{t}_i, \psi^{t}_i\right) - \nabla \Phi_i\left(\omega_i^{t}\right)\right\|.
\end{equation}
By sing the Lipschitz condition, we further obtain
\begin{equation}
      \left\|\nabla \Phi\left(\omega_{0}^{t}\right)\right\| \leq \frac{L_{\Phi}}{N}\sum_{i=1}^N\left\|\omega^t_0 - \omega^t_i\right\| + \frac{\mu_1}{N}\sum_{i=1}^N\left\|\left(\omega^{t-1}_i - \omega^t_i\right)\right\| + \frac{L_{12}}{N}\sum_{i=1}^N \left\|\psi_i^{t-1} - \psi^{\star}\left(\omega_i^{t+1}\right)\right\|.
\end{equation}
By substituting (\ref{eqn:boundomega0i}) in Lemma (\ref{boundomega0i}) into the above inequality, we have proved (\ref{eqn:bounddeltapsi1}).
\end{proof}
\begin{lemma}
In Algorithm~\ref{Alg:FedMM}, the following inequality holds.
\label{deltapsisumt}
\begin{equation}
\label{eqn:deltapsisumt}
\begin{split}
\sum_{t=1}^T \left\|\nabla \Phi\left(\omega_{0}^{t}\right)\right\|^2 \leq
& 
\frac{4L^2_{\Phi}\mu^2_1N 
+ 4C_1L^2_{12}\left((L_{\Phi}+\mu_1)^2 + L_{\Phi}^2\right)}{N\mu_1^2}\sum_{t=1}^T\left\|\omega^t_0 - \omega^{t-1}_0\right\|^2\\
&
+  \frac{4(\mu_1^2+L_{\Phi}^2)^2}{N\mu_1^2}\sum_{i=0}^N\sum_{t=1}^T\left\|\omega^{t-1}_i - \omega^t_i\right\|^2 \\
& + \frac{4C_2L^2_{12}\left((L_{\Phi}+\mu_1)^2 + L_{\Phi}^2\right)}{N\mu_1^2}\sum_{i=1}^N\left\|\psi_i^{0} - \psi^{\star}(\omega_i^{0})\right\|^2 \\
&+ \frac{4C_3L^2_{12}\left((L_{\Phi}+\mu_1)^2 + L_{\Phi}^2\right)}{N\mu_1^2}\sum_{i=1}^N\sum_{j \neq i}\left\|\omega^0_i - \omega^{0}_j\right\|^2.
\end{split}
\end{equation}
\end{lemma}
\begin{proof}
According to the Cauchy-Schwarz inequality, we have
\begin{equation}
\begin{split}
    \left\|\nabla\Phi(\omega_0^t)\right\|^2 
    \leq &
    4L_{\Phi}^2\left\|\omega_0^t- \omega_0^{t-1}\right\|^2 + \frac{4(\mu_1^2+L_{\Phi}^2)^2}{N\mu_1^2}\sum_{i=1}^N\left\|\omega_i^{t-1}-\omega_i^t\right\|^2 \\
    &+ \frac{4L_{12}^2(L_{\Phi}+\mu_1)^2}{N\mu_1^2}\sum_{i=1}^N\left\|\psi_i^{t-1}- \psi^{\star}(\omega_i^{t-1})\right\|^2 \\
    &+ \frac{4L_{12}^2L_{\Phi}^2}{N\mu_1^2}\sum_{i=1}^N\left\|\psi_i^t- \psi^{\star}(\omega_i^t)\right\|^2.
\end{split}
\end{equation}
By making a summation of all the inequalities for $t\in [T]$, we further obtain 
\begin{equation}
\label{eqn:sumPsit}
 \begin{split}
    \sum_{t=1}^T\left\|\nabla \Phi\left(\omega_{0}^{t}\right)\right\|^2 \leq & 4L_{\Phi}^2\sum_t\left\|\omega^t_0 - \omega^{t-1}_0\right\|^2 +  \frac{4(\mu_1^2+L_{\Phi}^2)^2}{N\mu_1^2}\sum_{i=1}^N\sum_{t=1}^T\left\|\omega^{t-1}_i - \omega^t_i\right\|^2 \\
    &+ \frac{4L_{12}^2\left((L_{\Phi}+\mu_1)^2+L_{\Phi}^2\right)}{N\mu_1^2}\sum_{i=1}^N\sum_{t=1}^T \left\|\psi_i^{t} - \psi^{\star}\left(\omega_i^{t}\right)\right\|^2. \\
 \end{split}   
\end{equation}
By substituting (\ref{eqn:boundci}) in Lemma~\ref{lemma:boundci} into (\ref{eqn:bounddeltapsi1}) in Lemma~\ref{bounddeltapsi1} and further substituting the result into the r.h.s. of (\ref{eqn:sumPsit}), we prove (\ref{eqn:deltapsisumt}).
\end{proof}

We are now ready to prove Theorem~\ref{theorem:convergence}, which is replicated  in the following Theorem~\ref{theorem:convege} to make the appendix self-contained.
\begin{theorem}
\label{theorem:convege}
 In Algorithm~\ref{Alg:FedMM}, the following inequality holds.
\begin{equation}
    \Phi(\omega_0^0) - \Psi(\omega_0^T) \leq - E_1 \sum_{t=1}^T\left\|\nabla \Phi\left(\omega_{0}^{t}\right)\right\|^2 + E_2 \sum_{i=1}^N\left\|\psi_i^{0} - \psi^{\star}(\omega_i^{0})\right\|^2 + E_3 \sum_{i=1}^N\sum_{j \neq i}\left\|\omega^0_i - \omega^{0}_j\right\|^2. 
\end{equation}
\end{theorem}
\begin{proof}
By substituting the result of Lemma~\ref{lemma:boundci} into (\ref{eqn:thm-descent-Phi}) in Lemma~\ref{lemma:thm-descent-Phi} and after some algebraic manipulations, we obtain 
\begin{equation}
\begin{split}
\Phi(\omega_0^0) - \Phi(\omega_0^T) 
\leq
&
\left(-\frac{\mu_1^2-2L_{\Phi}\mu_1-4L_{\Phi}^2}{2\mu_1N}+\frac{C_1(3\mu_1+ 16L_{\Phi})L_{12}^2}{NL_{\Phi}\mu_1}\right) \sum_{i=1}^N\sum_{t=1}^T\left\|\omega_i^{t+1}- \omega_i^t\right\|^2 \\ &-\frac{\mu_1}{4}\sum_{t=1}^T\left\|\omega_0^t - \omega_0^{t+1}\right\|^2 
+ \frac{C_2(3\mu_1+ 16L_{\Phi})L_{12}^2}{NL_{\Phi}\mu_1}\sum_{i=1}^N\left\|\psi_i^{0} - \psi^{\star}(\omega_i^{0})\right\|^2 \\
&+ \frac{C_3(3\mu_1+ 16L_{\Phi})L_{12}^2}{NL_{\Phi}\mu_1}\sum_{i=1}^N\sum_{j\neq i}\left\|\omega^0_i -\omega_{j}^0\right\|^2.
\end{split}
\end{equation}
By substituting (\ref{eqn:deltapsisumt}) in Lemma~\ref{deltapsisumt} into the r.h.s. of the above equation, we have
\begin{equation}
\begin{split}
 \Phi(\omega_0^0) - \Phi(\omega_0^T) &\leq - E_1 \sum_{t=1}^T\left\|\nabla \Phi\left(\omega_{0}^{t}\right)\right\|^2 + E_2 \sum_{i=0}^{N}\left\|\psi_i^{0} - \psi^{\star}(\omega_i^{0})\right\|^2 + E_3 \sum_{i=0}^{N}\sum_{j \neq i}\left\|\omega^0_i - \omega^{0}_j\right\|^2, 
\end{split}
\end{equation}
 where 
\begin{align}
    E_1 &= \min\left\{ \frac{D_3}{D_1}, \frac{D_4}{D_2} \right\}, \\
    E_2 &= \frac{4E_1C_2\left((L_{\Phi}+ \mu_1)^2+ L_{\Phi}^2\right)L_{12}^2}{N\mu_1^2} + \frac{C_2(3\mu_1+ 16L_{\Phi})L_{12}^2}{NL_{\Phi}\mu_1},  \\
    E_3 &= \frac{4E_1C_3\left((L_{\Phi}+ \mu_1)^2+L_{\Phi}^2\right)L_{12}^2}{N\mu_1^2} + \frac{C_3(3\mu_1+ 16L_{\Phi})L_{12}^2}{NL_{\Phi}\mu_1},\\ 
    D_1 &= \frac{4L^2_{\Phi}\mu^2_1N 
+ 4C_1L^2_{12}(L_{\Phi}+\mu_1)^2}{N\mu_1^2},\\
  D_2 &=\frac{4C_4L^2_{12}\left((L_{\Phi}+\mu_1)^2 + L_{\Phi}^2\right)}{N\mu_1^2}, \\
  D_3 &= \frac{\mu_1^2-2L_{\Phi}\mu_1-4L_{\Phi}^2}{2\mu_1N} - \frac{C_1^2(3\mu_1+ 16L_{\Phi})L_{12}^2}{NL_{\Phi}\mu_1}, \\
  D_4 &= \frac{\mu_1}{4}.
\end{align}
In particular, taking $\liminf$ on both sides of the inequality w.r.t.\ $T$, we have:
\begin{equation}
\label{equ:liminf}
    \Phi(\omega_0^0) - \Psi(\omega_0^T) \leq \liminf_{T\to\infty} - E_1 \sum_{t=1}^T\left\|\nabla \Phi\left(\omega_{0}^{t}\right)\right\|^2 + E_2 \sum_{i=1}^N\left\|\psi_i^{0} - \psi^{\star}(\omega_i^{0})\right\|^2 + E_3 \sum_{i=1}^N\sum_{j \neq i}\left\|\omega^0_i - \omega^{0}_j\right\|^2.     
\end{equation}
Note that in the inequality~\eqref{equ:liminf}, all the other terms are independent of $T$ except for $\liminf_{T\to\infty} - E_1 \sum_{t=1}^T\left\|\nabla \Phi\left(\omega_{0}^{t}\right)\right\|^2$. Rearranging the terms independent of $T$ to the other side, yielding:
\begin{equation*}
    \liminf_{T\to\infty} - E_1 \sum_{t=1}^T\left\|\nabla \Phi\left(\omega_{0}^{t}\right)\right\|^2 \geq C,
\end{equation*}
where $C$ is a constant independent of $T$. In particular, this implies that
\begin{equation*}
    \limsup_{T\to\infty} \sum_{t=1}^T\left\|\nabla \Phi\left(\omega_{0}^{t}\right)\right\|^2 \leq - C / E_1.
\end{equation*}
Given that the sequence $\{\left\|\nabla \Phi\left(\omega_{0}^{t}\right)\right\|^2\}$ is nonnegative, we must have 
\begin{equation*}
    \lim_{t\to\infty}\left\|\nabla \Phi\left(\omega_{0}^{t}\right)\right\|^2 = 0,
\end{equation*}
which completes the proof.
\end{proof}

\subsection{Discussion on the Range of $\mu_1$ and $\mu_2$ for Theoretically  Convergence Guarantee}
\label{mu1mu2bound}
Let we chose a large enough $l$ such that
\begin{equation}
l \geq \max\{\frac{L_{22}}{B},  \frac{L_{12}}{B}\}
\end{equation}
By choosing $\mu_1$ and $\mu_2$ in the following range, 
\begin{align}
\label{mu1mu2range1}
     &\max\{64L_{22}l, 64L_{12}l\} \leq \mu_2 \leq 64l^2B \\
\label{mu1mu2range2}
    &\mu_1 \geq \max\{L_{21}L_{12}l\sqrt{\frac{24576}{B\mu_2}},  L_{12}\sqrt{\frac{512N}{1600}},  4L_{\Phi},
 \frac{7\cdot164092^2 l^4\kappa^4L_{12}^2}{L_{\Phi}} \}
\end{align}

\begin{lemma}(Conditions on $B_2<1$)
Let we chose a large enough l such that
\begin{equation}
l \geq \max\left\{\frac{L_{22}}{B},  \frac{L_{12}}{B}\right\}.
\end{equation}
With the following assumptions holds, 
\begin{align*}
    \frac{L_{22}}{\mu_2} &\leq \frac{1}{64l}, \\
    \frac{L_{12}}{\mu_2} &\leq \frac{1}{64l}, \\
    \frac{B}{\mu_2} &\geq \frac{1}{64l^{2}}, \\
    \frac{192L_{21}^2L_{12}^2}{\mu_2\mu_1^2B} &\leq \frac{1}{128l^{2}} ,  \\
    \frac{L_{12}^2N}{1600\mu_1^2l^{2}} &\leq \frac{1}{512l^2}.
\end{align*}
By properly choosing $\bar{B}_1$ in Lemma~\ref{lemma:epsilonaddbarpsit}, we have $B_2 \leq 1$.
\end{lemma}
\begin{proof}
Following (\ref{A1-A4}), we have
\begin{align*}
A_1 &= 16L_{22}^4N /  \left(\mu_2^4  -  16L_{22}^2 \left(\mu_2+L_{22}\right)^2\right) \leq \frac{N}{614400l^{4}}, \\
A_3 & = 16L_{22}^2\mu_2^2/ \left(\mu_2^4  -  16L_{22}^2 \left(\mu_2+L_{22}\right)^2\right) \leq \frac{1}{1600l^{2}}.\\\\
\end{align*}
So we have
\begin{align*}
    A_{10} &= A_{3}\frac{12NL_{12}^2}{\mu_1^2} \leq \frac{L_{12}^2}{1600\mu_1^2l^{2}},\\
    A_{7} &= \left(\frac{96L_{21}^2L_{12}^2}{\mu_2\mu_1^2B} + \frac{2\mu_2}{2\mu_2+B}\right) / \left(1-\frac{96L_{21}^2L_{12}^2}{\mu_2\mu_1^2B}\right)\leq \left(1-\frac{B}{2\mu_2+B} + \frac{192L_{21}^2L_{12}^2}{\mu_2\mu_1^2B} \right ) \leq 1-\frac{1}{128l^2},\\ 
    A_{9} &= A_1 + A_{3}\frac{12NL_{12}^2}{\mu_1^2} \leq \frac{N}{614400l^{4}} + \frac{L_{12}^2N}{1600\mu_1^2l^{2}}, \\
    A_6 &= \frac{8\mu_2}{NB}/\left(1- \frac{96L_{21}^2L_{12}^2}{\mu_2\mu_1^2B} \right ) \leq \frac{516}{Nl^{2}}, \\
    A_4 &= 4L_{21}^2N\mu_2^2/ \left(\mu_2^4  -  16L_{22}^2 \left(\mu_2+L_{22}\right)^2\right) \leq \frac{N}{6400l^2}, \\
    A_5 &= \frac{8(\mu_2+L_{22})^2}{NB\mu_2} \big/ \left(1- \frac{96L_{21}^2L_{12}^2}{\mu_2\mu_1^2B} \right ) \leq \frac{573}{Nl^{2}}.
\end{align*}
By choosing $\bar{B}_1=\frac{1200}{Nl^2}$, we have
\begin{equation}
\begin{split}
    \frac{\bar{B}_1A_{10}+A_7}{1-\bar{B}_1A_9} & \leq 1- \frac{1}{512l^2}, \\
    \frac{\bar{B}_1A_4 + A_6}{\bar{B}_1-A_5} & \leq \frac{6}{7}.
 \end{split}
\end{equation}
Thus, we  bound the decay factor $B_2$ by
\begin{equation}
    B_2 = \max\{\frac{\bar{B}_1A_{10}+A_7}{1-\bar{B}_1A_9}, \frac{\bar{B}_1A_4 + A_6}{\bar{B}_1-A_5}\} \leq 1- \frac{1}{512l^2}.
\end{equation}
\end{proof}

\textsc{Proof of positive of $E_1$}.
As we have
\begin{align*}
    B_1 &= \frac{\bar{B}_1A_{11}+A_8}{1-\bar{B}_1A_9} \leq 516\kappa^2,\\
    B_3 &= \frac{\bar{B}_1-A_5}{1-\bar{B}_1A_9} \leq \frac{700}{Nl^2}.
\end{align*}
So we have
\begin{align*}
      C_1 &= \frac{B_3}{1-B_2} \leq 164092l^2\kappa^2, \\
    C_2 &= \frac{1}{1-B_2} \leq 512l^2, \\
    C_3 &=  \frac{B_1}{1-B_2}  \leq \frac{358400}{N}.
\end{align*}

As from Lemma~\ref{lemma:boundci}, we have
\begin{equation}
    C_1 \leq 164092l^2\kappa^2.
\end{equation}
So the $D_3$ is lower bounded by
\begin{equation}
    D_3 \geq \frac{\mu_1^2-2L_{\Phi}\mu_1-4L_{\Phi}^2}{2\mu_1N} - \frac{164092^2l^4\kappa^4(3\mu_1+ 16L_{\Phi})L_{12}^2}{NL_{\Phi}\mu_1}.
\end{equation}
By setting
\begin{align*}
    \mu_1 &\geq \max\left( 4L_{\Phi},
 \frac{7\cdot164092^2 l^4\kappa^4L_{12}^2}{L_{\Phi}}\right),
\end{align*}
we finally prove that  
\begin{equation}
    D_3 \geq \frac{\mu_1}{8N} - \frac{7\cdot164092^2 l^4\kappa^4L_{12}^2}{NL_{\Phi}} \geq 0.
\end{equation}
It is then evident that $E_1$, $E_2$, and $E_3$ are positive.

\subsection{Convergence Analysis with Bounded Local Gradient Error
}
\label{append:error-bound}
In the previous subsections, we assume that each client fully optimizes their local augmented Lagrangian function  and assume $\left\|\nabla \mathcal{L}_i(\omega_i^t, \psi_i^t)\right\| =0$. In this subsection, we remove this assumption by assuming that there exists a local residue  gradient error, i.e., 
\begin{equation}
\left\|\nabla \mathcal{L}_i(\omega_i^t, \psi_i^t)\right\|^2\leq \epsilon.
\end{equation}
More specifically, 
we define the residue of gradient as
\begin{equation}
\begin{split}
    \nabla_{\omega} \mathcal{L}_i(\omega_i^t, \psi_i^t)  = e_{\omega, i}^t,  \quad
    \nabla_{\psi} \mathcal{L}_i(\omega_i^t, \psi_i^t)  = e_{\psi, i}^t.  
\end{split}
\end{equation}
Following the above assumptions, the generalization of previous results is straightforward but tedious; thus, we provide the key results in the following directly.
Proposition~\ref{prop:psi_iter} can be simply generalized to the form that
\begin{equation}
\sum_{i=1}^N \psi_i^{t+1} = \sum_{i=1}^N \psi_i^t + \frac{1}{\mu_2}\sum_{i=1}^N  \nabla_{\psi} f_i\left(\omega^{t+1}_i, \psi^{t+1}_i\right) - \frac{1}{\mu_2}\sum_{i=1}^N e_{\psi, i}^t. 
\end{equation}
Similarly,  Proposition \ref{prop:omega_iter} becomes
\begin{equation}
\sum_{i=1}^N \omega_i^{t+1} = \sum_{i=1}^N \omega_i^t - \frac{1}{\mu_1}\sum_{i=1}^N  \nabla_{\psi} f_i\left(\omega^{t+1}_i, \psi^{t+1}_i\right) - \frac{1}{\mu_1}\sum_{i=1}^Ne_{\omega, i}^t,
\end{equation}
and Proposition \ref{prop:psi_step} becomes
\begin{equation}
\mu_2 \left(\psi_i^{t+1} - \psi^t_0\right) = \nabla_{\psi} f_i\left(\omega^{t+1}_i, \psi^{t+1}_i\right) -  \nabla_{\psi} f_i\left(\omega^t_i, \psi^t_i\right)+ e_{\psi, i}^t - e_{\psi, i}^{t+1}.
\end{equation}

Finally, Proposition \ref{prop:delta_omega} can be simply generalized to  
\begin{equation}
\mu_1 \left(\omega_i^{t+1} - \omega^t_0\right) = \nabla_{\omega} f_i\left(\omega^{t}_i, \psi^{t}_i\right) -  \nabla_{\omega} f_i\left(\omega^{t+1}_i, \psi^{t+1}_i\right)- e_{\omega, i}^t + e_{\omega, i}^{t+1}. 
\end{equation}
Then by algebraic manipulations, Eqn. (\ref{eq61}) is further generalized to 
\begin{equation}
\begin{split}
 \left\|\psi_i^{t} - \psi^{\star}\left(\omega_i^{t}\right)\right\|
 \leq 
 &
 \sqrt{\frac{\mu_2}{\mu_2 + B}} \left\|  \psi_i^{t-1} - \psi^{\star}\left(\omega_i^{t-1}\right)\right\| + \sqrt{\frac{\mu_2}{\mu_2 + B}}\kappa \left\|  \omega_i^{t-1} - \omega_i^{t}\right\|\\
 &+ \sqrt{\frac{\mu_2}{\mu_2 + B}}\frac{1}{N} \sum_{\substack{j=1 \\ j\neq i}}^N \left\|  \psi_j^{t-1}- \psi_i^{t-1}\right\|
 +  \sqrt{\frac{\mu_2}{\mu_2 + B}}\frac{L_{21}}{N\mu_2} \sum_{\substack{j=1 \\ j\neq i}}^N\left\|  \omega_j^{t}- \omega_i^{t}\right\| \\
 &+  \sqrt{\frac{\mu_2}{\mu_2 + B}}\frac{\mu_2+L_{22}}{N\mu_2} \sum_{\substack{j=1 \\ j\neq i}}^N\left\|  \psi_j^{t}- \psi_i^{t}\right\| + \sum_{j=1}^N \sqrt{\frac{\mu_2}{\mu_2 + B}}\frac{1}{N\mu_2}\left\|e_{\psi, i}^t\right\|.
\end{split}
\end{equation}
We get a similar result as Lemma~\ref{lemma:psiconverge} by using the Cauchy-Schwarz inequality to the above equation, which is given by
\begin{equation}
\begin{split}
    \left\|\psi_i^{t} - \psi^{\star}\left(\omega_i^{t}\right)\right\|^2 \leq
 &
 \frac{2\mu_2}{2\mu_2 + B} \left\|  \psi_i^{t-1} - \psi^{\star}\left(\omega_i^{t-1}\right)\right\|^2 + \frac{10\mu_2}{B}\kappa^2 \left\|  \omega_i^{t-1} - \omega_i^{t}\right\|^2\\
 &+ \frac{10\mu_2}{NB} \sum_{\substack{j=1 \\ j\neq i}}^N \left\|  \psi_j^{t-1}- \psi_i^{t-1}\right\|^2
 +  \frac{10L_{21}^2}{N\mu_2B} \sum_{\substack{j=1 \\ j\neq i}}^N\left\|  \omega_j^{t}- \omega_i^{t}\right\|^2 \\
 &+  \frac{10(\mu_2+L_{22})^2}{N\mu_2B} \sum_{\substack{j=1 \\ j\neq i}}^N\left\|  \psi_j^{t}- \psi_i^{t}\right\|^2 + \frac{10}{\mu_2B}\epsilon.
\end{split}   
\end{equation}
Similarly, Lemma \ref{lemma:lemmaboundbaromegat} can be simply  generalized to the following expression
\begin{equation}
\frac{1}{2N}\bar{\omega}_t \leq \frac{8L^2_{\Phi}}{\mu_1^2}\widetilde{\omega}_t + \frac{8L_{12}^2}{\mu_1^2}\epsilon_t+ \frac{8L_{12}^2}{\mu_1^2}\epsilon_{t-1} + \frac{8}{\mu_1^2}\epsilon.
\end{equation}
Then we reach a similar result as that of Lemma \ref{lemma:lemmatildepsit} given by
\begin{equation}
 \bar{\psi}_t \leq  A_1 \epsilon_t + A_2 \widetilde{\omega}_t + A_3 \bar{\omega}_t +  A_4\bar{\psi}_{t-1} + A_5 \epsilon.
\end{equation}
Note that the value of $A_1$, $A_2$, $A_3$, $A_4$ is different with that of Lemma \ref{lemma:lemmatildepsit}, because the coefficient is different in the previous equations. But the deriving steps are similar.

We get the same result as Lemma~\ref{lemma:epsilonaddbarpsit} by using similar algebraic manipulations:
\begin{equation}
\label{epsilonaddbarpsit1}
    \epsilon_t + B_1 \bar{\psi_t} \leq B_2 (\epsilon_t + B_1 \bar{\psi_t}) + B_3 \bar{\omega_t} + B_4 \epsilon.
\end{equation}
Finding a $B_2<1$ would be achieved by following by the similar steps in section \ref{mu1mu2bound} with $\mu_1$ and $\mu_2$ in given convergence range. 

By recursive update and 
summing up on Eqn (\ref{epsilonaddbarpsit1}) from $t=1$ to $T$, we have
\begin{equation}
    \sum_{t=1}^T (\epsilon_t + B_1\bar{\psi}_t) \leq \sum_{t=1}^TB_2^t(\epsilon_0+ B_1\bar{\psi}_0) + \sum_{t=1}^T \sum_{t' \leq t} B_3B_2^{t-t'} \bar{\omega}_t + \sum_{t=1}^T \sum_{t' \leq t} B_4B_2^{t-t'}\epsilon.
\end{equation}
By following the similar steps of amplifying by the sum of finite exponential series, we would reach
\begin{equation}
  \sum_{t=1}^T \left(\epsilon_t + B_1\bar{\psi}_t\right)  \leq \frac{1}{1-B_2}\left(\epsilon_0 + B_1\bar{\psi}_0\right)  + \frac{B_3}{1-B_2}\sum_{t=1}^T\bar{\omega}_t + \frac{B_4}{1-B_2}\sum_{t=1}^T\epsilon.
\end{equation}
Finally, we  reach
\begin{equation}
  \sum_{t=1}^T \epsilon_t  \leq C_1\sum_{t=1}^T\bar{\omega}_t + C_2\epsilon_0 + C_3\bar{\psi}_0 + C_4\epsilon T,
\end{equation}
where 
\begin{align}
    C_1 &= \frac{B_3}{1-B_2}, \\
    C_2 &= \frac{1}{1-B_2}  \\
    C_3 &=  \frac{B_1}{1-B_2} \\
    C_4 &= \frac{B_3}{1-B_2}
\end{align}
Note that the value of $C_1$, $C_2$, $C_3$ and $C_4$ is different from that value of $\delta \mathcal{L}_i(\omega_i^t, \psi_i^t) =0$. Because the convergence range of $\mu_1$ and $\mu_2$ and $B_1, B_2, B_3$ is different in our analysis). Thus, we reach a result similar to Lemma \ref{Lphiiter}, which is given by
\begin{equation}
\label{lemma11*}
\begin{split}
\sum_{t=1}^T \sum_{i=0}^N \left\|\psi_i^{t} - \psi^{\star}(\omega_i^{t-1})\right\|^2 
\leq
&
C_1\sum_{t=1}^T\left\|\omega_i^{t-1} - \omega_i^t\right\|^2 + C_2\sum_{i = 1}^N\left\|\psi_i^{0} - \psi^{\star}(\omega_i^{0})\right\|^2  \\
&+ C_3\sum_{i = 1}^N\sum_{\substack{j=1 \\ j\neq i}}^N\left\|\omega_i^0 - \omega_j^0\right\|^2 + TC_4\epsilon.
\end{split}
\end{equation}
Next, following the similar step in Lemma \ref{Lphiiter}, we obtain
\begin{equation}
\label{eqn:Lphiiter}
\begin{split}
\frac{1}{N}\sum_{i=1}^N\left(\mathcal{L}^{\Phi}_i(\omega_i^{t+1}, \omega_{0}^{t+1}, \lambda_i^{t+1}) - \mathcal{L}^{\Phi}_i(\omega_i^{t}, \omega_{0}^{t}, \lambda_i^{t})\right) \leq  &
\frac{1}{N}\sum_{i=1}^N \Big(- \frac{\mu_1 - 2L_{\Phi}}{2}\left\|\omega_i^{t+1} - \omega_i^{t}\right\|^2 + \frac{1}{\mu_1}\left\|\lambda_i^{t+1} - \lambda_i^t\right\|^2  \\
&
+
2L_{12}^2\left\|\psi_i^{t+1} - \psi^{\star}\left(\omega_i^{t+1}\right)\right\|^2 + 2\epsilon \Big) - \frac{\mu_1}{2}\left\|\omega_{0}^{t+1} - \omega_0^t\right\|^2. 
\end{split}
\end{equation}
Then following the similar steps in the proof of Lemma \ref{deltalambdabound}, we have 
\begin{equation}
\label{eqn:deltalambdabound}
 \left\|\lambda_i^t - \lambda_i^{t+1}\right\|^2 \leq 2L_{\Phi}^2\left\|\omega_i^{t}- \omega_i^{t+1}\right\|^2 + 8L_{12}^2\left\|\psi_i^{t+1} - \psi^{\star}\left(\omega_i^{t+1}\right)\right\|^2 + 8L_{12}^2\left\|\psi_i^{t} - \psi^{\star}\left(\omega_i^{t}\right)\right\|^2 + 8\epsilon.
\end{equation}
Then we arrive at a similar conclusion with Lemma \ref{Ldescent} that
\begin{equation}
\begin{split}
\sum_{t=1}^T \sum_{i=1}^N \left(\mathcal{L}_i^{\Phi}(\omega_i^{T}, \omega_{0}^{T}, \lambda_i^{T}) - \mathcal{L}_i^{\Phi}(\omega_i^{0}, \omega_{0}^{0}, \lambda_i^{0})\right) &\leq  -\frac{\mu_1^2-2L_{\Psi}\mu_1-4L_{\Phi}^2}{2\mu_1}\sum_{i=1}^N\sum_{t=1 }^T\left\|\omega_i^{t+1}- \omega_i^t\right\|^2 - \sum_{t=1}^T\frac{N\mu_1}{4}\left\|\omega_0^t - \omega_0^{t+1}\right\|^2 \\
&\quad + \frac{4(\mu_1+ 8L_{\Phi})L_{12}^2}{L_{\Phi}\mu_1}\sum_{i=1}^{N}\left\|\psi_i^{t+1} - \psi^{\star}\left(\omega_i^{t+1}\right)\right\|^2 + (2 + \frac{8}{\mu_1})T \epsilon,
\end{split}
\end{equation}
as well as  a similar result as Lemma \ref{lemma:descentPhi}, which is given by
\begin{equation}
\label{eqnpsidescent1}
\begin{split}
\Phi(\omega_0^T) - \Phi(\omega_0^0) 
\leq
&
-\frac{\mu_1^2-2L_{\Phi}\mu_1-4L_{\Phi}^2}{2\mu_1N}\sum_{i=1}^N\sum_{t =1}^T\left\|\omega_i^{t+1}- \omega_i^t\right\|^2 -\frac{\mu_1}{4} \sum_{t =1}^T\left\|\omega_0^t - \omega_0^{t+1}\right\|^2 \\
&+ \frac{(5\mu_1+ 32L_{\Phi})L_{12}^2}{\mu_1NL_{\Phi}}\sum_{i=1}^N\sum_{t =1}^T\left\|\psi_i^{t+1} - \psi^{\star}\left(\omega_i^{t+1}\right)\right\|^2 +  \left(2 + \frac{8}{\mu_1}\right)T \epsilon.
\end{split}
\end{equation}
Next, we have the similar result as Lemma \ref{boundomega0i}
\begin{equation}
\label{eqn:boundomega0i}
\left\|\omega_{0}^t - \omega_{i}^t\right\| \leq \left\|\omega_{0}^{t-1} - \omega_{0}^t\right\| + \frac{L_{\Phi}}{\mu_1}\left\|\omega_{i}^{t-1} - \omega_{i}^t\right\| + \frac{L_{12}}{\mu_1}\left\|\psi_i^{t} - \psi^{\star}\left(\omega_i^{t}\right)\right\| + \frac{L_{12}}{\mu_1}\left\|\psi_i^{t-1} - \psi^{\star}(\omega_i^{t-1})\right\| + \frac{1}{\mu_1}\left\|e_{\omega, i}^t\right\| + \frac{1}{\mu_1}\left\|e_{\omega, i}^{t-1}\right\|.
\end{equation}
Then we bound $\left\|\nabla \Psi\left(\omega_{0}^{t}\right)\right\|^{2}$ with the following result similar to Lemma \ref{bounddeltapsi1} 
\begin{equation}
\label{eqn:bounddeltapsi1}
\begin{split}
    \left\|\nabla \Phi\left(\omega_{0}^{t}\right)\right\| \leq 
    & L_{\Phi}\left\|\omega^t_0 - \omega^{t-1}_0\right\| +  \frac{\mu_1^2 + L_{\Phi}^2}{N\mu_1}\sum_{i=1}^N\left\|\omega^{t-1}_i - \omega^t_i\right\| + \frac{L_{12}(L_{\Phi}+\mu_1)}{N\mu_1}\sum_{i=1}^N\left\|\psi_i^{t-1} - \psi^{\star}(\omega_i^{t-1})\right\|\\
    & + \frac{L_{12}L_{\Phi}}{N\mu_1}\sum_{i=1}^N\left\|\psi_i^{t} - \psi^{\star}\left(\omega_i^{t}\right)\right\| + \frac{\mu_1 + L_{\Phi}}{N\mu_1} \sum_{i=1}^N\left\|e_{\omega, i}^{t}\right\| + \frac{L_{\Phi}}{N\mu_1}\sum_{i=1}^N \left\|e_{\omega, i}^{t-1}\right\|.
\end{split}
\end{equation}
By applying the Cauchy-Schwarz inequality and summing up $t$ from $1$ to $T$, we have a similar result as Eqn. (\ref{eqn:sumPsit}), with the particular form as 
\begin{equation}
 \begin{split}
    \sum_{t=1}^T\left\|\nabla \Phi\left(\omega_{0}^{t}\right)\right\|^2 \leq & 4L_{\Phi}^2\sum_t\left\|\omega^t_0 - \omega^{t-1}_0\right\|^2 +  \frac{4(\mu_1^2+L_{\Phi}^2)^2}{N\mu_1^2}\sum_{i=1}^N\sum_{t=1}^T\left\|\omega^{t-1}_i - \omega^t_i\right\|^2 \\
    &+ \frac{8L_{12}^2\left((L_{\Phi}+\mu_1)^2+L_{\Phi}^2\right)}{N\mu_1^2}\sum_{i=1}^N\sum_{t=1}^T \left\|\psi_i^{t} - \psi^{\star}\left(\omega_i^{t}\right)\right\|^2 + 8\frac{(\mu_1 + 2L_{\Phi})^2}{\mu_1^2}T\epsilon. \\
 \end{split}   
\end{equation}
Then by substituting  Eqn. (\ref{lemma11*})
into the previous equation, we have the following result 
\begin{equation}
\begin{split}
\sum_{t=1}^T \left\|\nabla \Phi\left(\omega_{0}^{t}\right)\right\|^2 \leq
& 
\frac{4L^2_{\Phi}\mu^2_1N 
+ 8C_1L^2_{12}\left((L_{\Phi}+\mu_1)^2 + L_{\Phi}^2\right)}{N\mu_1^2}\sum_{t=1}^T\left\|\omega^t_0 - \omega^{t-1}_0\right\|^2\\
&
+  \frac{4(\mu_1^2+L_{\Phi}^2)^2}{N\mu_1^2}\sum_{i=0}^N\sum_{t=1}^T\left\|\omega^{t-1}_i - \omega^t_i\right\|^2 \\
& + \frac{8C_2L^2_{12}\left((L_{\Phi}+\mu_1)^2 + L_{\Phi}^2\right)}{N\mu_1^2}\sum_{i=1}^N\left\|\psi_i^{0} - \psi^{\star}(\omega_i^{0})\right\|^2 \\
&+ \frac{8C_3L^2_{12}\left((L_{\Phi}+\mu_1)^2 + L_{\Phi}^2\right)}{N\mu_1^2}\sum_{i=1}^N\sum_{j \neq i}\left\|\omega^0_i - \omega^{0}_j\right\|^2
\\
&+ \left(\frac{8(\mu_1 + 2L_{\Phi})^2}{\mu_1^2} 
+ \frac{8C_4L^2_{12}\left((L_{\Phi}+\mu_1)^2 + L_{\Phi}^2\right)}{N\mu_1^2}\right) T\epsilon.
\end{split}
\end{equation}
Finally, we  reach the convergence result, which is the counterpart of  Theorem \ref{theorem:convege} by following  similar manipulations in our final proof of Theorem \ref{theorem:convege}(substituting the above Eqn into Eqn (\ref{eqnpsidescent1})), which is given by
\begin{equation}
\label{eqn:final}
    \Phi(\omega_0^0) - \Psi(\omega_0^T) \leq - E_1 \sum_{t=1}^T\left\|\nabla \Phi\left(\omega_{0}^{t}\right)\right\|^2 + E_2 \sum_{i=1}^N\left\|\psi_i^{0} - \psi^{\star}(\omega_i^{0})\right\|^2 + E_3 \sum_{i=1}^N\sum_{j \neq i}\left\|\omega^0_i - \omega^{0}_j\right\|^2 + E_4T\epsilon.
\end{equation}
where
\begin{align}
    E_1 &= \min\left\{ \frac{D_3}{D_1}, \frac{D_4}{D_2} \right\}, \\
    E_2 &= \frac{8E_1C_2\left((L_{\Phi}+ \mu_1)^2+ L_{\Phi}^2\right)L_{12}^2}{N\mu_1^2} + \frac{C_2(5\mu_1+ 32L_{\Phi})L_{12}^2}{NL_{\Phi}\mu_1},  \\
    E_3 &= \frac{8E_1C_3\left((L_{\Phi}+ \mu_1)^2+L_{\Phi}^2\right)L_{12}^2}{N\mu_1^2} + \frac{C_3(5\mu_1+ 32L_{\Phi})L_{12}^2}{NL_{\Phi}\mu_1},\\ 
    E_4 &= \frac{8NE_1(\mu_1 + 2L_{\Phi})^2+ 8E_1C_4L^2_{12}\left((L_{\Phi}+\mu_1)^2 + L_{\Phi}^2\right)}{N\mu_1^2} + C_4(2+\frac{8}{\mu_1}), \\
    D_1 &= \frac{4L^2_{\Phi}\mu^2_1N 
+ 4C_1L^2_{12}(L_{\Phi}+\mu_1)^2}{N\mu_1^2},\\
  D_2 &=\frac{8C_4L^2_{12}\left((L_{\Phi}+\mu_1)^2 + L_{\Phi}^2\right)}{N\mu_1^2}, \\
  D_3 &= \frac{\mu_1^2-2L_{\Phi}\mu_1-4L_{\Phi}^2}{2\mu_1N} - \frac{C_1^2(5\mu_1+ 32L_{\Phi})L_{12}^2}{NL_{\Phi}\mu_1}, \\
  D_4 &= \frac{\mu_1}{4}.
\end{align}
Similarly, since $\Phi(\omega_0^0) - \Psi(\omega_0^T)$ is lower-bounded by a constant $C$. Rearranging terms on the RHS of (\ref{eqn:final}), we have
\begin{equation*}
    \sum_{t=1}^T\left\|\nabla \Phi\left(\omega_{0}^{t}\right)\right\|^2 \leq \frac{E_2 \sum_{i=1}^N\left\|\psi_i^{0} - \psi^{\star}(\omega_i^{0})\right\|^2 + E_3 \sum_{i=1}^N\sum_{j \neq i}\left\|\omega^0_i - \omega^{0}_j\right\|^2 - C}{E_1} + \frac{E_4T\epsilon}{E_1}.
\end{equation*}
Dividing both sides by $T$ and taking $\limsup_{T\to\infty}$, we obtain
\begin{align*}
    \limsup_{T\to\infty}\frac{\sum_{t=1}^T\left\|\nabla \Phi \left(\omega_{0}^{t}\right)\right\|^2}{T} &\leq \limsup_{T\to\infty}\frac{E_2 \sum_{i=1}^N\left\|\psi_i^{0} - \psi^{\star}(\omega_i^{0})\right\|^2 + E_3 \sum_{i=1}^N\sum_{j \neq i}\left\|\omega^0_i - \omega^{0}_j\right\|^2 - C}{TE_1} + \frac{E_4\epsilon}{E_1}\\
    &= \frac{E_4\epsilon}{E_1},
\end{align*}
which implies that $\sum_{t=1}^T\left\|\nabla \Phi \left(\omega_{0}^{t}\right)\right\|^2 = O(T\epsilon)$, and for sufficiently large $t$, $\left\|\nabla \Phi \left(\omega_{0}^{t}\right)\right\|^2 = O(\epsilon)$, completing the proof.



\begin{thebibliography}{37}
\providecommand{\natexlab}[1]{#1}
\providecommand{\url}[1]{\texttt{#1}}
\expandafter\ifx\csname urlstyle\endcsname\relax
  \providecommand{\doi}[1]{doi: #1}\else
  \providecommand{\doi}{doi: \begingroup \urlstyle{rm}\Url}\fi

\bibitem[Acar et~al.(2021)Acar, Zhao, Navarro, Mattina, Whatmough, and
  Saligrama]{acar2021federated}
Acar, D. A.~E., Zhao, Y., Navarro, R.~M., Mattina, M., Whatmough, P.~N., and
  Saligrama, V.
\newblock Federated learning based on dynamic regularization.
\newblock In \emph{International Conference on Learning Representations}, 2021.

\bibitem[Arbelaez et~al.(2010)Arbelaez, Maire, Fowlkes, and
  Malik]{arbelaez2010contour}
Arbelaez, P., Maire, M., Fowlkes, C., and Malik, J.
\newblock Contour detection and hierarchical image segmentation.
\newblock \emph{IEEE transactions on pattern analysis and machine
  intelligence}, 33\penalty0 (5):\penalty0 898--916, 2010.

\bibitem[Ben-David et~al.(2010)Ben-David, Blitzer, Crammer, Kulesza, Pereira,
  and Vaughan]{ben2010theory}
Ben-David, S., Blitzer, J., Crammer, K., Kulesza, A., Pereira, F., and Vaughan,
  J.~W.
\newblock A theory of learning from different domains.
\newblock \emph{Machine learning}, 79\penalty0 (1):\penalty0 151--175, 2010.

\bibitem[Chen et~al.(2020)Chen, Kairouz, and {\"O}zg{\"u}r]{chen2020breaking}
Chen, W.-N., Kairouz, P., and {\"O}zg{\"u}r, A.
\newblock Breaking the communication-privacy-accuracy trilemma.
\newblock \emph{arXiv preprint arXiv:2007.11707}, 2020.

\bibitem[Deng \& Mahdavi(2021)Deng and Mahdavi]{deng2021local}
Deng, Y. and Mahdavi, M.
\newblock Local stochastic gradient descent ascent: Convergence analysis and
  communication efficiency.
\newblock In \emph{International Conference on Artificial Intelligence and
  Statistics}, pp.\  1387--1395. PMLR, 2021.

\bibitem[Ganin \& Lempitsky(2015)Ganin and Lempitsky]{ganin2015unsupervised}
Ganin, Y. and Lempitsky, V.
\newblock Unsupervised domain adaptation by backpropagation.
\newblock In \emph{International conference on machine learning}, pp.\
  1180--1189. PMLR, 2015.

\bibitem[Ganin et~al.(2016)Ganin, Ustinova, Ajakan, Germain, Larochelle,
  Laviolette, Marchand, and Lempitsky]{ganin2016domain}
Ganin, Y., Ustinova, E., Ajakan, H., Germain, P., Larochelle, H., Laviolette,
  F., Marchand, M., and Lempitsky, V.
\newblock Domain-adversarial training of neural networks.
\newblock \emph{The journal of machine learning research}, 17\penalty0
  (1):\penalty0 2096--2030, 2016.

\bibitem[Jin et~al.(2020)Jin, Netrapalli, and Jordan]{jin2020local}
Jin, C., Netrapalli, P., and Jordan, M.
\newblock What is local optimality in nonconvex-nonconcave minimax
  optimization?
\newblock In \emph{International Conference on Machine Learning}, pp.\
  4880--4889. PMLR, 2020.

\bibitem[Kairouz et~al.(2019)Kairouz, McMahan, Avent, Bellet, Bennis, Bhagoji,
  Bonawitz, Charles, Cormode, Cummings, et~al.]{kairouz2019advances}
Kairouz, P., McMahan, H.~B., Avent, B., Bellet, A., Bennis, M., Bhagoji, A.~N.,
  Bonawitz, K., Charles, Z., Cormode, G., Cummings, R., et~al.
\newblock Advances and open problems in federated learning.
\newblock \emph{arXiv preprint arXiv:1912.04977}, 2019.

\bibitem[Karimireddy et~al.(2020)Karimireddy, Kale, Mohri, Reddi, Stich, and
  Suresh]{karimireddy2020scaffold}
Karimireddy, S.~P., Kale, S., Mohri, M., Reddi, S., Stich, S., and Suresh,
  A.~T.
\newblock Scaffold: Stochastic controlled averaging for federated learning.
\newblock In \emph{International Conference on Machine Learning}, pp.\
  5132--5143. PMLR, 2020.

\bibitem[Li et~al.(2018)Li, Sahu, Zaheer, Sanjabi, Talwalkar, and
  Smith]{li2018federated}
Li, T., Sahu, A.~K., Zaheer, M., Sanjabi, M., Talwalkar, A., and Smith, V.
\newblock Federated optimization in heterogeneous networks.
\newblock \emph{arXiv preprint arXiv:1812.06127}, 2018.

\bibitem[Li et~al.(2020)Li, Sahu, Talwalkar, and Smith]{li2020federated}
Li, T., Sahu, A.~K., Talwalkar, A., and Smith, V.
\newblock Federated learning: Challenges, methods, and future directions.
\newblock \emph{IEEE Signal Processing Magazine}, 37\penalty0 (3):\penalty0
  50--60, 2020.

\bibitem[Lin et~al.(2020{\natexlab{a}})Lin, Jin, and Jordan]{lin2020gradient}
Lin, T., Jin, C., and Jordan, M.
\newblock On gradient descent ascent for nonconvex-concave minimax problems.
\newblock In \emph{International Conference on Machine Learning}, pp.\
  6083--6093. PMLR, 2020{\natexlab{a}}.

\bibitem[Lin et~al.(2020{\natexlab{b}})Lin, Jin, and Jordan]{lin2020near}
Lin, T., Jin, C., and Jordan, M.~I.
\newblock Near-optimal algorithms for minimax optimization.
\newblock In \emph{Conference on Learning Theory}, pp.\  2738--2779. PMLR,
  2020{\natexlab{b}}.

\bibitem[Long et~al.(2017)Long, Cao, Wang, and Jordan]{long2017conditional}
Long, M., Cao, Z., Wang, J., and Jordan, M.~I.
\newblock Conditional adversarial domain adaptation.
\newblock \emph{arXiv preprint arXiv:1705.10667}, 2017.

\bibitem[Lu et~al.(2020)Lu, Tsaknakis, Hong, and Chen]{lu2020hybrid}
Lu, S., Tsaknakis, I., Hong, M., and Chen, Y.
\newblock Hybrid block successive approximation for one-sided non-convex
  min-max problems: algorithms and applications.
\newblock \emph{IEEE Transactions on Signal Processing}, 68:\penalty0
  3676--3691, 2020.

\bibitem[Luo et~al.(2020)Luo, Ye, Huang, and Zhang]{luo2020stochastic}
Luo, L., Ye, H., Huang, Z., and Zhang, T.
\newblock Stochastic recursive gradient descent ascent for stochastic
  nonconvex-strongly-concave minimax problems.
\newblock \emph{arXiv preprint arXiv:2001.03724}, 2020.

\bibitem[McMahan et~al.(2017)McMahan, Moore, Ramage, Hampson, and
  y~Arcas]{mcmahan2017communication}
McMahan, B., Moore, E., Ramage, D., Hampson, S., and y~Arcas, B.~A.
\newblock Communication-efficient learning of deep networks from decentralized
  data.
\newblock In \emph{Artificial Intelligence and Statistics}, pp.\  1273--1282.
  PMLR, 2017.

\bibitem[Nouiehed et~al.(2019)Nouiehed, Sanjabi, Huang, Lee, and
  Razaviyayn]{nouiehed2019solving}
Nouiehed, M., Sanjabi, M., Huang, T., Lee, J.~D., and Razaviyayn, M.
\newblock Solving a class of non-convex min-max games using iterative first
  order methods.
\newblock \emph{arXiv preprint arXiv:1902.08297}, 2019.

\bibitem[Peng et~al.(2019)Peng, Huang, Zhu, and Saenko]{peng2019federated}
Peng, X., Huang, Z., Zhu, Y., and Saenko, K.
\newblock Federated adversarial domain adaptation.
\newblock In \emph{International Conference on Learning Representations}, 2019.

\bibitem[Qui{\~n}onero-Candela et~al.(2009)Qui{\~n}onero-Candela, Sugiyama,
  Lawrence, and Schwaighofer]{quinonero2009dataset}
Qui{\~n}onero-Candela, J., Sugiyama, M., Lawrence, N.~D., and Schwaighofer, A.
\newblock \emph{Dataset shift in machine learning}.
\newblock Mit Press, 2009.

\bibitem[Rafique et~al.(2018)Rafique, Liu, Lin, and Yang]{rafique2018non}
Rafique, H., Liu, M., Lin, Q., and Yang, T.
\newblock Non-convex min-max optimization: Provable algorithms and applications
  in machine learning.
\newblock \emph{arXiv preprint arXiv:1810.02060}, 2018.

\bibitem[Rasouli et~al.(2020)Rasouli, Sun, and Rajagopal]{rasouli2020fedgan}
Rasouli, M., Sun, T., and Rajagopal, R.
\newblock Fedgan: Federated generative adversarial networks for distributed
  data.
\newblock \emph{arXiv preprint arXiv:2006.07228}, 2020.

\bibitem[Reisizadeh et~al.(2020)Reisizadeh, Farnia, Pedarsani, and
  Jadbabaie]{reisizadeh2020robust}
Reisizadeh, A., Farnia, F., Pedarsani, R., and Jadbabaie, A.
\newblock Robust federated learning: The case of affine distribution shifts.
\newblock \emph{arXiv preprint arXiv:2006.08907}, 2020.

\bibitem[Rockafellar(2015)]{rockafellar2015convex}
Rockafellar, R.~T.
\newblock \emph{Convex analysis}.
\newblock Princeton university press, 2015.

\bibitem[Russakovsky et~al.(2015)Russakovsky, Deng, Su, Krause, Satheesh, Ma,
  Huang, Karpathy, Khosla, Bernstein, et~al.]{russakovsky2015imagenet}
Russakovsky, O., Deng, J., Su, H., Krause, J., Satheesh, S., Ma, S., Huang, Z.,
  Karpathy, A., Khosla, A., Bernstein, M., et~al.
\newblock Imagenet large scale visual recognition challenge.
\newblock \emph{International journal of computer vision}, 115\penalty0
  (3):\penalty0 211--252, 2015.

\bibitem[Saenko et~al.(2010)Saenko, Kulis, Fritz, and
  Darrell]{saenko2010adapting}
Saenko, K., Kulis, B., Fritz, M., and Darrell, T.
\newblock Adapting visual category models to new domains.
\newblock In \emph{European conference on computer vision}, pp.\  213--226.
  Springer, 2010.

\bibitem[Sandler et~al.(2018)Sandler, Howard, Zhu, Zhmoginov, and
  Chen]{sandler2018mobilenetv2}
Sandler, M., Howard, A., Zhu, M., Zhmoginov, A., and Chen, L.-C.
\newblock Mobilenetv2: Inverted residuals and linear bottlenecks.
\newblock In \emph{Proceedings of the IEEE conference on computer vision and
  pattern recognition}, pp.\  4510--4520, 2018.

\bibitem[Tachet~des Combes et~al.(2020)Tachet~des Combes, Zhao, Wang, and
  Gordon]{tachet2020domain}
Tachet~des Combes, R., Zhao, H., Wang, Y.-X., and Gordon, G.~J.
\newblock Domain adaptation with conditional distribution matching and
  generalized label shift.
\newblock \emph{Advances in Neural Information Processing Systems}, 33, 2020.

\bibitem[Tzeng et~al.(2017)Tzeng, Hoffman, Saenko, and
  Darrell]{tzeng2017adversarial}
Tzeng, E., Hoffman, J., Saenko, K., and Darrell, T.
\newblock Adversarial discriminative domain adaptation.
\newblock In \emph{Proceedings of the IEEE conference on computer vision and
  pattern recognition}, pp.\  7167--7176, 2017.

\bibitem[Wang et~al.(2020)Wang, Liu, Liang, Joshi, and Poor]{wang2020tackling}
Wang, J., Liu, Q., Liang, H., Joshi, G., and Poor, H.~V.
\newblock Tackling the objective inconsistency problem in heterogeneous
  federated optimization.
\newblock \emph{arXiv preprint arXiv:2007.07481}, 2020.

\bibitem[Xia et~al.(2021)Xia, Yang, Li, Myronenko, Xu, Obinata, Mori, An,
  Harmon, Turkbey, et~al.]{xia2021auto}
Xia, Y., Yang, D., Li, W., Myronenko, A., Xu, D., Obinata, H., Mori, H., An,
  P., Harmon, S., Turkbey, E., et~al.
\newblock Auto-fedavg: Learnable federated averaging for multi-institutional
  medical image segmentation.
\newblock \emph{arXiv preprint arXiv:2104.10195}, 2021.

\bibitem[Zhang et~al.(2013)Zhang, Sch{\"o}lkopf, Muandet, and
  Wang]{zhang2013domain}
Zhang, K., Sch{\"o}lkopf, B., Muandet, K., and Wang, Z.
\newblock Domain adaptation under target and conditional shift.
\newblock In \emph{International Conference on Machine Learning}, pp.\
  819--827. PMLR, 2013.

\bibitem[Zhang et~al.(2020)Zhang, Hong, Dhople, Yin, and Liu]{zhang2020fedpd}
Zhang, X., Hong, M., Dhople, S., Yin, W., and Liu, Y.
\newblock Fedpd: A federated learning framework with optimal rates and
  adaptivity to non-iid data.
\newblock \emph{arXiv preprint arXiv:2005.11418}, 2020.

\bibitem[Zhang et~al.(2019)Zhang, Liu, Long, and Jordan]{zhang2019bridging}
Zhang, Y., Liu, T., Long, M., and Jordan, M.
\newblock Bridging theory and algorithm for domain adaptation.
\newblock In \emph{International Conference on Machine Learning}, pp.\
  7404--7413. PMLR, 2019.

\bibitem[Zhao et~al.(2018)Zhao, Zhang, Wu, Moura, Costeira, and
  Gordon]{zhao2018adversarial}
Zhao, H., Zhang, S., Wu, G., Moura, J.~M., Costeira, J.~P., and Gordon, G.~J.
\newblock Adversarial multiple source domain adaptation.
\newblock \emph{Advances in neural information processing systems},
  31:\penalty0 8559--8570, 2018.

\bibitem[Zhao et~al.(2019)Zhao, Des~Combes, Zhang, and
  Gordon]{zhao2019learning}
Zhao, H., Des~Combes, R.~T., Zhang, K., and Gordon, G.
\newblock On learning invariant representations for domain adaptation.
\newblock In \emph{International Conference on Machine Learning}, pp.\
  7523--7532. PMLR, 2019.

\end{thebibliography}
\end{document}